\documentclass[ijoc,sglanonrev]{informs4}
\usepackage{eqndefns-left} % For checking the display equation width and equation environment definitions %
\RequirePackage{tgtermes}
\RequirePackage{newtxtext}
\RequirePackage{newtxmath}
\RequirePackage{bm}
\RequirePackage{endnotes}

\OneAndAHalfSpacedXII % Current default line spacing
\usepackage{algorithm}
\usepackage{tikz}
\usepackage{subfigure}
\usepackage{booktabs}
\usepackage{algorithmic}

\usepackage{natbib}
 \bibpunct[, ]{(}{)}{,}{a}{}{,}%
 \def\bibfont{\small}%

\def\N{\mathbb{N}}
\def\P{\mathbb{P}}
\def\p{\mathbb{P}}
\def\E{\mathbb{E}}

\def\R{\mathbb{R}}

\def\M{\mathcal{M}}

\def\H{\mathcal{H}}

\def\X{\mathcal{X}}

\def\d{\,\mathrm{d}}
\DeclareMathOperator*{\esssup}{ess\,sup}
\DeclareMathOperator*{\essinf}{ess\,inf}

\newcommand{\VaR}{\mathrm{VaR}}
\newcommand{\ES}{\mathrm{ES}}

\makeatletter  %算法跨页
\newenvironment{breakablealgorithm}
  {% \begin{breakablealgorithm}
   \begin{center}
     \refstepcounter{algorithm}% New algorithm
     \hrule height.8pt depth0pt \kern2pt%
     \renewcommand{\caption}[2][\relax]{% Make a new \caption
       {\raggedright\textbf{\ALG@name~\thealgorithm} ##2\par}%
       \ifx\relax##1\relax%
         \addcontentsline{loa}{algorithm}{\protect\numberline{\thealgorithm}##2}%
       \else%
         \addcontentsline{loa}{algorithm}{\protect\numberline{\thealgorithm}##1}%
       \fi%
       \kern2pt\hrule\kern2pt
     }
  }{% \end{breakablealgorithm}
     \kern2pt\hrule\relax%
   \end{center}
  }
\makeatother

\EquationsNumberedThrough    % Default: (1), (2), ...
\TheoremsNumberedThrough     % Preferred (Theorem 1, Lemma 1, Theorem 2)
\ECRepeatTheorems  %  

\MANUSCRIPTNO{IJOC-0001-2024.00}

\begin{document}
%%%%%%%%%%%%%%%%

%\setcounter{page}{1}
%\DOI{10.1287/xxxx.0000.0000}%

\RUNAUTHOR{Yang Liu, Yuhao Liu and Yunran Wei}

% \RUNTITLE{Predictive Maintenance in Manufacturing}
\RUNTITLE{A Noise-Robust Elicit-to-Optimize Framework for Distortion Riskmetrics via Inverse Reinforcement Learning}

% \TITLE{Optimal Resource Allocation in Humanitarian Logistics: A Stochastic Programming Approach}
\TITLE{A Noise-Robust Elicit-to-Optimize Framework for Distortion Riskmetrics via Inverse Reinforcement Learning}

\ARTICLEAUTHORS{%
%\AUTHOR{John Doe,\textsuperscript{a} Jane Smith,\textsuperscript{b}}
%\AFF{\textsuperscript{a}Department of Industrial Engineering, University of XYZ, \EMAIL{john.doe@xyz.edu; \textsuperscript{b}Department of Computer Science, University of ABC, \EMAIL{jane.smith@abc.edu}} 
\AUTHOR{Yang Liu}
\AFF{School of Science and Engineering, The Chinese University of Hong Kong (Shenzhen), China. \EMAIL{yangliu16@cuhk.edu.cn}}

\AUTHOR{Yuhao Liu\textsuperscript{*}}
\AFF{School of Science and Engineering, The Chinese University of Hong Kong (Shenzhen), China. \EMAIL{liuyuhao@cuhk.edu.cn}\\
\textsuperscript{*}Corresponding author}

\AUTHOR{Yunran Wei}
\AFF{School of Mathematics and Statistics, Carleton University, Canada. \EMAIL{yunran.wei@carleton.ca}}

} % end of the block

\ABSTRACT{%
% Enter your abstract
We propose a noise-robust elicit-to-optimize framework that integrates inverse reinforcement learning (IRL) and reinforcement learning (RL) for eliciting agents' risk preferences and  optimizing the policies under a broad class of risk objectives, characterized by distortion riskmetrics. On the elicitation side, we propose an adaptive Bayesian IRL method that infers agents’ latent risk objectives from their noisy observed decisions, explicitly allowing agents to take stochastic and suboptimal actions. We establish the existence of a finite set of distinguishing questions that identifies the
preferred distortion riskmetrics within the candidate class and prove that the convergence rate of the algorithm is of order $O\left(\exp\left(-cm+O\left(\sqrt{m\log m}\right)\right)\right)$ under general settings, where $c>0$ is a constant and $m$ denotes the number of algorithm iterations. On the optimization side, we develop a model-free RL algorithm for optimizing policies under conditional distortion riskmetrics. By representing the objective as an integral of the conditional cost quantile function with respect to the distortion function, the method unifies all distortion-riskmetric objectives. We optimize diverse risk objectives by extending the Proximal Policy Optimization (PPO) algorithm with policy, value, and quantile neural networks, where the quantile network estimates the full conditional cost quantile function and enables numerical evaluation of the general risk objectives. A comprehensive
empirical study demonstrates the framework’s elicitation accuracy and effectiveness in complex financial environments.
}%

%\FUNDING{This research was supported by [grant number, funding agency].}

%Supplemental Material:
%Data Ethics & Reproducibility Note:

%\KEYWORDS{Stochastic programming, Decision support,Uncertainty, Disaster response, Optimization}

%\HISTORY{Received: Month DD, YYYY; Accepted: Month DD, YYYY; Published Online: Month DD, YYYY}

\maketitle
%%%%%%%%%%%%%%%%%%%%%%%%%%%%%%%%%%%%%%%%%%%%%%%%%%%%%%%%%%%%%%%%%%%%%%

% Text of your paper here

\section{Introduction}
In risk-sensitive decision-making,  it is important to specify which risk objective the agent aims to optimize. This identification problem is crucial in personalized applications such as robo-advising (\cite{alsabah2021robo}, \cite{capponi2022personalized}) and autonomous driving (\cite{majumdar2017risk}), where agents may exhibit heterogeneous and complex risk preferences. Existing studies often focus on mean-variance objectives or coherent risk measures,
which cover only part of the range of risk preferences observed in practice. Real-world preferences may involve quantiles, deviation-based objectives, asymmetric tail concerns, or dispersion control, which often fall outside the convex or coherent class. In quantitative risk management, \cite{WWW20} proposed distortion riskmetrics as a unified framework covering many risk measures, deviation measures, and other functionals for financial risk. \cite{glynn2021computing} proposed a sensitivity estimator for the distortion risk
measure. \cite{DLZ22} further introduced conditional distortion risk measures and distortion risk contribution measures to quantify systemic risk. Despite the theoretical generality and unifying power of distortion riskmetrics, eliciting the appropriate risk objective remains challenging, since agents are often unable to precisely specify their own risk preferences. A common approach is to infer such preferences from agents’ observed behavior. Inverse reinforcement learning (IRL) estimates an agent's objective from its policy (\cite{arora2021survey}) and has been applied to autonomous driving (\cite{majumdar2017risk}, \cite{kuderer2015learning}, \cite{sadigh2016planning}), legged robot locomotion (\cite{zucker2010optimization}, \cite{park2013inverse}), and robo-advising (\cite{cheng2023eliciting}). While these frameworks are promising, they assume that agents’ choices are always consistent with their risk preferences, which may not hold in practice.

Once the risk objective is identified, the next step is to optimize decisions under that objective. Risk-sensitive reinforcement learning (RSRL) has attracted growing attention for decision making under uncertainty. Unlike traditional risk-neutral RL, which optimizes expected cumulative rewards, RSRL explicitly accounts for the variability and tail behavior of returns, aiming to mitigate potentially catastrophic outcomes. Prior work has demonstrated its effectiveness in option hedging (\cite{buehler2019deep}, \cite{peng2024risk}), statistical arbitrage (\cite{coache2024reinforcement}, \cite{han2025risk}), marketing (\cite{chow2018risk}), pandemic control (\cite{du2025transfer}), autonomous driving (\cite{stachowicz2024racer}, \cite{bernhard2019addressing}, \cite{kamran2020risk}), robotics and control (\cite{zhang2021mean}), and other domains. However, existing RSRL methods mainly focus on a limited set of risk measures or scoring functions, especially convex or coherent ones; see \cite{han2025risk}. This leaves a gap between the broad class of risk objectives that may be elicited from agents and the narrower objectives currently supported by RL algorithms. Integrating general distortion riskmetrics into RL therefore provides a natural way to connect risk-objective identification with risk-sensitive policy optimization.

In this paper, we develop a noise-robust elicit-to-optimize framework that integrates IRL and RL for risk-preference elicitation and risk-sensitive decision making under a broad class of distortion riskmetrics. The framework first elicits the agents' latent risk preferences from their noisy observed choices and then uses the elicited distortion riskmetrics as the objectives for subsequent policy optimization. To reduce the cognitive burden on agents, we collect agents' choices through binary-choice questions. An adaptive framework (\cite{buning2022interactive}, \cite{cheng2023eliciting}) is then applied, which dynamically adjusts the questions according to the previous responses. We establish the existence of a finite set of distinguishing questions, characterize their discrimination power, and derive an upper bound for the power. Since agents may act inconsistently with their true risk preferences, the observed choices are inherently noisy. To address this issue, we propose a Bayesian IRL method that identifies the agent's preferred distortion riskmetric from stochastic choices. We prove that, under mild conditions, 
%our algorithm effectively identifies the agent's risk preference from a very broad class of distortion riskmetrics despite stochastic choice-making in a general setting. In particular, 
our algorithm converges to the agents' preferred distortion riskmetrics at the rate $O\left(\exp\left(-cm+O\left(\sqrt{m\log m}\right)\right)\right)$, where $c>0$ is a constant and 
$m$ denotes the number of adaptive questions. Compared with \cite{cheng2023eliciting}, whose IRL method achieves an inverse-square-root convergence rate up to logarithmic factors because it identifies more structural components, our approach is complementary: it covers a broad class of distortion riskmetrics and explicitly accounts for stochastic decision noise.   Given the elicited risk objective,  we develop a model-free RL algorithm that minimizes costs under the corresponding conditional distortion riskmetric. By representing the objective as a definite integral of the cost quantile function with respect to the distortion function on $[0,1]$, the RL method unifies all distortion-riskmetric objectives. We extend the Proximal Policy Optimization (PPO) algorithm (\cite{schulman2017proximal}) by introducing three neural networks into the RL framework: a policy network, a value network, and a quantile network. Unlike \cite{peng2024risk}, which incorporates a VaR network to estimate a fixed tail quantile, our quantile network approximates the entire cost quantile function, allowing general distortion-riskmetric objectives to be evaluated by midpoint Riemann sums. By considering conditional distortion riskmetrics, the proposed RL method accounts for scenario-dependent losses and strengthens risk control across heterogeneous environments.

Our contribution can be summarized as follows: (i) We extend the conditional distortion risk measures studied in \cite{DLZ22} to the broader class of distortion riskmetrics by allowing the distortion function to be non-monotone.  (ii) We propose a Bayesian IRL framework for identifying a broad class of distortion riskmetrics through adaptive interactive questioning. The method accommodates stochastic choices and therefore applies to settings in which agents do not always act consistently with their latent risk preferences. Under mild conditions, we prove that the proposed IRL algorithm converges at the rate $O\left(\exp\left(-cm+O\left(\sqrt{m\log m}\right)\right)\right)$, where $c>0$ is a constant and 
$m$ denotes the number of adaptive questions. Extensive numerical experiments verify the theoretical results and effectiveness of our IRL algorithm. (iii) We develop a model-free RL algorithm based on an extension of the Proximal Policy
Optimization (PPO), for optimizing objectives defined by conditional distortion riskmetrics. Our algorithm accommodates a wide range of conditional distortion-riskmetric objectives, thereby capturing diverse attitudes toward downside risk, variability, and tail events within a unified framework. Moreover, by incorporating scenario dependence into the risk objective, the resulting policies perform effectively in complex and dynamic environments. Numerical experiments further validate the performance of the proposed RL algorithm.

The rest of the paper is organized as follows. Section \ref{sec: Conditional Distortion Riskmetrics} introduces the notion and properties of distortion riskmetrics and conditional distortion riskmetrics. Section \ref{sec:Estimation of the Distortion Function} establishes the IRL algorithm for eliciting the true underlying risk preference, provides important quantities for the theoretical results, and proves the convergence rate of our algorithm. Section \ref{sec: Risk-Sensitive Reinforcement Learning} develops an RL algorithm  to solve the decision optimization problem under different distortion-riskmetric objectives. Section \ref{sec:Numerical Experiment} illustrates the performances of our proposed algorithms by numerical experiments, and Section \ref{sec:conslusion} concludes the paper.

\section{Notation and Preliminaries}\label{sec: Conditional Distortion Riskmetrics}
Let $(\Omega,\mathcal{F}, \p)$ be an atomless probability space. Let $(\mathcal{F}_t)_{t\ge 0}$ denote the information available up to time $t$.  For any $p\in[1,\infty)$, $L^p$ is the space of random variables with finite $p$-th moment, and $L^\infty$ is the space of essentially bounded random variables.  
Let $\X\supset L^\infty$ be a convex cone and $\M$ denote the set of distribution functions of random variables in $\X$. For $F\in\M$,   $X\sim F$ means that $X\in\X$ has distribution $F$. Denote by $F_X$   the distribution function of the random variable $X$. Equality in distribution is denoted by $X \overset{\mathrm{d}}{=} Y$. Accordingly, $X \overset{\mathrm{d}}{\ne} Y$ indicates that $X$ and $Y$ have different distributions. We define the left-continuous generalized inverse of $F$ (left-quantile) as
 $$F^{-1}(t)=\inf\{x\in\R:F(x)\geq t\},\quad t\in (0,1],$$
 while the right-continuous generalized inverse of $F$ (right-quantile) is defined as
 $$F^{-1+}(t)=\inf\{x\in\R:F(x)> t\},\quad t\in [0,1).$$
 Conventionally, we define $F^{-1}(0)= F^{-1+}(0)$ and 
 $F^{-1+}(1)=F^{-1}(1)$. 
 
The conditional distortion riskmetric is defined in terms of the distortion riskmetric studied in \cite{WWW20} on a general space denoted by  
$$\H= \{h: h \text{ maps }[0,1] \text{ to } \R,~ h \text{ is of bounded variation, } h(0)=0 \}.$$
For convenience, we include the definition of the distortion riskmetric  below. 
 \begin{definition}\label{def:distortion-risk-measure}
  A functional $\rho_h:\X\to\R$, whose domain $\X\supset L^\infty$ is a law-invariant convex cone, is a \emph{distortion riskmetric} if there exists $h\in \H$ such that $\rho_h(X)=\int X\d h\circ\P$, where $\int X\d h\circ\P$ is a signed Choquet integral defined by
 	\begin{equation}\label{eq:distortion-riskmetric-unconditional}
 		\int X\d h\circ\P=\int_{-\infty}^{0}\left(h(\P(X\geq x))-h(1)\right)\d x+\int_0^\infty h(\P(X\geq x))\d x.
 		\end{equation}
 	The function $h$ is called the \emph{distortion function} of $\rho_h$.
    %We denote the distortion riskmetric with distortion function $h$ as $\rho_h$.
 \end{definition} 
For a given distortion riskmetric, the distortion function is unique. For example, if $h(t)=\mathbf{1}_{\{t>1-\alpha\}}$, the distortion riskmetric is Value-at-Risk (VaR$_\alpha$): $\rho_h(X)=F^{-1}_X(\alpha)$, where $\alpha\in(0,1)$.
More examples can be found in \cite{WWW20}. In the following, we provide the definition of a conditional distortion riskmetric.

\begin{definition}\label{cond-distortion}
Let $
\mathcal C:=\{C\in\mathcal F:\mathbb P(C)>0\}.
$
For
$C\in\mathcal C$, define
$
\mathbb P_C(A):=\mathbb P(A\mid C),~ A\in\mathcal F.$ A \textit{conditional distortion riskmetric}  $\rho_{h|\cdot}:\mathcal X\times\mathcal C\to\mathbb R$ is defined by
$$
\rho_{h|\cdot}(X,C)=\rho_{h|C}(X):=\int X\,\d h\circ\mathbb P_C.
$$
Equivalently,
$$
\begin{aligned}
\rho_{h|C}(X)
=
&
\int_{-\infty}^{0}
\left(
h\left(
\mathbb P(X\ge x\mid C)
\right)
-h(1)
\right)\d x +
\int_{0}^{\infty}
h\left(
\mathbb P(X\ge x\mid C)
\right)\d x .
\end{aligned}
$$

% For a fixed event $C$, we simply write
% $$
% \rho_{h|C}:\mathcal X\to\mathbb R.
% $$
% For $X\in\mathcal X$, the event-conditioned distortion riskmetric is defined by
% $$
% \rho_{h|C}(X)
% :=
% \int X\,dh\circ\mathbb P_C.
% $$
% Equivalently,
% $$
% \begin{aligned}
% \rho_{h|C}(X)
% =
% &
% \int_{-\infty}^{0}
% \left(
% h\left(
% \mathbb P(X\ge x\mid C)
% \right)
% -h(1)
% \right)dx 
% +
% \int_{0}^{\infty}
% h\left(
% \mathbb P(X\ge x\mid C)
% \right)dx .
% \end{aligned}
% $$

% If the conditioning event is generated by a condition process
% $Y_{0:T}$, then $C$ may be chosen as
% $$
% C=C_{y_{0:T}}
% :=
% \{Y_{0:T}=y_{0:T}\}
% =
% \{Y_0=y_0,\ldots,Y_T=y_T\}.
% $$
% For threshold-type conditions, one may instead take
% $$
% C=C_{y_{0:T}}^{\ge}
% :=
% \{Y_0\ge y_0,\ldots,Y_T\ge y_T\}.
% $$
\end{definition}

\begin{comment}
\begin{remark}
	A conditional distortion riskmetric $\rho_{h|y}$ can be equivalently expressed as
	\begin{equation}
	\label{eq:alt}
	\rho_{h|y}(X| Y)=\int_{-\infty}^{0}\left(h(\P(X> x|Y \ge y))-h(1)\right)\d x+\int_0^\infty h(\P(X> x|Y \ge y))\d x.
	\end{equation}
	Indeed, since $\P(X>x|Y \ge \rho_g(Y))=\P(X\ge x|Y \ge \rho_g(Y))$ almost everywhere on $\R$.
\end{remark}
\end{comment}
 
 \begin{comment}
 \begin{remark}
The conditional distortion riskmetric is a $\mathcal{Q}$-based risk measure defined in \cite{WZ21} with $\mathcal{Q}=\{\p[\cdot|Y\ge\rho_g(Y)]\}$.
 \end{remark} 
\end{comment}
Below we obtain the quantile representation of conditional distortion riskmetrics, which can be seen as the parallel results of Lemma 1 in \cite{WWW20}. 
For $X\in L^0$, denote by $F_{X|C}$ the conditional distribution function of
$X$ under $\mathbb P_C$.
\begin{lemma}\label{lemma:representation-of-conditional-distortion-riskmetrics}
For $h\in\mathcal H$, $C\in\mathcal C$, and $X\in L^0$ such that
$\rho_{h|C}(X)$ is well-defined, possibly taking values in $\pm\infty$, the
following statements hold.

(i) If $h$ is right-continuous, then
$$
\rho_{h|C}(X)
=
\int_0^1
F_{X|C}^{-1+}(1-u)\,\d h(u).
$$

(ii) If $h$ is left-continuous, then
$$
\rho_{h|C}(X)
=
\int_0^1
F_{X|C}^{-1}(1-u)\, \d h(u).
$$

(iii) If $F_{X|C}^{-1}$ is continuous on $(0,1)$, then
$$
\rho_{h|C}(X)
=
\int_0^1
F_{X|C}^{-1}(1-u)\, \d h(u)
=
\int_0^1
F_{X|C}^{-1+}(1-u)\, \d h(u).
$$
\end{lemma}
\begin{proof}
\noindent
Note $\mathbb P_C$ is a probability measure on $(\Omega,\mathcal F)$.  The result follows by applying the unconditional quantile
representation (Lemma 1 in \cite{WWW20}) under the probability measure $\mathbb P_C$. In particular, replace
$\mathbb P$  by $\mathbb P_C$,  $F_X$ by $F_{X|C}$, and $
\rho_h$ by $\rho_{h|C}.$
This gives (i), (ii), and (iii).
\end{proof}

\section{Noise-Robust IRL Approach for Elicitation}\label{sec:Estimation of the Distortion Function}
\subsection{Setup}
In this section, we investigate the IRL problem with the goal of recovering the agent’s preferred distortion function. We consider a finite candidate set of distortion functions:
\begin{equation*}
    \hat{\mathcal{H}}:=\{h_1(t),h_2(t),\cdots,h_L(t) \}\subset \mathcal{H},
\end{equation*}
where $t\in [0,1]$ and $|\hat{\H}|=L$. The reason for the restriction to a finite set of distortion functions is that different distortion functions may induce the same ordering over the
available actions and are therefore indistinguishable from observed choices. Many applications naturally describe risk preferences through a finite set of representative types. For example, driving styles in autonomous driving are often classified as conservative, moderate, or aggressive. In robo-advising, clients are commonly assigned to a finite set of risk
profiles or portfolio types.
Thus it
is sufficient to recover a representative distortion function from a finite set that best
matches the observed behavior. 

%By iteratively querying the agents to elicit their choices, we select from this finite set $\hat{\mathcal{H}}$ the distortion function that best matches the agent's true preference. 
%Because distortion functions and distortion riskmetrics are in one-to-one correspondence, we identify the agent's preferred distortion function by identifying the agent's preferred distortion riskmetric.
We consider binary-choice problems with action space $\mathcal{A}'=\{a_1,a_2\}$.
\begin{comment}
The cost function is:
\begin{equation}
    c(S_t,A_t,S_{t+1})=S_t-S_{t+1}.
\end{equation}
\end{comment}
  Let $\widehat{Q}_l^{(n)}$ denote the posterior probability at round $n$, for $l=1,2,\cdots, L$. We initialize the prior as $\widehat{Q}_l^{(0)}=1/L$ for all $l$, and update the probability vector $\{\widehat{Q}_l^{(n)}:l=1,\ldots,L\}$ after each observed response. The binary-choice question at round $n$ is defined as: $G^{(n)}:=\left\{(X_n, Y_n): X_n\overset{\mathrm{d}}{\ne} Y_n, X_n,Y_n\in L^\infty\right\}$. Let $C(a^{(n)})\in \{X_n,Y_n\}$ denote the choice under $a^{(n)}\in\mathcal{A}'$. Hence, its risk value under the distortion risk measure $\rho_{h_l}$ is given by
\begin{equation}\label{eq:value-function-Terminal-any-action}
    \begin{aligned}
    \rho_{h_l}(C(a^{(n)}))=  \int_{-\infty}^{0}\left(h_l\left(\P\left( C(a^{(n)})\geq x \right)\right)-h_l(1)\right)\d x 
    +\int_0^\infty h_l \left(\p \left(C(a^{(n)}) \geq x \right)\right) \d x,
    \end{aligned}
\end{equation}
 Binary-choice questions arise naturally in many applications. For example, in the context of portfolio investment, a typical question with two choices to identify the risk preference of an investor takes the form in Example \ref{example:question}.

\begin{comment}
Then we can calculate the probability that the true distortion function is $h_l(\cdot)$:
\begin{equation}\label{eq:probability}
    Q_l^{(n)}=\frac{\exp \left(-\eta \sum^{n}_{p=1}\Phi(a^{(p)};G^{(p)},h_l)\right)}{\sum_{l=1}^{L}\exp \left(-\eta \sum^{n}_{p=1}\Phi(a^{(p)};G^{(p)},h_l)\right)},
\end{equation}
where $\eta$ is the learning rate. When $n=0$, we set $Q_l^{(0)}=\frac{1}{L}$ for $l\in \{1,2,\cdots,L\}$.
\end{comment}

\begin{example}\label{example:question}
    \textit{Here are two portfolios A and B (see Tables \ref{tab:sales_q1}-\ref{tab:sales_q2}) in the current financial market. The first row of each table shows the cumulative probability, and the second row shows the corresponding loss quantile. Negative losses represent gains. 
 Which of the two portfolios would you choose in the current market environment?}
 \begin{table}[htbp]
    \centering
    
    \caption{Portfolio A}
    \label{tab:sales_q1}
    \begin{tabular}{@{}l*{11}{r}@{}}
        \toprule
        Cumulative probability & 0\% & 10\% & 20\% & 30\% & 40\% & 50\% & 60\% & 70\% & 80\% & 90\% & 100\% \\
        \midrule
        Loss (units) & -10 & -5 & -3.7 & -2.5 & -1 & 0 & 3 & 5 & 7 & 9 & 13 \\
        \bottomrule
    \end{tabular}
    
    \vspace{0.8em} % 控制上下表格之间的垂直间距
    
    \caption{Portfolio B}
    \label{tab:sales_q2}
    \begin{tabular}{@{}l*{11}{r}@{}}
        \toprule
        Cumulative probability & 0\% & 10\% & 20\% & 30\% & 40\% & 50\% & 60\% & 70\% & 80\% & 90\% & 100\% \\
        \midrule
        Loss (units) & -12 & -11 & -6 & -1.4 & -0.3 & 0.5 & 4 & 8 & 9 & 9.5 & 10 \\
        \bottomrule
    \end{tabular}
\end{table}
 \begin{comment}
    \begin{table}[htbp]
    \centering % Center the entire group of tables
    \begin{minipage}{0.48\textwidth} % First table container
        \centering % Center the content inside the minipage
        \caption{Portfolio A}
        \label{tab:sales_q1}
        \begin{tabular}{@{}ccc@{}} % @{} removes extra space at the edges
            \toprule
            Cumulative probability & Loss ($C$)  \\
            \midrule
            0\%  & -100  \\
            10\% & -50  \\
            20\%  & -37  \\
            30\% & -25   \\
            40\% & -10   \\
            50\% & 0   \\
            60\% & 30   \\
            70\% & 50   \\
            80\% & 70   \\
            90\% & 90   \\
            100\% & 130   \\
            \bottomrule
        \end{tabular}
    \end{minipage}
    \hfill % Crucial: Fills the horizontal space between the minipages
    \begin{minipage}{0.48\textwidth} % Second table container
        \centering
        \caption{Portfolio B}
        \label{tab:sales_q2}
        \begin{tabular}{@{}ccc@{}}
            \toprule
            Cumulative probability & Loss ($C$) \\
            \midrule
            0\%  & -120  \\
            10\% & -110  \\
            20\%  & -60  \\
            30\% & -14   \\
            40\% & -3   \\
            50\% & 5   \\
            60\% & 40   \\
            70\% & 80   \\
            80\% & 90   \\
            90\% & 95   \\
            100\% & 100   \\
            \bottomrule
        \end{tabular}
    \end{minipage}
\end{table}
\end{comment}
\end{example}

If an agent is perfectly rational, he or she will always choose an action that minimizes the risk value induced by his or her preferred distortion riskmetric:
\begin{equation*}
    a^{(n)}= \underset{a\in \mathcal{A}'}{\arg\min} \rho_{h^*}(C(a)),
\end{equation*}
where $\rho_{h^*}$ is the agent's preferred distortion riskmetric. 
%In Example \ref{example:question}, suppose the distortion risk measure preferred by an absolutely rational agent is VaR at confidence level 0.8, the agent will invariably choose portfolio A that exhibits a smaller $\text{VaR}_{0.8}$ of expected loss. In such scenarios, the agent's observed policy faithfully aligns with his or her own risk preferences, allowing the IRL algorithm to converge to the true distortion function.
%However, in practice, various factors lead to agents not behaving fully rationally when making decisions: \begin{itemize}\item As time goes by, agents may experience fatigue and select the suboptimal options with non-negligible probability.\item Agents often do not have full awareness of their own risk preferences, which typically become clearer only gradually through the interactive querying process.\item If the risk values of the two options under agents' preferred risk measure differ only marginally, they may not make choices strictly according to their preferred measure.\end{itemize}
%Frequent suboptimal choices by the agent inject misleading information into the observations, systematically biasing the traditional IRL procedure away from recovering the true distortion function. Therefore, when we construct the IRL algorithm, we will explicitly account for the uncertainty inherent in the agent's decision-making. 
However, in practice, observed choices may deviate from this deterministic optimality rule. Such deviations may arise from fatigue during repeated questioning, limited awareness of latent risk preferences, or near indifference when the two choices have similar risk values under $\rho_{h^*}$. Therefore, agents' responses should be modeled as stochastic rather than perfectly rational choices. This motivates the development of noise-robust IRL algorithms for recovering the preferred distortion riskmetric from uncertain choice data.

\subsection{Distinguishing Method}
\begin{comment}
In each round $n$, two distortion functions $h_i$ and $h_j$ are drawn from the set $\hat{\H}$, yielding the corresponding distortion risk measures $\rho_{h_i}$ and $\rho_{h_j}$. We then construct $G^{(n)}$ consisting of alternatives $X_n$ and $Y_n$ that receive different risk evaluations under $\rho_{h_i}$ and $\rho_{h_j}$, and ask the client to select between them. The risk values of the alternatives under both $\rho_{h_i}$ and $\rho_{h_j}$ are provided in the question.
\end{comment}
We measure the distance between two distortion functions by $L^\infty$ distance:
\begin{equation*}
||h_i-h_j||_\infty=\sup_{t\in[0,1]}|h_i(t)-h_j(t)|.
\end{equation*} 
The following proposition establishes the quantitative relationship between the $L^\infty$ distance of distortion functions and the $L^\infty$ distance of distortion riskmetrics.  
\begin{proposition}\label{prop:bound-of-risk-measure}
    For all $h_i, h_j\in \mathcal{H}$ and $X\in L^\infty$, we have:
    \begin{equation*}
        |\rho_{h_i}(X)-\rho_{h_j}(X)|\le (|U|+2|M|)||h_i-h_j||_\infty,
    \end{equation*}
    where $U\ge\max\{0, \esssup X\}$ and $M\le\min\{0, \essinf X\}$.
    \end{proposition}  
\begin{proof} {Proof of Proposition \ref{prop:bound-of-risk-measure}}
    Without loss of generality, we assume that $M<0$. We set $\Delta h_{ij}=h_i-h_j$. According to \eqref{eq:distortion-riskmetric-unconditional}, we have that
    \begin{equation*}
    \begin{aligned}
        \rho_{\Delta h_{ij}}(X)&=\int_{-\infty}^{0}\left(\Delta h_{ij}(\P(X\geq x))-\Delta h_{ij}(1)\right)\d x+\int_0^\infty \Delta h_{ij}(\P(X\geq x))\d x\\
        &=\int_{M}^{0}\left(\Delta h_{ij}(\P(X\geq x))-\Delta h_{ij}(1)\right)\d x+\int_0^{U} \Delta h_{ij}(\P(X\geq x))\d x.
        \end{aligned}
    \end{equation*}
    Note that 
    $
        ||h_i-h_j||_\infty=\sup_{t\in[0,1]}|h_i(t)-h_j(t)|
        =\sup_{t\in[0,1]}|\Delta h_{ij}(t)|.
    $
    So
    \begin{equation*}
    \begin{aligned}
        \int_0^{U} \Delta h_{ij}(\P(X\geq x))\d x&\le \bigg|\int_0^{U} \Delta h_{ij}(\P(X\geq x))\d x  \bigg|\\
        &\le \int_0^{U}||h_i-h_j||_\infty\d x
        =|U| ||h_i-h_j||_\infty.
    \end{aligned}
    \end{equation*}
    Since
    \begin{equation*}
    \begin{aligned}
        \int_{M}^{0}\left(\Delta h_{ij}(\P(X\geq x))-\Delta h_{ij}(1)\right)\d x&\le \int_{M}^{0}\Big|\Delta h_{ij}(\P(X\geq x))-\Delta h_{ij}(1)\Big|\d x\\
        &\le \int_{M}^{0}\left(\Big|\Delta h_{ij}(\P(X\geq x))\Big|+\Big|\Delta h_{ij}(1)\Big|\right)\d x\\
        &\le \int_{M}^{0} 2||h_i-h_j||_\infty\d x
        =2|M|||h_i-h_j||_\infty,
    \end{aligned}
    \end{equation*}
    we have
    $|\rho_{h_i}(X)-\rho_{h_j}(X)|\le (|U|+2|M|)||h_i-h_j||_\infty$.
%which leads to $\sup_{X\in L^\infty}|\rho_{h_i}(X)-\rho_{h_j}(X)|\le (|U|+2|M|)||h_i-h_j||_\infty.$

\end{proof}

Proposition \ref{prop:existence-of-distinguishing-question} guarantees the existence of a distinguishing question for any pair of non-proportional distortion functions.
\begin{proposition}\label{prop:existence-of-distinguishing-question}
    For all $h_i$, $h_j\in \H$, if $h_i\ne c h_j$ for all $c\ge 0$ and $h_j\not\equiv 0$, there exist $X,Y\in L^\infty$, such that
\begin{equation*}
        \rho_{h_i}(X)<\rho_{h_i}(Y), \quad
    \rho_{h_j}(X)>\rho_{h_j}(Y).
    \end{equation*}
\end{proposition}
\begin{proof}{Proof of Proposition \ref{prop:existence-of-distinguishing-question}}
For any distortion function $h$, we define the linear functional $L_h$ on quantile functions $q(u)$ by
\begin{equation*}
L_h(q)=\int_0^1 q(1-u)\,\d h(u).
\end{equation*}
By the quantile representation, if a random variable $Z$ has the quantile function $q_Z$, then $\rho_h(Z)=L_h(q_Z).$
We first show that there exists a bounded continuous function $r$ such that $L_{h_i}(r)<0$ and $L_{h_j}(r)>0.$
If not, then every $r$ satisfying $L_{h_j}(r)>0$ must also satisfy $L_{h_i}(r)\ge 0$. 
Since \(h_j\not\equiv 0\), the functional \(L_{h_j}\) is not identically zero.
Hence there exists \(g\in C[0,1]\) such that $L_{h_j}(g)>0.$ Now take any \(f\in \ker L_{h_j}\), so that \(L_{h_j}(f)=0\). We claim that
\(L_{h_i}(f)=0\).  If \(L_{h_i}(f)>0\), we set $r_\lambda=g-\lambda f$. Then for sufficiently large
\(\lambda>0\), we have
\[
L_{h_j}(r_\lambda)=L_{h_j}(g)-\lambda L_{h_j}(f)=L_{h_j}(g)>0, \quad L_{h_i}(r_\lambda)=L_{h_i}(g)-\lambda L_{h_i}(f)<0,
\]
which contradicts the assumption. Similarly, if \(L_{h_i}(f)<0\), then for
sufficiently large \(\lambda>0\), by setting $r_\lambda=g+\lambda f$, there is also a contradiction.
Hence, we have $\ker L_{h_j}\subseteq \ker L_{h_i},$ and 
there exists a constant $c\ge 0$ such that
\begin{equation*}
L_{h_i}=cL_{h_j},
\end{equation*}
which implies $h_i=ch_j$.
This contradicts the assumption that $h_i\ne c h_j$ for all $c\ge 0$. Therefore, there exists a bounded continuous function $r\in C[0,1]$ such that $L_{h_i}(r)<0$ and $L_{h_j}(r)>0$.  Since \(C^1[0,1]\) is dense in \(C[0,1]\) and the above inequalities are
strict,
we may choose such an \(r\) in \(C^1[0,1]\) such that $L_{h_i}(r)<0$ and $L_{h_j}(r)>0.$
Choose a constant $M>\|r'\|_\infty$ and define
\begin{equation*}
q_Y(t)=Mt,\qquad
q_X(t)=Mt+r(t),
\qquad t\in[0,1].
\end{equation*}
Then $q_Y$ is strictly increasing. Moreover, we have
\begin{equation*}
q_X'(t)=M+r'(t)\ge M-\|r'\|_\infty>0,
\end{equation*}
so $q_X$ is also strictly increasing. Hence $q_X$ and $q_Y$ are valid bounded quantile functions.
We define 
$X=q_X(U)$ and 
$Y=q_Y(U)$, where $U\sim \mathcal{U}(0,1)$.
Then $X,Y\in L^\infty$, and their quantile functions are $q_X$ and $q_Y$.
Using the quantile representation, we obtain
\begin{equation*}
\rho_{h_i}(X)-\rho_{h_i}(Y)
=
L_{h_i}(q_X)-L_{h_i}(q_Y)
=
L_{h_i}(q_X-q_Y)
=
L_{h_i}(r)<0.
\end{equation*}
Therefore, we have $\rho_{h_i}(X)<\rho_{h_i}(Y)$.
Similarly, we can show that $\rho_{h_j}(X)>\rho_{h_j}(Y)$. Hence, we complete the proof.
\end{proof}

Theorem \ref{thm:finite-time-existence-of-the-question} guarantees finite identifiability: for all $h_i\in\hat{\mathcal{H}}$, a finite set of distinguishing questions is sufficient to recover $h_i$ within the candidate class $\hat{\mathcal{H}}$.
\begin{theorem}\label{thm:finite-time-existence-of-the-question}
\begin{comment}
Let $PC^1([0,1])$ stand for the set of all continuous functions on $[0,1]$ which are $C^1$ on finitely many sub-intervals. We define
\begin{align}
    \mathcal{H}_c=\left\{h\in \mathcal H\cap PC^1([0,1]): 
 \exists q>1,\ \exists M<\infty
\text{ such that }
\sup_h\|h'\|_{L^q([0,1])}\le M\right\},
\end{align}
and $\H_d$ as a finite set of discontinuous functions in $\H$. Let $\H'=\H_c\cup \H_d$.
\end{comment}
    We define 
    \begin{equation*}
        K_{\delta,i}=\left\{h\in \hat{\H}: \delta\le ||h-h_i||_{\infty}, h\ne ch_i\ \text{for every} \ c\ge 0\right\}
    \end{equation*}
    for all $\delta>0$ such that $K_{\delta,i}\ne \emptyset$.
    For each distortion function $h_{i}\in \hat{\H}$ and $h_i\not\equiv 0$, there exists a constant $M\in \mathbb{N}^*$, for all $h\in K_{\delta,i}$,  there exist $k,l\in \{1,2,\cdots,M\}$ and $k\ne l$, such that
    \begin{equation*}
        \rho_{h}(X_k)<\rho_{h}(X_l),\ \rho_{h_{i}}(X_k)>\rho_{h_{i}}(X_l).
    \end{equation*}
\end{theorem}
\begin{comment}
    Define
    \begin{align}
        \phi(h)=D_{L^\infty}(h,h_{i}).
    \end{align}
    Since
    \begin{align}
        |\phi(h)-\phi(g)|=|||h-h_{i}||_\infty-||g-h_{i}||_\infty|
        \le ||h-g||_\infty,
    \end{align}
    $\phi(h)$ is continuous, and the inverse function $\phi^{-1}([\delta,\infty))$ on $[\delta,\infty)$ is closed. Since
    \begin{align}
        K_{\delta,i}=\hat{\H}\cap \phi^{-1}([\delta,\infty)),
    \end{align}
    $K_{\delta,i}$ is a closed subset in $\hat{\H}$. Since $\hat{\H}$ is compact,
\end{comment}
\begin{proof}{Proof of Theorem \ref{thm:finite-time-existence-of-the-question}}
    %By the one-dimensional Sobolev embedding, $\H_c$ is equicontinuous. So by Arzela-Ascoli Theorem, $\H_c$ is compact. Because $\H_d$ is a finite set, it is also compact. 
    Since $\hat{\H}$ is a finite set,  
    $K_{\delta,i}$ is finite and compact. Because $h\in K_{\delta,i}$, 
    %following the proof of Theorem \ref{thm:existence-of-the-question}, 
    we can construct $X_h, Y_h\in L^\infty$, such that $\rho_{h}(X_h)<\rho_{h}(Y_h)$ and $\rho_{h_{i}}(X_h)>\rho_{h_{i}}(Y_h)$ by Proposition \ref{prop:existence-of-distinguishing-question}.
    We fix $X_h$ and $Y_h$, and we set 
        $\Delta_h=\min \{|\rho_{h}(X_h)-\rho_{h}(Y_h)|,|\rho_{h_{i}}(X_h)-\rho_{h_{i}}(Y_h)| \}.$
    By Proposition \ref{prop:bound-of-risk-measure}, there exists $\eta_h>0$ such that $||g-h||_\infty<\eta_h$, and
    \begin{equation*}
        |\rho_{g}(X_h)-\rho_{h}(X_h)|<\frac{\Delta_h}{4}, \ |\rho_{g}(Y_h)-\rho_{h}(Y_h)|<\frac{\Delta_h}{4}.
    \end{equation*}
    Note that 
    \begin{equation*}
        \big|\rho_{g}(X_h)-\rho_{g}(Y_h)-(\rho_{h}(X_h)-\rho_{h}(Y_h))\big|\le |\rho_{g}(X_h)-\rho_{h}(X_h)|+|\rho_{g}(Y_h)-\rho_{h}(Y_h)|
        <\frac{\Delta_h}{2}.
    \end{equation*}
    We have
    \begin{equation*}
        \rho_{g}(X_h)-\rho_{g}(Y_h)<\rho_{h}(X_h)-\rho_{h}(Y_h)+\frac{\Delta_h}{2}
        \le -\Delta_h+\frac{\Delta_h}{2}
        <0.
    \end{equation*}
    We construct an open set
        $U_h=\{g\in K_{\delta,i}:||g-h||_\infty<\eta_h\}.$
    Hence, for all $g\in U_h$, there exists $X_h, Y_h\in L^\infty$, such that $\rho_{g}(X_h)<\rho_{g}(Y_h)$ and $\rho_{h_{i}}(X_h)>\rho_{h_{i}}(Y_h)$.
    Since $\{U_h\}_{h\in K_{\delta,i}}$ is an open cover of $K_{\delta,i}$ and $K_{\delta,i}$ is compact, there exist $h_1$, $h_2$, $\cdots$, $h_N\in K_{\delta,i}$, such that
    \begin{equation}\label{eq:union-question}
        K_{\delta,i} \subset \bigcup_{j=1}^{N} U_{h_j}.
    \end{equation}
    %Since $\hat{\H}$ is finite, we can find a global positive integer $N$ such that \eqref{eq:union-question} holds for all $h_i\in \hat{\H}$. 
    We set $M=2N$ and construct 
    \begin{equation*}
        \{X_1,X_2,\cdots,X_M\}=\{X_{h_j},Y_{h_j}:\rho_{h_j}(X_{h_j})<\rho_{h_j}(Y_{h_j}),\ \rho_{h_{i}}(X_{h_j})>\rho_{h_{i}}(Y_{h_j}),\ j=1,2,\cdots,N\}.
    \end{equation*}
    For all $h\in K_{\delta,i}$, we can find a set $U_{h_j}$, $j\in \{1,2,\cdots,N\}$, such that $h\in U_{h_j}$. And we can find $X_k$, $X_l\in \{X_1,X_2,\cdots,X_M\}$ where $k\ne l$, such that $\rho_{h}(X_k)<\rho_{h}(X_l)$ and $\rho_{h_{i}}(X_k)>\rho_{h_{i}}(X_l)$.
\end{proof}

According to Theorem \ref{thm:finite-time-existence-of-the-question}, we can construct a finite set of questions $\mathcal{K}$ to identify the agent's preferred distortion function:
\begin{equation*}
    \mathcal{K}:=\left\{(X_i,Y_i): X_i\overset{\mathrm{d}}{\ne}Y_i, X_i,Y_i\in L^\infty, i=1,2,\cdots,N\right\}.
\end{equation*}
We define the regret of the action $a^{(n)}$ for each distortion function $h_l$ under the binary-choice question $G^{(n)}\in \mathcal{K}$ at round $n$ as
\begin{equation}\label{eq:regret-of-action}
    \Phi(a^{(n)};G^{(n)},h_l)=\rho_{h_l}(C(a^{(n)}))-\min_{a\in\mathcal{A}'}\rho_{h_l}(C(a)).
\end{equation}
The regret $\Phi$ quantifies how well the distortion function $h_l$ aligns with the agent's observed choices. To quantify the discriminative power of the question,
we define the distinguishing power of $G^{(n)}$ as:
\begin{equation}
\begin{aligned}\label{eq:distinguish-power-of-the-environment}
\Psi(G^{(n)},i,j)=\Phi(a^{(n), *,i};G^{(n)},h_j)\Phi(a^{(n), *,j};G^{(n)},h_i),
    \end{aligned}
\end{equation}
where $a^{(n), *,i}$ and $a^{(n), *,j}$ represent the optimal actions for the distortion functions $h_i(t)$ and $h_j(t)$ respectively. It is straightforward to verify that
\begin{equation*}
    \Psi(G^{(n)},i,j)=\left(\rho_{h_i}(C(a^{(n), *,j}))-\rho_{h_i}(C(a^{(n), *,i}))\right)\left(\rho_{h_j}(C(a^{(n), *,i}))-\rho_{h_j}(C(a^{(n), *,j}))\right).
\end{equation*}
Obviously, $\Psi(G^{(n)},i,j)>0$ if and only if $a^{(n), *,i}\ne a^{(n), *,j}$.
If $ \Psi(G^{(n)}, i, j) = 0 $ for some pair $i \neq j  $, then $G^{(n)}$ cannot separate the distortion functions $  h_i  $ and $  h_j  $. 
To ensure that all distortion functions in $\hat{\H}$ can be distinguished, we impose the following condition on every pair $(X,Y)\in \mathcal{K}$.
\begin{assumption}\label{assumption:unique-optimal-action}
For all $h_i\in \hat{\H}$ and for all $(X,Y)\in \mathcal{K}$, we have $\rho_{h_i}(X) \ne \rho_{h_i}(Y)$.
\end{assumption}
\begin{comment}
    and whenever this condition is satisfied, the following inequality holds:
    \begin{align}
        J \leq \frac{\rho_{h_i}(X) - \rho_{h_i}(Y)}{\rho_{h_j}(Y) - \rho_{h_j}(X)} \leq \frac{1}{J},
    \end{align}
    where $1\ge J>0$ is independent of $h_i$, $h_j$, $X$ and $Y$.
\end{comment}
\begin{comment}
\begin{assumption}\label{assumption:lower-bound-of-diff-risk-measures}
    For all $h\in \hat{\H}$ and $(X,Y)\in \mathcal{K}$, there exist constants $\hat{\Delta}$ and $D$ such that 
    \begin{align}
        D\ge|\rho_{h}(X)-\rho_{h}(Y)|\ge \hat{\Delta}>0.
    \end{align}
\end{assumption}
\end{comment}
To improve the efficiency of the algorithm, we adopt an adaptive questioning method with exploration rate $\epsilon>0$. At the end of round $n-1$, we select the question for round $n$. With probability $1-\epsilon$, we select the question that best distinguishes the two distortion functions $h_{i^*}$ and $h_{j^*}$ with the highest and second-highest posterior probabilities until round $n-1$:
\begin{equation}\label{eq:environment-selection}
    G^{(n)}= \underset{G\in \mathcal{K}}{\arg \max} \Psi(G,i^*,j^*).
\end{equation}
With exploration rate $\epsilon$, we instead randomly choose two distinct distortion functions $h_i$ and $h_j$ from $\hat{\H}$ and select the question that best distinguishes them.
To quantify how effectively the question $G$ can discriminate between different distortion riskmetrics, we derive an upper bound on the distinguishing power $\Psi(G,i,j)$ given two different distortion functions in Proposition \ref{prop:upper-bound-of-distingushing-power}.

\begin{proposition}\label{prop:upper-bound-of-distingushing-power}
    Under Assumption \ref{assumption:unique-optimal-action}, let $(X,Y)\in \mathcal K$ and define $\ell=\max\{||X||_\infty, ||Y||_\infty\}$.
    For all $h_i, h_j\in \hat{\H}$, if there exists a constant $m\in \N^*$ and a partition: $0=t_0<t_1<\cdots<t_m=1$,
    such that $h_i(t)-h_j(t)$ is monotone on each interval $[t_{r-1},t_r]$ for $r=1,\cdots, m$,
    we have 
    \begin{equation*}
        \Psi(G,i,j)\le 4m^2\ell^2||h_i-h_j||_\infty^2.
    \end{equation*}
\end{proposition}
\begin{proof}{Proof of Proposition \ref{prop:upper-bound-of-distingushing-power}}
    Let $A=\rho_{h_i}(X) - \rho_{h_i}(Y)$, $B=\rho_{h_j}(Y) - \rho_{h_j}(X)$, and $g=h_i-h_j$. If $AB<0$, then we have $\Psi(G,i,j)=0$, and the proposition holds. If $AB>0$, without loss of generality, we let $A>0$ and $B>0$. 
    By \cite{WWW20}, we have
    \begin{equation*}
        A+B
        =\rho_g(X)-\rho_g(Y)
        \le\esssup |X-Y|\cdot  \text{TV}_g[0,1]
        \le 2\ell \cdot  \text{TV}_g[0,1],
    \end{equation*}
    where $\text{TV}_g[0,1]$ is the total variation of $g$ on [0,1]. Then we have
    \begin{equation*}
        AB\le \frac{(A+B)^2}{4}
        \le\frac{4\ell^2 (\text{TV}_g[0,1])^2}{4}
        =\ell^2 (\text{TV}_g[0,1])^2.
    \end{equation*}
    Since $g$ is of bounded variation on [0,1] and monotone on each interval $[t_{r-1},t_r]$ for $r=1,\cdots, m$, by additive property of total variation, 
    \begin{equation*}
        \text{TV}_g[0,1]=\sum^m_{r=1}\text{TV}_g[t_{r-1},t_r]
        =\sum^m_{r=1}|g(t_r)-g(t_{r-1})|
        \le \sum^m_{r=1}2||g||_\infty
        =2m ||h_i-h_j||_\infty,
    \end{equation*}
    where $\text{TV}_g[t_{r-1},t_r]$ is the total variation of $g$ on $[t_{r-1},t_r]$.
    We have
    \begin{equation*}
        \Psi(G,i,j)=(\rho_{h_i}(X) - \rho_{h_i}(Y))(\rho_{h_j}(Y) - \rho_{h_j}(X))
        \le 4m^2\ell^2 ||h_i-h_j||_\infty^2.
    \end{equation*}
\end{proof}

Proposition \ref{prop:upper-bound-of-distingushing-power} suggests that the upper bound of the distinguishing power decays quadratically with respect to the distance between the distortion functions. Hence, candidates with similar induced risk preferences are intrinsically more difficult to separate and we therefore construct $\hat{\mathcal{H}}$ as a set of representative and well-separated distortion functions.

\subsection{Probabilistic Models for the Agent's Decision-Making Process}
We define $h_{\hat{l}}$ as the distortion function that best approximates the agent's risk preference. 
Since the agent may not always choose the option that is optimal under his or her preference, we propose the following decision-making models of an agent.
\subsubsection{The Specific Decision-Making Model}\label{subsubsection-The-Specific-Decision-Making-Model}
We introduce a parameter $\beta > 0$ reflecting the degree of rationality in the agent's decision-making process, which is assumed to be unknown.
We model the probability of choosing action $  a_1  $ under question $  G^{(n)}  $ and the preferred distortion function $  h_{\hat{l}}  $ as:
\begin{equation}\label{eq:probability-agent-a1}
    \widehat{P}(a_1|G^{(n)},h_{\hat{l}})=\frac{\exp{(-\beta\Phi(a_1;G^{(n)},h_{\hat{l}}))}}{\exp{(-\beta\Phi(a_1;G^{(n)},h_{\hat{l}}))}+\exp{(-\beta\Phi(a_2;G^{(n)},h_{\hat{l}}))}}.
\end{equation}
The probability of choosing the action $a_2$ can be computed as:
\begin{equation}\label{eq:probability-agent-a2}
    \widehat{P}(a_2|G^{(n)},h_{\hat{l}})=1-\widehat{P}(a_1|G^{(n)},h_{\hat{l}}).
\end{equation}

The above model is reasonable for the following reasons. First, the model assigns a probability greater than $50\%$ to the choice preferred under the agent's distortion riskmetric. Second, the model captures the behavioral pattern that a larger risk-value gap between the two
choices leads to a higher probability of selecting the option with lower risk under the
agent’s preferred distortion riskmetric. Moreover, the inclusion of the parameter $  \beta  $ enables the model to flexibly reflect real-world behaviors. 
If $\beta$ is large, the agent behaves with more confidence and is more likely to choose the optimal choice. If $\beta$ is small, the agent's choice becomes more uncertain, increasing the likelihood of selecting a suboptimal choice. If $\beta=0$, the agent exhibits no risk preferences and is indifferent between the two choices. When $\beta\rightarrow{\infty}$, the agent consistently selects the optimal choice. 
\begin{comment}
The limiting case $\beta\rightarrow{\infty}$ has been studied in \cite{cheng2023eliciting}. 
\end{comment}
The value of $  \beta  $ is influenced by various factors including the user's fatigue and the extent to which they have insight into their own risk preferences.
\subsubsection{A General Decision-Making Model}
As for the general model, we assume that there exists a constant $\bar{P}\in (\frac{1}{2},1]$, such that the agent selects the optimal choice $a^*\in \mathcal{A}'$ given his or her preferred distortion riskmetric at round $n$ with a stochastic probability 
\begin{equation}\label{eq:fixed-probability-agent-a1}
    \widehat{P}(a^*|G^{(n)},h_{\hat{l}})=\widetilde{P}_n,
\end{equation}
where $\widetilde{P}_n\in [\bar{P},1]$, while the other choice is selected with probability 
\begin{equation}\label{eq:fixed-probability-agent-a2}
    1-\widehat{P}(a^*|G^{(n)},h_{\hat{l}})=1-\widetilde{P}_n.
\end{equation}
The above model represents a more general case, since it does not impose a specific functional form. The model in Section \ref{subsubsection-The-Specific-Decision-Making-Model} can be viewed as a special case of this general decision-making model.

\subsection{Estimation Method}
We establish a Bayesian-style framework to update the probability vector of the distortion functions. The Bayesian paradigm suggests that 
\begin{equation}
    \text{Posterior Probability}\propto \text{Likelihood} \times \text{Prior Probability}.
\end{equation}
For any distortion function $h_i\in\hat{\H}$,  we introduce the normalized regrets of selecting different actions $a_1$ and $a_2$ under the question $G^{(n)}$ as:
\begin{equation*} 
\begin{aligned}
&\widetilde\Phi(a_1;G^{(n)},h_i)=\frac{\Phi(a_1;G^{(n)},h_i)}{\max\{\Phi(a_1;G^{(n)},h_i), \Phi(a_2;G^{(n)},h_i)\}}, \\ &\widetilde\Phi(a_2;G^{(n)},h_i)=\frac{\Phi(a_2;G^{(n)},h_i)}{\max\{\Phi(a_1;G^{(n)},h_i), \Phi(a_2;G^{(n)},h_i)\}}.
\end{aligned}
\end{equation*}
If $a_1$ is the optimal action, we have $\widetilde\Phi(a_1;G^{(n)},h_i)=0$ and $\widetilde\Phi(a_2;G^{(n)},h_i)=1$. Conversely, if $a_2$ is optimal, then $\widetilde{\Phi}(a_2;G^{(n)},h_i)=0$ and $\widetilde{\Phi}(a_1;G^{(n)},h_i)=1$. The pseudo likelihood of selecting action $a_1$ under $G^{(n)}$ is computed as:
\begin{equation}\label{eq:likelihood-agent-a1}
P(a_1|G^{(n)},h_i)=\frac{\exp{(-\kappa\widetilde \Phi(a_1;G^{(n)},h_i))}}{\exp{(-\kappa\widetilde\Phi(a_1;G^{(n)},h_i))}+\exp{(-\kappa\widetilde\Phi(a_2;G^{(n)},h_i))}},
\end{equation}
where $\kappa>0$ is the learning rate. 
Meanwhile, the pseudo likelihood of action $a_2$ is
\begin{equation}\label{eq:likelihood-agent-a2}
P(a_2|G^{(n)},h_i)=1-P(a_1|G^{(n)},h_i).
\end{equation}
Based on the agent's choice $  a^{(n)} \in \{a_1, a_2\}  $ in round $n$, we update the cumulative likelihood of each distortion function $h_i$ in $\hat{\H}$ by
\begin{equation}\label{eq:update-process}
    Q^{(n)}_i=P(a^{(n)}|G^{(n)},h_i)Q^{(n-1)}_{i}, \quad Q^{(0)}_i=\widehat{Q}_i^{(0)}, \quad n=1,\cdots, N.
\end{equation}
We normalize the cumulative likelihood across different distortion functions to obtain the set of posterior probabilities for the distortion functions: $\{\widehat{Q}_i^{(n)}: i=1,2,\cdots,L\}$, where
\begin{equation*}
    \widehat{Q}_i^{(n)}=\frac{Q_i^{(n)}}{\sum^L_{j=1}Q_j^{(n)}}.
\end{equation*}
Once the update process has stabilized or the final round is reached, we select the
distortion function with the highest posterior probability as the estimate of the agent’s
preferred distortion function. 
The details for the estimation process are provided in Algorithm \ref{alg:estimate-distortion-function}.

\begin{breakablealgorithm}
\caption{Estimation of the distortion function}
\label{alg:estimate-distortion-function}
\begin{algorithmic}
\REQUIRE Early stopping parameter $q$, total number of rounds $N$, number of questions $S$, learning rate $\kappa$, exploration rate $\epsilon$, question pool  $\mathcal{K}$, initial cumulative likelihood set $\{Q^{(0)}_i: i=1,2,\cdots,L \}$.
%initial probability set $\{\widehat{Q}^{(0)}_i: i=1,2,\cdots,L \}$
 \FOR{each epoch $n = 1$ to $N$}
% \FOR{$s = 1$ to $S$}
% \FOR{$m = 1$ to $M$}
%\STATE Select two distortion functions $h_{i}(\cdot)$ and $h_{j}(\cdot)$ according to the probability set $\{Q^{(n)}_l: l=1,2,\cdots,L\}$ randomly. 
%\STATE Compute $d^{(m)}=\Psi(G^{(n)}_s,i,j)$.
%\ENDFOR
%\STATE Calculate $e_s^{(n)}=\frac{1}{M}\sum_{m=1}^{M}d^{(m)}$.
%\ENDFOR
%\STATE Obtain $s^*=\arg \min_{s\in \{1,2,\cdots, S\}}e_s^{(n)}$.
%\ELSE
\STATE Generate a random number $u \sim \mathcal{U}(0,1)$.
\IF{$u < 1-\epsilon$}
\STATE Select two distortion functions $h_{i}$ and $h_{j}$ with the largest and the second largest probability.
\ELSE
\STATE Select two distortion functions $h_{i}$ and $h_{j}$ randomly.
\ENDIF
\STATE Choose the question $G^{(n)}_{s^*}$ that best distinguishes $h_i$ and $h_j$ and collect the action $a^{(n)}$ chosen by the agent.
\FOR{$l = 1$ to $L$}
\STATE Obtain $a'\in \mathcal{A}'\setminus \{a^{(n)}\}$.
\STATE Compute
\begin{equation*}
P(a^{(n)}|G^{(n)},h_l)=\frac{\exp{(-\kappa\widetilde\Phi(a^{(n)};G^{(n)}_{s^*},h_l))}}{\exp{(-\kappa\widetilde\Phi(a^{(n)};G^{(n)}_{s^*},h_l))}+\exp{(-\kappa\widetilde\Phi(a';G^{(n)}_{s^*},h_l))}}.
\end{equation*}
\STATE Update the cumulative likelihood:
\begin{equation*}
    Q^{(n)}_l=P(a^{(n)}|G^{(n)}_{s^*},h_l)Q^{(n-1)}_{l}.
\end{equation*}
\ENDFOR
\STATE Obtain $l^{*,(n)}=\arg \max_{l\in \{1,2,\cdots,L\}}{Q}_l^{(n)}$. 
\IF{$n+1\ge q$}
\STATE Compute $\mu=\frac{1}{q}\sum^{q-1}_{j=0}{Q}_{l^{*,(n-j)}}^{(n-j)}$ and $\sigma=\sqrt{\frac{1}{q}\sum^{q-1}_{j=0}({Q}_{l^{*,(n-j)}}^{(n-j)}-\mu)^2}$.
\IF{$\sigma < 10^{-3}$}
        \STATE \textbf{break}  % 使用\textbf{break}
    \ENDIF
    \ENDIF
\ENDFOR
\STATE Obtain the distortion function $h_{l^*}$ where $l^{*,(n)}=\arg \max_{l\in \{1,2,\cdots,L\}}{Q}_l^{(n)}$.
\end{algorithmic}
\end{breakablealgorithm}

\subsection{Convergence Rate Analysis}
We establish the convergence rate of the algorithm under the case $h_{\hat{l}}\in \hat{\H}$ and the general setting \eqref{eq:fixed-probability-agent-a1} and \eqref{eq:fixed-probability-agent-a2}. We introduce a quantity:
\begin{equation*}
        Z_{i,n}(a^{(n)})=\log \frac{P(a^{(n)}|G^{(n)},h_i)}{P(a^{(n)}|G^{(n)},h_{\hat{l}})},
    \end{equation*}
    where $a^{(n)}$ is the action chosen by the agent at round $n$. Here $Z_{i,n}(a^{(n)})$ represents the advantage of the likelihood of the distortion function $  h_i  $ over the likelihood of the agent's preferred distortion function $  h_{\hat{l}}  $, given that the agent selects action $  a^{(n)}  $. 
    \begin{comment}
    It is worth noting that at round $n$,
    \begin{align*}
        \widehat {E}[Z_{i,n}(a^{(n)})| \mathcal{F}_{n-1}]=\text{KL}(p_n||q_{\hat{l},n})-\text{KL}(p_n||q_{i,n}), 
    \end{align*}
    where $\widehat {E}[Z_{i,n}(a^{(n)})| \mathcal{F}_{n-1}]$ is the expectation of $Z_{i,n}(a^{(n)})$, $\text{KL}(p||q)$ is the KL-divergence of $p$ and $q$, 
    $p_n$ denotes the true probability of the agent's decision, and $q_{i,n}$ represents the likelihood \eqref{eq:likelihood-agent-a1} and \eqref{eq:likelihood-agent-a2} under the distortion function $h_i$.
    \end{comment}
    To simplify the notation, we set $d_j(n)=\rho_{h_j}(X_n)-\rho_{h_j}(Y_n)$. We define $\sigma(u)=\frac{1}{1+e^{-u}}$ and the set of distinguishing rounds of the distortion function pair $\{h_i, h_{\hat{l}}\}$ by $\mathcal{N}_i=\left\{t\ge 1:\ d_{\hat l}(t)d_i(t)<0\right\}$. We also define $\mathcal{N}_i^e$ as the set of distinguishing rounds when distinguishing questions between $h_i$ and $h_{\hat l}$ are generated by the exploration step, and $\mathcal{N}_i^e\subset \mathcal{N}_i$.
\begin{comment}
\begin{proposition}\label{prop:existence-rounds-distinguishing-problem}
For all $i\in \{1,2,\cdots, L\}\setminus  \{\hat{l}\}$, 
%there exists a collection of rounds $\mathcal{N}_i\subset \{1,2,\cdots,n\}$ such that $d_{\hat{l}}(t)d_{i}(t)<0$ if $t\in \mathcal{N}_i$, and $d_{\hat{l}}(t)d_{i}(t)>0$ if $t\notin \mathcal{N}_i$. 
there exist constants $\pi_0\in (0,1]$ and $n_0\in \mathbb{N}^+$, such that
    \begin{align*}
        |\mathcal{N}_i^e\cap \{1,2,\cdots, n\}|\ge \pi_0n, \quad \forall n\ge n_0.
    \end{align*} \end{proposition}
\begin{proof}[Proof of Proposition \ref{prop:existence-rounds-distinguishing-problem}]
Let $E_{i,t}$ be the event that, at round $t$, the algorithm enters the
exploration step and the randomly selected pair of distortion functions is
$\{h_{\hat l},h_i\}$. 
%On the event $E_{i,t}$, under Assumption \ref{assumption:bound-of-ratio-of-difference-risk-measure}, there exists $G=\{X,Y\}\in \mathcal{K}$ such that $d_{\hat l}(t)d_i(t)<0$.
Let $I_{i,t}:=\mathds 1_{E_{i,t}}$ and $S_i(n):=\sum_{t=1}^n I_{i,t}.$ 
By the strong law of large numbers, $\frac{S_i(n)}{n}\longrightarrow \frac{2\epsilon}{L(L-1)}$ almost surely.
Choose $\pi_0:=\frac{\epsilon}{L(L-1)}.$
Then, for each $i\neq \hat l$, there exists a finite $n_{0,i}$ such that for all $n\ge n_{0,i}$,
\begin{align}
S_i(n)
\ge
\pi_0 n, \quad a.s.
\end{align}
Finally, since there are finitely many $i\neq \hat l$, taking $n_0:=\max_{i\neq \hat l} n_{0,i}$
yields the desired result simultaneously for all
$i\in\{1,2,\ldots,L\}\setminus\{\hat l\}$.
\end{proof}
\end{comment}
The following lemma shows the upper bound of $Z_{i,n}$.
\begin{lemma}\label{lemma:upper-bound-of-log-ratio-probability}
     Under Assumption \ref{assumption:unique-optimal-action}, for all $i\in \{1,2,\cdots, L\}\setminus  \{\hat{l}\}$ and every round $n$, we have 
    \begin{equation*}
        |Z_{i,n}(a^{(n)})|\le \kappa.
    \end{equation*}
\end{lemma}
\begin{proof}{Proof of Lemma \ref{lemma:upper-bound-of-log-ratio-probability}}
    Without loss of generality, let $a_1^{(n)}$ denote the action of choosing the option $X_n$, and $a_2^{(n)}$ the action of selecting option $Y_n$. According to the likelihood \eqref{eq:likelihood-agent-a1} and \eqref{eq:likelihood-agent-a2}, %If $\rho_h(X_n)< \rho_h(Y_n)$, then $\Phi(a_1^{(n)};G^{(n)},h)=0$, and $\Phi(a_2^{(n)};G^{(n)},h)=\rho_h(Y_n)-\rho_h(X_n)=\Delta^n_h$. If $\rho_h(X_n)> \rho_h(Y_n)$,$\Phi(a_1^{(n)};G^{(n)},h)=\rho_h(X_n)- \rho_h(Y_n)=-\Delta^n_h$, $\Phi(a_2^{(n)};G^{(n)},h)=0$. 
    if $d_i(n)<0$ and $a_1^{(n)}$ is the optimal choice, we have
    \begin{equation*}
        P(a_1^{(n)}|G^{(n)},h)=\sigma(\kappa), \ P(a_2^{(n)}|G^{(n)},h)=\sigma(-\kappa).
    \end{equation*}
    Hence, for $a^{(n)}\in\{a_1^{(n)},a_2^{(n)}\}$ and $i\ne \hat{l}$, we have
    \begin{equation*}
        |Z_{i,n}(a^{(n)})|=\left|\log \frac{P(a^{(n)}|G^{(n)},h_i)}{P(a^{(n)}|G^{(n)},h_{\hat{l}})}\right|\le \log \frac{\sigma(\kappa)}{\sigma(-\kappa)}=\kappa. 
    \end{equation*}
\end{proof}
The following proposition demonstrates that distinguishing questions can increase the expected log-likelihood ratio of the distortion function preferred by the agent more than that of other distortion functions. By contrast, for non-distinguishing questions, a wrong risk preference will not gain any advantage over the true one from these questions.

\begin{proposition}\label{prop:negative-upper-bound-of-the-ratio}
Under Assumption \ref{assumption:unique-optimal-action}, for all $\kappa>0$ and $i\in \{1,\cdots, L\}\setminus  \{\hat{l}\}$, given the question $(X_n, Y_n)$, if $(\rho_{h_{\hat{l}}}(X_n)-\rho_{h_{\hat{l}}}(Y_n))(\rho_{h_{i}}(X_n)-\rho_{h_{i}}(Y_n))<0$, 
%if an agent makes choices according to \eqref{eq:probability-agent-a1} and \eqref{eq:probability-agent-a2}, and $0<\kappa\le \beta$, then we have
under the general decision-making model \eqref{eq:fixed-probability-agent-a1} and \eqref{eq:fixed-probability-agent-a2},  the expectation of $Z_{i,n}(a^{(n)})$ satisfies
\begin{equation*}
    \widehat E[Z_{i,n}(a^{(n)})| \mathcal{F}_{n-1}]\le -\gamma, 
\end{equation*}
where $\gamma=(2\bar{P}-1)\kappa>0$. If $(\rho_{h_{\hat{l}}}(X_n)-\rho_{h_{\hat{l}}}(Y_n))(\rho_{h_{i}}(X_n)-\rho_{h_{i}}(Y_n))>0$, we have
\begin{equation*}
    \widehat E[Z_{i,n}(a^{(n)})| \mathcal{F}_{n-1}]=0.
\end{equation*}
%where $\gamma=\kappa \hat{\Delta}\tanh{\frac{\beta\hat{\Delta}}{2}}>0$. 
%Similarly, if an agent makes choices according to \eqref{eq:fixed-probability-agent-a1} and \eqref{eq:fixed-probability-agent-a2}, and $0<\kappa\le \frac{1}{D}\log \frac{\bar{P}}{1-\bar{P}}$, then $\gamma=(2\bar{P}-1)\kappa \hat{\Delta}>0$.
\end{proposition}
\begin{proof}{Proof of Proposition \ref{prop:negative-upper-bound-of-the-ratio}}
Without loss of generality, let $\rho_{h_{\hat{l}}}(X_n)<\rho_{h_{\hat{l}}}(Y_n)$ and $\rho_{h_{i}}(X_n)>\rho_{h_{i}}(Y_n)$. Let $a_1^{(n)}$ denote the action of choosing $X_n$, and $a_2^{(n)}$ denote the action of choosing $Y_n$. 
%Then $\Phi(a_1^{(n)},G^{(n)},h_{\hat{l}})=0$, $\Phi(a_2^{(n)},G^{(n)},h_{\hat{l}})\ge \hat{\Delta}$, $\Phi(a_2^{(n)},G^{(n)},h_{i})=0$, and $\Phi(a_1^{(n)},G^{(n)},h_{i})\ge \hat{\Delta}$. 
To simplify the notation, we set $p_n:=\widetilde P_n\in[\bar P,1]$ according to \eqref{eq:fixed-probability-agent-a1}.
Under the general decision-making model, we have $\widehat P(a^{(n)}_1\mid G^{(n)},h_{\hat l})=p_n$ and $\widehat P(a^{(n)}_2\mid G^{(n)},h_{\hat l})=1-p_n.$ 
By the likelihood \eqref{eq:likelihood-agent-a1} and \eqref{eq:likelihood-agent-a2}, we have
%\begin{align}P(a^{(n)}_1\mid G^{(n)},h_{\hat l})=\sigma(\kappa d),\ P(a^{(n)}_2\mid G^{(n)},h_{\hat l})=1-\sigma(\kappa d),\ P(a^{(n)}_1\mid G^{(n)},h_i)=1-\sigma(\kappa e), \ P(a^{(n)}_2\mid G^{(n)},h_i)=\sigma(\kappa e).\end{align}
%$P(a^{(n)}_1\mid G^{(n)},h_{\hat l})=\sigma(\kappa d)$, $P(a^{(n)}_2\mid G^{(n)},h_{\hat l})=1-\sigma(\kappa d)$, $P(a^{(n)}_1\mid G^{(n)},h_i)=1-\sigma(\kappa e),$ and $P(a^{(n)}_2\mid G^{(n)},h_i)=\sigma(\kappa e).$
\begin{equation*}
\begin{aligned}
\widehat E[Z_{i,n}(a^{(n)})| \mathcal{F}_{n-1}]
&=
p_n\log\frac{1-\sigma(\kappa )}{\sigma(\kappa )}
+(1-p_n)\log\frac{\sigma(\kappa )}{1-\sigma(\kappa )} \\
&=p_n\log\frac{\sigma(-\kappa )}{\sigma(\kappa )}
+(1-p_n)\log\frac{\sigma(\kappa )}{\sigma(-\kappa )}\\
&=-p_n\kappa+(1-p_n)\kappa\\
&=-(2p_n-1)\kappa \le -(2\bar P-1)\kappa.
\end{aligned}
\end{equation*}
If $\rho_{h_{\hat{l}}}(X_n)<\rho_{h_{\hat{l}}}(Y_n)$ and $\rho_{h_{i}}(X_n)<\rho_{h_{i}}(Y_n)$, we have 
\begin{equation*}
    \widehat E[Z_{i,n}(a^{(n)})| \mathcal{F}_{n-1}]=p_n\log\frac{\sigma(\kappa )}{\sigma(\kappa )}
+(1-p_n)\log\frac{1-\sigma(\kappa )}{1-\sigma(\kappa )}=0.
\end{equation*}
\end{proof}

\begin{corollary}\label{corollary:KL-divergence-sum}
Under Assumption \ref{assumption:unique-optimal-action}, we have
    \begin{equation*}
    \sum_{t=1}^n \widehat {E}[Z_{i,t}(a^{(t)})| \mathcal{F}_{t-1}]
    \le -(2\bar P-1)\kappa \pi_0n,
    \end{equation*}
    where $\pi_0=\frac{2\epsilon}{L(L-1)}$.
\end{corollary}
\begin{proof}{Proof of Corollary \ref{corollary:KL-divergence-sum}}
We define $\mathcal{G}_t=\mathcal{F}_{t-1}\vee \sigma(G^{(t)})$ as the $\sigma-$algebra containing all information  available up
to round $t-1$, together with the question selected for round $t$.
Let $V_{i,t}$ be the event that the selected question for round $t$ is to distinguish the pair of distortion functions $\{h_{\hat l},h_i\}$. According to Proposition \ref{prop:negative-upper-bound-of-the-ratio}, we have 
\begin{equation*}
\widehat {E}[Z_{i,t}(a^{(t)})| \mathcal{G}_{t}]\le -(2\bar P-1)\kappa\mathbf{1}_{V_{i,t}}. 
\end{equation*}
Since $\widehat{P}(V_{i,t}|\mathcal{F}_{t-1})\ge \frac{2\epsilon}{L(L-1)}$, we have
\begin{equation*}
\begin{aligned}
    \widehat {E}[Z_{i,t}(a^{(t)})| \mathcal{F}_{t-1}]&=\widehat {E}[\widehat {E}[Z_{i,t}(a^{(t)})| \mathcal{G}_{t}]| \mathcal{F}_{t-1}]\\
    &\le -(2\bar P-1)\kappa \widehat{P}(V_{i,t}|\mathcal{F}_{t-1})\le -(2\bar P-1)\kappa\cdot \frac{2\epsilon}{L(L-1)}.
\end{aligned}
\end{equation*}
Hence, we complete the proof.
\end{proof}

Theorem \ref{theorem:client-choose-P-convergence} establishes an upper bound on the convergence rate of our algorithm under the general setting \eqref{eq:fixed-probability-agent-a1} and \eqref{eq:fixed-probability-agent-a2}. 
%It shows that within a finite number of rounds, the probability of the agent’s preferred distortion function has a high probability of ranking first and will remain so thereafter. 
The convergence rate is of order $O\left(\exp\left(-cm+O\left(\sqrt{m\log m}\right)\right)\right)$ for some constant $c>0$, where $m$ is the number of questions.
\begin{theorem}\label{theorem:client-choose-P-convergence}
Under Assumption \ref{assumption:unique-optimal-action}, if $\kappa>0$, under the general model setting \eqref{eq:fixed-probability-agent-a1} and \eqref{eq:fixed-probability-agent-a2}, for all $\delta\in(0,1)$ and all $m\ge 1$, we have
\begin{equation*}
\widehat P\left(\left|\widehat{Q}_{\hat{l}}^{(m)}-1\right|\le r(m)\right)\ge 1-\delta,
\end{equation*}
where 
$r(m)=(L-1)\exp\left(-g_0\pi_0m+ 2\kappa\sqrt{2ma_{m}}\right)$, $g_0=(2\bar P-1)\kappa$, $\pi_0=\frac{2\epsilon}{L(L-1)}$, 
and $a_{m}=\log\left(\frac{(L-1)\pi^2m^2}{6\delta}\right)$.
\begin{comment}
$\forall \delta_1, \delta_2>0$, there exists some constant $\Delta>0$ and $n_0\in\mathbb N_+$ such that
{\color{yellow}\begin{align*}
\P\left(A_{\hat{l}}^{(G)}\right)\ge 1-\delta_1, \ \P\left(B_{\hat{l}}^{(H,\infty)}\right)\ge 1-\delta_2,
    \end{align*}}
    where $G=\left\lceil \max \left\{n_0,  \frac{C_0}{g_0\pi_0}+1, \frac{8B^2}{\Delta^2}\log \frac{L-1}{\delta_1} \right\} \right\rceil$, $H=\bigg\lceil \max \Big\{n_0, \frac{C_0}{g_0\pi_0}+1, \frac{8B^2}{\Delta^2}\Big(\log \frac{L-1}{\delta_2}+\log\frac{1}{1-\exp(-\Delta^2/(8B^2)}\Big) \Big\} \bigg\rceil$, $g_0=(2\bar{P}-1)\kappa \hat{\Delta}$, and  $B=\kappa D$ for some constant $D>0$. 
Moreover, $\forall \delta>0$, $\forall m\ge 1$, we have 
    \begin{align*}
\P\left(\left|\widehat{Q}_{\hat{l}}^{(m)}-1\right|\le \varepsilon(m)\right)\ge 1-\delta,
    \end{align*}
    where 
        $\varepsilon(m)=(L-1)\exp\left(-m\Delta+ 2B\sqrt{2ma_{m}}\right),$
    and $a_{m}=\log\left(\frac{(L-1)\pi^2m^2}{6\delta}\right)$.
\end{comment}
\end{theorem}
\begin{proof}{Proof of Theorem \ref{theorem:client-choose-P-convergence}}

   Let $S_i(m):=\log \frac{Q_{\hat{l}}^{(m)}}{Q_i^{(m)}}$. We have
   \begin{equation*}
\widehat{Q}_{\hat{l}}^{(m)}=\frac{Q_{\hat{l}}^{(m)}}{\sum_{j=1}^LQ_{j}^{(m)} }
       =\frac{1}{1+\sum_{i\ne \hat{l}}\frac{Q_{i}^{(m)}}{Q_{\hat{l}}^{(m)}}}
       =\frac{1}{1+\sum_{i\ne \hat{l}}e^{-S_i(m)}},
   \end{equation*}
   and
       $1-\widehat{Q}_{\hat{l}}^{(m)}=\frac{\sum_{i\ne \hat{l}}e^{-S_i(m)}}{1+\sum_{i\ne \hat{l}}e^{-S_i(m)}}\le \sum_{i\ne \hat{l}}e^{-S_i(m)}.$
   %Suppose $\tau$ is the first time such that $\hat{l}=i^{*,\tau}$. Obviously $\tau\ge n_0$. Let $R_{\tau}=\sum_{i\ne \hat{l}}e^{-S_i(\tau)}$. Since $\hat{l}\in \{i^{*,\tau},j^{*,\tau}\}$, for any $i\notin \{i^{*,\tau},j^{*,\tau}\}$, $\frac{Q_i^{(\tau)}}{Q_{\hat{l}}^{(\tau)}}\le 1$. %Since $\widehat{Q}_{\hat{l}}^{(\tau)}\ge \frac{1}{L}$, $\frac{Q_{i^{*,\tau}}^{(\tau)}}{Q_{\hat{l}}^{(\tau)}}=\frac{\widehat{Q}_{i^{*,\tau}}^{(\tau)}}{\widehat{Q}_{\hat{l}}^{(\tau)}}\le \frac{1-\widehat{Q}_{\hat{l}}^{(\tau)}}{\widehat{Q}_{\hat{l}}^{(\tau)}}\le L-1$.
   %So $R_\tau=\sum_{j\ne \hat{l}}\frac{Q_{j}^{(n)}}{Q_{\hat{l}}^{(n)}}\le (L-1)$.
   We set $\widehat{Z}_{i,t}(a^{(t)})=-{Z}_{i,t}(a^{(t)})$. According to the update process \eqref{eq:update-process}, we have $S_i(t)=S_i(t-1)+\widehat{Z}_{i,t}(a^{(t)})$.  We define $\widehat{\eta}_{i,t}:=\widehat{Z}_{i,t}(a^{(t)})-\widehat E[\widehat{Z}_{i,t}(a^{(t)})|\mathcal{F}_{t-1}]$ and have $\widehat E[\widehat{\eta}_{i,t}|\mathcal{F}_{t-1}]=0$. We define
   \begin{equation*}
       \widehat{M}_{i,m}:=\sum_{t=1}^m \widehat{\eta}_{i,t}, \quad \widehat{M}_{i,0}:=0.
   \end{equation*}
   Hence, $\widehat{M}_{i,m}$ is a martingale with respect to $\mathcal{F}_{m}$. According to Lemma \ref{lemma:upper-bound-of-log-ratio-probability}, we have $\left|\widehat E[\widehat Z_{i,t}(a^{(t)})|\mathcal{F}_{t-1}]\right|\le \kappa$.  Hence,
    \begin{equation*}
        |\widehat \eta_{i,t}|\le \left|\widehat Z_{i,t}(a^{(t)})\right|+\left|\widehat E[\widehat Z_{i,t}(a^{(t)})|\mathcal{F}_{t-1}]\right|\le 2\kappa.
    \end{equation*}
   Noting that $S_i(0)=0$, we have
   \begin{equation*}
       S_i(m)=\sum_{t=1}^{m}\widehat{Z}_{i,t}(a^{(t)})=\sum_{t=1}^{m}\widehat{\eta}_{i,t}+\sum_{t=1}^{m}\widehat E[\widehat{Z}_{i,t}(a^{(t)})|\mathcal{F}_{t-1}]
       = -\sum_{t=1}^m\Delta_{i,t}+\widehat{M}_{i,m},
   \end{equation*}
   where $\Delta_{i,t}=-\widehat E[\widehat{Z}_{i,t}(a^{(t)})|\mathcal{F}_{t-1}]$.
   %$p_t$ denotes the true probability of the agent's decision  \eqref{eq:fixed-probability-agent-a1} and \eqref{eq:fixed-probability-agent-a2}, and $q_{i,t}$ represents the likelihood \eqref{eq:likelihood-agent-a1} and \eqref{eq:likelihood-agent-a2} under the distortion function $h_i$. 
   According to Corollary \ref{corollary:KL-divergence-sum},  we have $-\sum_{t=1}^m\Delta_{i,t}\ge g_0\pi_0m>0,$
    where $g_0=(2\bar{P}-1)\kappa$ and $\pi_0=\frac{2\epsilon}{L(L-1)}$.
   Since $\widehat M_{i,m}$ is a martingale and $|\widehat \eta_{i,t}|\le 2\kappa$, by Azuma-Hoeffding inequality, for any $x>0$,
   \begin{equation*}
       \widehat P(\widehat{M}_{i,m}\le -x)\le \exp\left(-\frac{x^2}{2\sum_{t=1}^{m}4\kappa^2}\right)=\exp\left( -\frac{x^2}{8\kappa^2 m}\right).
   \end{equation*}
   Let $\delta_{m}=\frac{6\delta}{\pi^2(m)^2}$, $a_{m}=\log\left(\frac{(L-1)\pi^2(m)^2}{6\delta}\right)$, $x_{m}=2\kappa \sqrt{2ma_{m}}$. Then
   \begin{equation*}
        \widehat P(\widehat{M}_{i,m}\le -x_m)\le \exp(-a_{m})=\frac{\delta_{m}}{L-1}. 
   \end{equation*}
   By Boole’s inequality,
   \begin{equation*}
       \widehat P\left(\exists i\ne\hat{l}: \widehat{M}_{i,m}\le -x_{m}\right)\le \sum_{i\ne \hat{l}}\P \left(\widehat{M}_{i,m}\le -x_{m} \right)\le (L-1)*\frac{\delta_{m}}{L-1}=\delta_{m}.
   \end{equation*}
   Since 
       $\sum^\infty_{m=1}\delta_{m}=\frac{6\delta}{\pi^2}\sum^\infty_{m=1}\frac{1}{m^2}=\frac{6\delta}{\pi^2}\cdot \frac{\pi^2}{6}=\delta,$
   by Boole’s inequality,
   \begin{equation*}
       \widehat P\left(\exists m\ge 1, \exists i\ne\hat{l}: \widehat{M}_{i,m}\le -x_{m} \right)\le \sum_{m=1}^\infty \widehat P\left(\exists i\ne\hat{l}: \widehat{M}_{i,m}\le -x_{m}\right)=\delta.
   \end{equation*}
   Hence,
    $\widehat P\left(\forall m\ge 1, \forall i\ne\hat{l}: \widehat{M}_{i,m}\ge -x_{m} \right)\ge 1-\delta.$
   In the event $\Theta :=\{\forall m\ge 1, \forall i\ne\hat{l}: \widehat{M}_{i,m}\ge -x_{m}\}$, we have
       $\sum_{t=1}^{m}\widehat{Z}_{i,t}(a^{(t)})\ge g_0\pi_0m-x_{m}.$
   Since $S_i(m)=\sum^{m}_{t=1}\widehat{Z}_{i,t}(a^{(t)}),$
   in the event $\Theta$, we have
   \begin{equation*}
   \begin{aligned}
       S_i(m)&\ge g_0\pi_0m-2\kappa \sqrt{2ma_{m}},\\
       %e^{-S_i(\tau+m)}&\le e^{-S_i(\tau)}\exp\left(-m\Delta+2\kappa D\sqrt{2ma_{m}}\right)\\
       \sum_{i\ne \hat{l}}e^{-S_i(m)}&\le \sum_{i\ne \hat{l}}\exp\left(-g_0\pi_0m+2\kappa \sqrt{2ma_{m}}\right),\\
       \sum_{i\ne \hat{l}}e^{-S_i(m)}&\le (L-1)\exp\left(-g_0\pi_0m+2\kappa \sqrt{2ma_{m}}\right).
   \end{aligned}
   \end{equation*}
   Since $1-\widehat{Q}_{\hat{l}}^{(m)}\le \sum_{i\ne \hat{l}}e^{-S_i(m)}$, we have
   \begin{equation*}
\widehat P\left(\left|\widehat{Q}_{\hat{l}}^{(m)}-1\right|\le (L-1)\exp\left(-g_0\pi_0m+2\kappa \sqrt{2ma_{m}}\right)\right)\ge 1-\delta.
    \end{equation*}
\end{proof}
Theorem \ref{theorem:client-choose-P-convergence} shows that regardless of the specific pattern the agent follows during the selection process, as long as the agent does not choose options completely at random, we can identify the distortion riskmetric that reflects the agent’s underlying preferences by setting an appropriate learning rate $\kappa$. 
\begin{remark}
    Although any learning rate $\kappa>0$ is sufficient for the convergence analysis, it should be chosen carefully in practice when the agent's level of uncertainty is unknown: an overly large value of $\kappa$ may lead to overconfident updates of the posterior probabilities, whereas an excessively small value may slow the update of the posterior probabilities.  An appropriate learning rate should align the likelihood \eqref{eq:likelihood-agent-a1} and \eqref{eq:likelihood-agent-a2} with the agent’s true probability. Since the general decision-making model assumes $\widetilde P_n\in[\bar P,1]$, we choose $\kappa$ conservatively by minimizing
the worst-case expected negative log-likelihood:
\begin{equation*}
    \min_{s\in(1/2,1)}
    \sup_{p\in[\bar P,1]}
    \left[-p\log s-(1-p)\log(1-s)\right], 
\end{equation*}
where $s=\sigma(\kappa)$. Since $f(s,p):=-p\log s-(1-p)\log(1-s)$ is decreasing in $p$, we have
\begin{equation*}
    \min_{s\in(1/2,1)}
    \sup_{p\in[\bar P,1]}
    \left[-p\log s-(1-p)\log(1-s)\right]\iff \min_{s\in(1/2,1)} \left[-\bar P\log s-(1-\bar P)\log(1-s)\right].
\end{equation*}
If $\bar P<1$, the minimizer of $f(s,\bar P)$ is that 
$s=\bar P$, and we have $\kappa=\log\frac{\bar P}{1-\bar P}$. If $\bar P=1$, then $\kappa$ can be chosen arbitrarily large. In practice, we can set $\kappa=\log\frac{0.55}{0.45}\approx 0.2$, which corresponds to assuming that the minimum probability with which agents choose their preferred choices is 0.55.
\end{remark}

\section{Elicit-to-Optimize Reinforcement Learning}\label{sec: Risk-Sensitive Reinforcement Learning}
\subsection{Formulation}
In this section, we solve the risk-sensitive reinforcement learning problems after identifying the agent’s preferred distortion function. Let  $T \in \mathbb{N}$ be a finite time horizon, and $\mathcal{L} := \{0, 1, 2, \ldots, T\}$. 
%We denote by $\mathcal{Y}$ the linear space of all bounded measurable functions (i.e. random variables) on $(\Omega, \mathcal{F})$ and $\mathbb{F}=(\mathcal{F}_t, t=0,1,2,\cdots)$ the natural filtration.
Let $\mathcal{S}$ and $\mathcal{A}$ denote the state and the action spaces, respectively.
\begin{comment}
and let $\mathcal{C}\in \R$ be a cost space. Moreover, a cost function is denoted by $c: \mathcal{S} \times \mathcal{A} \times \mathcal{S} \to \mathbb{R}$, where a realized cost is given by $c(s, a, s')$.
\end{comment}
The gain of the agent at time $  t+1  $  satisfies the following dynamics:
\begin{equation*}
    W_{t+1}=\sum_{\tau=0}^tR(S_\tau, A_\tau, S_{\tau+1}; W_0), \quad 0<t\le T-1,
\end{equation*}
where $S_\tau\in \mathcal{S}$, $A_\tau\in \mathcal{A}$, $W_0$ denotes the initial gain, and $R:\mathcal{S}\times \mathcal{A}\times \mathcal{S}\times \R\rightarrow{\R}$ is the deterministic measurable function.

We are interested in the RL problem where a conditional distortion riskmetric is applied to the accumulated costs.   Throughout the RL section, we use the  conditional
distortion riskmetric in Definition \ref{cond-distortion}. Since the condition variable \(Y\) is discrete,
for a fixed scenario \(y_0\) with \(\mathbb P(Y=y_0)>0\), we take
$C=\{Y=y_0\}$.
For simplicity, we write
$\rho_{h|y_0}(X):=\rho_{h|C}(X)$. We use $\mathbf{s}_t$ to denote the collection of variables at time $t$, which includes the state variable $S_t$, time-to-maturity $\tau_t=T-t$, condition variable $y_t$, and the gain $W_t^\pi$: $\mathbf{s}_t=(S_t, \tau_t, y_t, W_t^\pi)$.
%$\mathbf{s}_t:=(S_t,\tau_t,Y_t,W_t^\pi)$, $t=0,1,\cdots,T.$
We denote by $ \pi(a|\mathbf{s}_t) $ a stochastic feedback policy, which is a probability density function on the action space $  \mathcal{A}  $ for any fixed $\mathbf{s}_t\in \mathcal{S}\times \mathcal{L}\times \X\times \R$.  Specifically,  we aim to solve the following optimization problem: 
\begin{equation}\begin{aligned}\label{prob:SRP}
 &  \inf_{\pi \in \Pi} \int \left(W_0-W_T^\pi \right) \d h \circ \p_{Y = y_0}   \\
 =&\inf_{\pi\in\Pi} \int_{-\infty}^{0}\left(h\left(\P\left(W_0-W_T^\pi\geq x \bigg|Y= y_0\right)\right)-h(1)\right)\d x 
  +\int_0^\infty h \left(\p \left(W_0-W_T^\pi \geq x \bigg| Y = y_0 \right)\right) \d x \\
=&\inf_{\pi\in\Pi} \rho_{h|y_0}\left(W_0-W_T^\pi \right),
 \end{aligned}
\end{equation}
where $\Pi$ is the set of admissible feedback control distributions and $W_t^\pi$ is the gain at time $t$ under policy $\pi$. 
Across different fields, $W_t^\pi$ and the condition variable $Y$ may take a wide range of meanings. For example, in investment, $W_t^\pi$ may represent the total wealth of an investor and $Y$ may represent macroeconomic conditions or market regimes. 
%In systemic risk analysis, $W_t^\pi$ may be the profit of a particular financial institution and $Y$ may be the distress or capital shortfall of another financial institution. 
In operations management, $W_t^\pi$ may describe order revenues and $Y$ may describe order cancellation rate or port congestion.

According to Lemma \ref{lemma:representation-of-conditional-distortion-riskmetrics}, the objective \eqref{prob:SRP} can be represented by
\begin{equation*}
    \sup_{\pi \in \Pi} -\int^1_0F^{-1+,\pi}_{X|Y= y}(1-z)\d h(z), 
\end{equation*}
where 
    $F^{-1+,\pi}_{X|Y=y}(1-z)=\inf\left\{x\in\R:\P\left(W_0-W_T^\pi\le x|Y=y_0\right)>1- z\right\},$
or
\begin{equation*}
    \sup_{\pi \in \Pi}-\int^1_0F^{-1,\pi}_{X|Y=y}(1-z)\d h(z),
\end{equation*}
where $F^{-1,\pi}_{X|Y=y}(1-z)=\inf\left\{x\in\R:  \P\left(W_0-W_T^\pi\le x|Y=y_0\right)\ge 1-z\right\}.$
Hence, the reward received by the agent at time $T$ can be defined as $R_T=-\int^1_0F^{-1+,\pi}_{X|Y=y}(1-z)\d h(z)$ or $R_T=-\int^1_0F^{-1,\pi}_{X|Y=y}(1-z)\d h(z)$. If $t+1<T$, the reward can be defined as
\begin{equation}\label{eq:reward}
R_{t+1}=-\left|W_{t+1}^\pi-W_0\right|\mathbf{1}_{\left\{{W_{t+1}^\pi-W_0<0}\right\}}.
\end{equation}
The formulation of the reward in Equation \eqref{eq:reward} indicates that only positive costs are penalized, which is used as a loss-oriented shaping term to guide the policy toward optimizing the final conditional distortion-riskmetric objective.
Starting from the initial state $\mathbf{s}_0$, the agent takes actions following a policy $\pi$ until maturity, which results in a trajectory
$\mathcal{T}=(\mathbf{s}_0,a_0,R_1, \mathbf{s}_1,a_1,R_2,\cdots, \mathbf{s}_{T-1},  a_{T-1},R_T, \mathbf{s}_{T}).$
Under the policy $\pi$, the value function at the state $\mathbf{s}_t$ is
\begin{equation*}
    V^\pi(\mathbf{s}_t)=\E^\pi\left[\sum_{k=0}^{T-t-1}\gamma^k R_{t+1+k}\Bigg| \mathbf{s}_t \right],
\end{equation*}
where $\gamma$ is the discount rate. 

\subsection{Policy Network and Value Network}
Both the value function and the policy are represented by fully connected, multilayer perceptrons (MLP), with their parameters denoted by $\psi\in \Psi$ and $\theta\in \Theta$, where $\Psi$ and $\Theta$ are two compact domains. The two networks are denoted by $V^{\pi,\psi}(\mathbf{s}_t)$ and $\pi^{\theta}(\mathbf{s}_t)$, respectively. The policy network $\pi^\theta$
parameterizes a Gaussian distribution over the continuous action space:
\begin{equation*}
  \mu_\theta(\mathbf{s}_t)
    = \mathrm{MLP}_\theta(\mathbf{s}_t)\in\R,
   \
  \sigma_\theta
    = \exp(\hat{\sigma}_\theta),
    \
  \pi^\theta(a\mid\mathbf{s}_t)
    = \mathcal{N}\!\left(a;\,\mu_\theta(\mathbf{s}_t),\,\sigma_\theta^2\right),
\end{equation*}
where $\hat{\sigma}_\theta$ is a learnable scalar parameter. The value network $V^{\pi,\psi}$ estimates the expected
discounted return of the policy at the current state:
\begin{equation}
  V^{\pi,\psi}(\mathbf{s}_t) = \mathrm{MLP}_\psi(\mathbf{s}_t)\in\R.
  \label{eq:value_net}
\end{equation}
The 1-step TD residual of the value function is defined as
\begin{equation*}
    \delta_t^{\pi^\theta}= R_{t+1} + \gamma\,V^{\pi,\psi}(\mathbf{s}_{t+1}) - V^{\pi,\psi}(\mathbf{s}_t),
\end{equation*}
where $\gamma=1$ is the discount factor. Then the advantage function and target returns
are computed using Generalized Advantage Estimation (GAE):
\begin{equation}\label{eq:GAE}
  \hat{A}_t^{\pi^\theta}
    = \sum_{l=0}^{T-t-1}(\gamma\lambda_{\mathrm{GAE}})^l\,\delta_{t+l}^{\pi^\theta},
    \
  \hat{G}_t
    = \hat{A}_t^{\pi^\theta} + V^{\pi,\psi}(\mathbf{s}_t),
\end{equation}
where $\lambda_{\mathrm{GAE}}$ is the GAE decay coefficient. Hence, the policy network updates are carried out using the clipped surrogate
objective of the PPO algorithm:
\begin{equation}
  \mathcal{L}_{\mathrm{PPO}}(\theta)
  = -
      \min\!\left(
        r_t(\theta)\,\hat{A}_t^{\pi^{\theta_{\mathrm{old}}}},\;
        g\left(\epsilon,\hat{A}_t^{\pi^{\theta_{\mathrm{old}}}}\right)
      \right)
    ,
  \label{eq:ppo}
\end{equation}
where the policy probability ratio $r_t(\theta)$ and the function $g(\epsilon, A)$ that prevents the policy from changing too aggressively in one update are defined as
\begin{equation}
  r_t(\theta)
  = \frac{\pi^\theta(a_t\mid\mathbf{s}_t)}
         {\pi^{\theta_{\mathrm{old}}}(a_t\mid\mathbf{s}_t)}
  = \exp\!\left(
      \log\pi^\theta(a_t\mid\mathbf{s}_t)
      - \log\pi^{\theta_{\mathrm{old}}}(a_t\mid\mathbf{s}_t)
    \right),
  \label{eq:ratio}
\end{equation}
and
\begin{equation*}
\begin{aligned}
    g(\epsilon, A)=\begin{cases} 
(1 + \epsilon)A, & \text{if } A \geq 0, \\
(1 - \epsilon)A, & \text{otherwise}.
\end{cases}
\end{aligned}
\end{equation*}
with clipping threshold $\epsilon=0.2$ and $  \theta_{\text{old}}  $ representing the parameters of the policy network obtained from the previous update. The value network is trained using a clipped mean squared error loss:
\begin{equation}
  \mathcal{L}_V(\psi)
  =  \max\!\left(
        \left(V^{\pi,\psi}(\mathbf{s}_t) - \hat{G}_t\right)^2,\;
        \left(V^{\pi,\psi}_{\mathrm{clip}}(\mathbf{s}_t) - \hat{G}_t\right)^2
      \right)
    ,
  \label{eq:value_loss}
\end{equation}
where the clipped value estimate is given by
\begin{equation*}
  V^{\pi,\psi}_{\mathrm{clip}}(\mathbf{s}_t)
  = \max\left(\min \left(V^{\pi,\psi}(\mathbf{s}_t), V^{\pi,\psi_{\mathrm{old}}}(\mathbf{s}_t) + \epsilon_V\right),
      V^{\pi,\psi_{\mathrm{old}}}(\mathbf{s}_t) - \epsilon_V\right),
\end{equation*}
and $\epsilon_V=0.2$ is the value clipping threshold and $\psi_{\mathrm{old}}$ represents the parameters of the value network obtained from the previous update. To encourage sufficient exploration, an analytical entropy
regularization term for the Gaussian policy is included:
\begin{equation}
  \mathcal{L}_{\mathrm{ent}}(\theta)
  = -\mathcal{H}[\pi_\theta(\cdot\mid\mathbf{s}_t)]
  = -\tfrac{1}{2}\bigl(1 + \log(2\pi) + 2\log\sigma_\theta\bigr).
  \label{eq:entropy_loss}
\end{equation}
Overall, we minimize the following optimization objective for the policy and value networks:
\begin{equation}
  \mathcal{L}_{\mathrm{total}}(\theta,\psi)
  = \E\left[\mathcal{L}_{\mathrm{PPO}}(\theta)
  + c_0\,\mathcal{L}_{\mathrm{ent}}(\theta)
  + c_1\,\mathcal{L}_V(\psi)\right].
  \label{eq:total_loss}
\end{equation}
\subsection{Quantile Network}
In order to calculate the objective \eqref{prob:SRP}, we approximate the $\alpha$-quantile $F^{-1+,\pi}_{X|Y= y}(\alpha)$ and $F^{-1,\pi}_{X|Y= y}(\alpha)$ by a quantile network $\omega^\phi$, which is also represented by a multilayer feedforward network. The quantile network takes the initial state $\mathbf{s}_0=(S_0,\tau_0,Y_0,W_0^\pi)$ and a confidence level $\alpha$ as inputs. According to \cite{peng2024risk}, we can set the loss function of the quantile network $\omega^\phi(\boldsymbol{s}_0,\alpha)$ as
\begin{equation*}
\mathcal{L}_{\mathrm{QR}}(\phi) = 
\begin{cases} 
\alpha \left| \omega^\phi(\boldsymbol{s}_0,\alpha) - \left(W_0-W_T^\pi \right) \right|, & W_0-W_T^\pi \geq \omega^\phi(\boldsymbol{s}_0,\alpha), \\
(1 - \alpha) \left| \omega^\phi(\boldsymbol{s}_0,\alpha) -\left(W_0-W_T^\pi \right) \right|, & W_0-W_T^\pi  \leq \omega^\phi(\boldsymbol{s}_0,\alpha).
\end{cases}
\end{equation*}
To enforce the structural constraint that higher confidence levels
correspond to larger loss quantiles, we introduce a monotonicity
regularization term:
\begin{equation}
  \mathcal{L}_{\mathrm{mono}}(\phi)
  = l\bigl(
      \omega_\phi(\mathbf{s}_0,\alpha)
      - \omega_\phi(\mathbf{s}_0,\alpha+\delta)
    \bigr),
  \label{eq:mono_loss}
\end{equation}
where $\delta=0.05$ is the finite-difference step size and $l(x)=\max (x,0)$ is the ReLU term.
For $\alpha_2=\alpha+\delta>\alpha$, the ReLU term produces a positive
penalty whenever the network output violates
$\omega^\phi(\cdot,\alpha_2)\geq\omega^\phi(\cdot,\alpha)$. Hence, the overall training objective for the quantile network is
\begin{equation}
  \mathcal{L}_{\mathrm{quantile}}(\phi)
  = \E\left[c_2\,\mathcal{L}_{\mathrm{QR}}(\phi)
  + c_3\,\mathcal{L}_{\mathrm{mono}}(\phi)\right].
  \label{eq:var_total_loss}
\end{equation}

With the quantile network, the objective can be approximated numerically using the midpoint Riemann sum. Suppose that the distortion function $h$ has finitely many right jumps at
$\mathcal J_h=\{\xi_1,\ldots,\xi_J\}$.
Partitioning
$[0,1]$ into $n$ equal subintervals $\{[t_{i-1},t_i]\}_{i=1}^n$ with
midpoints $t_i^{\mathrm{mid}}=(t_{i-1}+t_i)/2$, the estimate for the objective  is
\begin{equation}
  \hat{\mathcal{C}}_h
  = -\sum_{i=1}^n
      \omega^\phi\left(\mathbf{s}_0,\;1-t_i^{\mathrm{mid}}\right)\,
      \left(h^c(t_i) - h^c(t_{i-1}) \right)-
\sum_{j=1}^{J}
\Delta h(\xi_j)
\omega^\phi
\left(\mathbf s_0,1-\xi_j
\right),
  \label{eq:midpoint}
\end{equation}
where $h^{c}$ denotes the continuous component of $h$. For a  left-continuous distortion function, we have $\Delta h(\xi_j)=h(\xi_j+)-h(\xi_j)$.
For a right-continuous distortion function, we have $\Delta h(\xi_j)=h(\xi_j)-h(\xi_j-)$.

%To further reduce the low-order truncation error of the numerical integration \eqref{eq:midpoint}, we construct a resolution sequence $n_k = n_0\cdot 2^k$ ($k=0,1,\ldots,L_R$) from an initial step count $n_0$, compute the estimates $\hat{\mathcal{C}}^{(n_k)}$ at each resolution level, and then apply the Richardson extrapolation recursion:\begin{equation}I_k^{(j+1)}= 2\,I_{k+1}^{(j)} - I_k^{(j)},\quad k = 0,1,\ldots,L_R-j-1,\label{eq:richardson}\end{equation}with initial values $I_k^{(0)} = \hat{\mathcal{C}}^{(n_k)}$. After $L$ levels of extrapolation, the high-accuracy estimate $I_0^{(L_R)}$ is obtained. If the distortion function is not continuous, we directly estimate \eqref{eq:midpoint}.

\begin{comment}
{\color{red}\begin{remark}
    We may consider this objective:
\begin{equation*}
\begin{aligned}
\inf_{\pi \in \Pi}\left( \E \left[  \sum_{t = 0}^{T-1} c(S_t, A_t,S_{t+1}) \bigg| Y_t=y \right]+\int \left[\sum_{t = 0}^{T-1} c(S_t,A_t,S_{t+1}) \right] \d h \circ \p_{Y_t = y}\right)% \\
%&= \min_{d \in \D} \min_{y \in \R} \E \left[ y + \frac{1}{1-\alpha} \left( \frac{1}{T} \sum_{t = 0}^{T-1} c(X_t, A_t) - y \right)^+ \right]\\
%&= \min_{y \in \R} \left\{ y + \frac{1}{1-\alpha} \min_{d \in \D} \E \left[ \left( \frac{1}{T} \sum_{t = 0}^{T-1} c(X_t, A_t) - y \right)^+ \right] \right\}. 
\end{aligned}
\end{equation*}
\end{remark}}
\end{comment}

The details of the RL algorithm are provided in Algorithm \ref{alg:RL}.

\begin{breakablealgorithm}
\caption{Risk-Sensitive Reinforcement Learning}
\label{alg:RL}
\begin{algorithmic}[1]
\STATE \textbf{Initialize} policy network $\pi^\theta$, value network $V^{\pi, \psi}$, and quantile network
       $\omega^\phi(s_0,\alpha)$ with parameters $\theta_0$, $\psi_0$, $\phi_0$, experience replay buffer $\mathcal{R}\leftarrow\emptyset$
       with capacity $C_{\mathcal{R}}$, initial gain $W_0$, distortion function $h$, discount factor $\gamma$, GAE decay coefficient $\lambda_{\mathrm{GAE}}$
\FOR{$k = 0,\,1,\,\ldots,\,K-1$}
    \STATE Initialize buffers $\mathcal{D}_k\leftarrow\emptyset$ and $\mathcal{H}_k\leftarrow\emptyset$.
    \STATE Randomly sample $N$ trajectories from the training set. 
    \STATE Draw $N$ confidence levels $\alpha^{(n)}$ with $  n=1,\ldots,N.$
    \FOR{$t = 0,\,1,\,\ldots,\,T-1$}
        \STATE Observe current states $\{\mathbf{s}_t^{(n)}\}_{n=1}^{N}$ from the vectorized environment.
        \STATE Sample actions $a_t^{(n)}\!\sim\!\pi^{\theta_k}(\cdot\mid \mathbf{s}_t^{(n)})$;
               record $\log\pi^{\theta_k}(a_t^{(n)}\mid \mathbf{s}_t^{(n)})$ and $V^{\pi, \psi}(\mathbf{s}_t^{(n)})$.
        \STATE Execute $a_t^{(n)}$; receive next states $\mathbf{s}_{t+1}^{(n)}$,
               intermediate rewards $R_{t+1}^{(n)}$, and gain $W_{t+1}^{(n)}$.
    \ENDFOR
    \STATE Record initial states $\mathbf{s}_0^{(n)}$ and normalized terminal gain
           $\tilde{W}_T^{(n)}=W_T^{(n)}/W_0$ for $n=1,\ldots,N$.
    \FOR{$n = 1,\,\ldots,\,N$}
        \STATE Approximate the objective value by \eqref{eq:midpoint}: 
               \begin{equation*}
                   \hat{\mathcal{C}}_h=-\sum_{i=1}^n
      \omega^\phi\left(\mathbf{s}_0^{(n)},\;1-t_i^{\mathrm{mid}}\right)\,
      \left(h^c(t_i) - h^c(t_{i-1}) \right)-
\sum_{j=1}^{J}
\Delta h(\xi_j)
\omega^\phi
\left(\mathbf s_0^{(n)},1-\xi_j
\right).
               \end{equation*}
        \STATE Set the terminal reward:
               $R_T^{(n)}\leftarrow
               \hat{\mathcal{C}}_h$.
    \ENDFOR
    \FOR{$t=T{-}2,\,\ldots,\,0$}
    \STATE Compute GAE advantages and target returns:
           \begin{equation*}
           \begin{aligned}
               \hat{A}_{T-1}^{(n)} &= R_T^{(n)}-V^{\pi, \psi}(\mathbf{s}_{T-1}^{(n)}),
               \qquad G_{T-1}^{(n)} = R_T^{(n)};\\
               \delta_t^{(n)} &= R_{t+1}^{(n)}+\gamma\,V^{\pi, \psi}(\mathbf{s}_{t+1}^{(n)})-V^{\pi, \psi}(\mathbf{s}_t^{(n)}),\\
               \hat{A}_t^{(n)} &= \delta_t^{(n)}+\gamma\lambda_{\mathrm{GAE}}\,\hat{A}_{t+1}^{(n)},
               \qquad G_t^{(n)} = \hat{A}_t^{(n)}+V^{\pi, \psi}(\mathbf{s}_t^{(n)}).
               \end{aligned}
           \end{equation*}
    \ENDFOR
    \STATE Add tuples
           $\bigl\{\bigl(\mathbf{s}_t^{(n)},\,a_t^{(n)},\,\pi^{\theta_k}(a_t^{(n)}\mid s_t^{(n)}),\,      G_t^{(n)},\,\hat{A}_t^{(n)}\bigr)\bigr\}_{t=0}^{T-1}$
           to $\mathcal{D}_k$, for each $n$.
    \STATE Append $\bigl(\mathbf{s}_0^{(n)},\,\tilde{W}_T^{(n)},\,\alpha^{(n)}\bigr)$ to $\mathcal{H}_k$
           and to replay buffer $\mathcal{R}$, for each $n$.
    \STATE Set $\theta\leftarrow\theta_k$,\; $\psi\leftarrow\psi_k$,\; $\phi\leftarrow\phi_k$.
    \FOR{$m = 0,\,1,\,\ldots,\,M-1$} 
        \STATE Shuffle $\mathcal{D}_k$ and partition into mini-batches of size $n_{\mathcal{B}}$.
        \FOR{each mini-batch
             $\mathcal{B}=\{(\mathbf{s}_j,a_j,\pi_j^{\theta_\mathrm{old}},G_j,\hat{A}_j)\}_{j=1}^{n_{\mathcal{B}}}$}
            \STATE Normalize advantages $\hat{A}_j$.
            \STATE Update policy parameter:
                   \begin{equation*}
                       \theta\leftarrow\theta
                       -\eta_k\!\left[\frac{1}{n_{\mathcal{B}}}\sum_{j=1}^{n_{\mathcal{B}}}\left(
                       \nabla_{\theta}\mathcal{L}_{\mathrm{PPO}}(\theta)
                       +c_0\nabla_{\theta}\mathcal{L}_{\mathrm{ent}}(\theta)\right)\right]
                       ;
                   \end{equation*}
            \STATE Update value parameter:
                   \begin{equation*}
                       \psi\leftarrow\psi
                       -c_1\eta_k\!\left[\frac{1}{n_{\mathcal{B}}}\sum_{j=1}^{n_{\mathcal{B}}}
                       \nabla_{\psi}\mathcal{L}_V(\psi)\right]
                       ;
                   \end{equation*}
        \ENDFOR
    \ENDFOR
        \FOR{$p = 1,\,\ldots,\,P$}
            \STATE Sample mini-batch
                   $\mathcal{H}'=\bigl\{\bigl(\mathbf{s}_0^{(j)},1-\tilde{W}_T^{(j)},\alpha^{(j)}\bigr)
                   \bigr\}_{j=1}^{n_\phi}$
                   uniformly from $\mathcal{R}$.
            \STATE Update quantile network parameter:
                   \begin{equation*}
                       \phi\leftarrow\phi
                       -\eta_k^{\phi}\,\nabla_{\phi}\!\left[\frac{1}{n^\phi}
                       \sum_{j=1}^{n^\phi}\left( c_2L_{\mathrm{QR}}(\phi)
                       +c_3 L_{\mathrm{mono}}(\phi)\right)\right]
                       ;
                   \end{equation*}
        \ENDFOR
    \STATE $\theta_{k+1}\leftarrow\theta$;\quad
           $\psi_{k+1}\leftarrow\psi$;\quad
           $\phi_{k+1}\leftarrow\phi$.
\ENDFOR
\end{algorithmic}
\end{breakablealgorithm}

\section{Numerical Experiment}\label{sec:Numerical Experiment}
In our numerical experiments, we conduct the relevant studies in a financial investment setting. We evaluate the IRL and RL components of our framework separately in order to provide a detailed and transparent assessment of the effectiveness of each module.

\subsection{Estimation of the Distortion Function}
We set 
\begin{equation*}
\begin{aligned}
    \hat{\mathcal{H}}=&\bigg\{\mathbf{1}_{\{t>1-\alpha_1\}}, \frac{t}{1-\alpha_1}\wedge 1,   \mathbf{1}_{\{\alpha_1\ge t\ge 1-\alpha_1\}}, \frac{t}{1-\alpha_1} \wedge 1 + \frac{\alpha_1 - t}{1-\alpha_1} \wedge 0, \mathbf{1}_{\{0<t<1\}}, \\&\quad \omega_1 \left( \frac{t}{1-\alpha_2} \wedge 1 \right) + \omega_2 \left( \frac{t}{1-\beta_1} \wedge 1 \right) + \omega_3 \mathbf{1}_{\{t > 1 - \alpha_2\}}, \left( \frac{\max(t-1+\beta_2, 0)}{\beta_2 - \alpha_3} \right) \wedge 1\bigg\},
\end{aligned}
\end{equation*}
where $\alpha_1\in \{0.7, 0.8, 0.9\}$, $\alpha_2=0.7$, $\beta_1=0.9$, $(\alpha_3,\beta_2)\in \{(0.7, 0.9), (0.7, 0.8), (0.8, 0.9)\}$, and $(\omega_1, \omega_2, \omega_3) = \bigl\{(0.4, 0.4, 0.2)\allowbreak, (0.3, 0.5, 0.2)\allowbreak, (0.5, 0.3, 0.2)\bigr\}$.
 Each distortion function in $\hat{ \mathcal{H}}$ corresponds to a specific riskmetric: Value-at-Risk ($\text{VaR}_\alpha$), Expected Shortfall ($\text{ES}_\alpha$, $\text{CVaR}_\alpha$),  Inter-Quantile Range ($\text{IQR}_\alpha$), Inter-ES Range ($\text{IER}_\alpha$), Range, Range Value-at-Risk ($\text{RVaR}_{\alpha,\beta}$) and GlueVaR (see \cite{WWW20}). We select 100 stocks from the S$\&$P 500 constituents to form the basis of our choice set $G$, utilizing their historical data to compute the risk thresholds at the 20 quantile points. We test our estimation algorithm under the two decision-making models: (1) The agent’s probabilities of choosing actions are given by \eqref{eq:probability-agent-a1} and \eqref{eq:probability-agent-a2}; (2) The agent selects the choice consistent with his or her risk preference at a fixed probability $  \hat{p} > 0.5  $ and the other choice is selected with probability $  1 - \hat{p}  $.  We set the learning rate $\kappa=0.2$ and exploration rate $\epsilon=0.01$. Figure \ref{fig:random-or-adaptive} shows that our algorithm accurately identifies the agent’s risk preference under both adaptive questioning strategy ($\epsilon=0.01$) and random questioning strategy ($\epsilon=1$). When the adaptive questioning method is used, the algorithm converges faster and exhibits narrower uncertainty bands.
\begin{figure}[htbp]   
    \centering
    \includegraphics[width=0.5\textwidth]{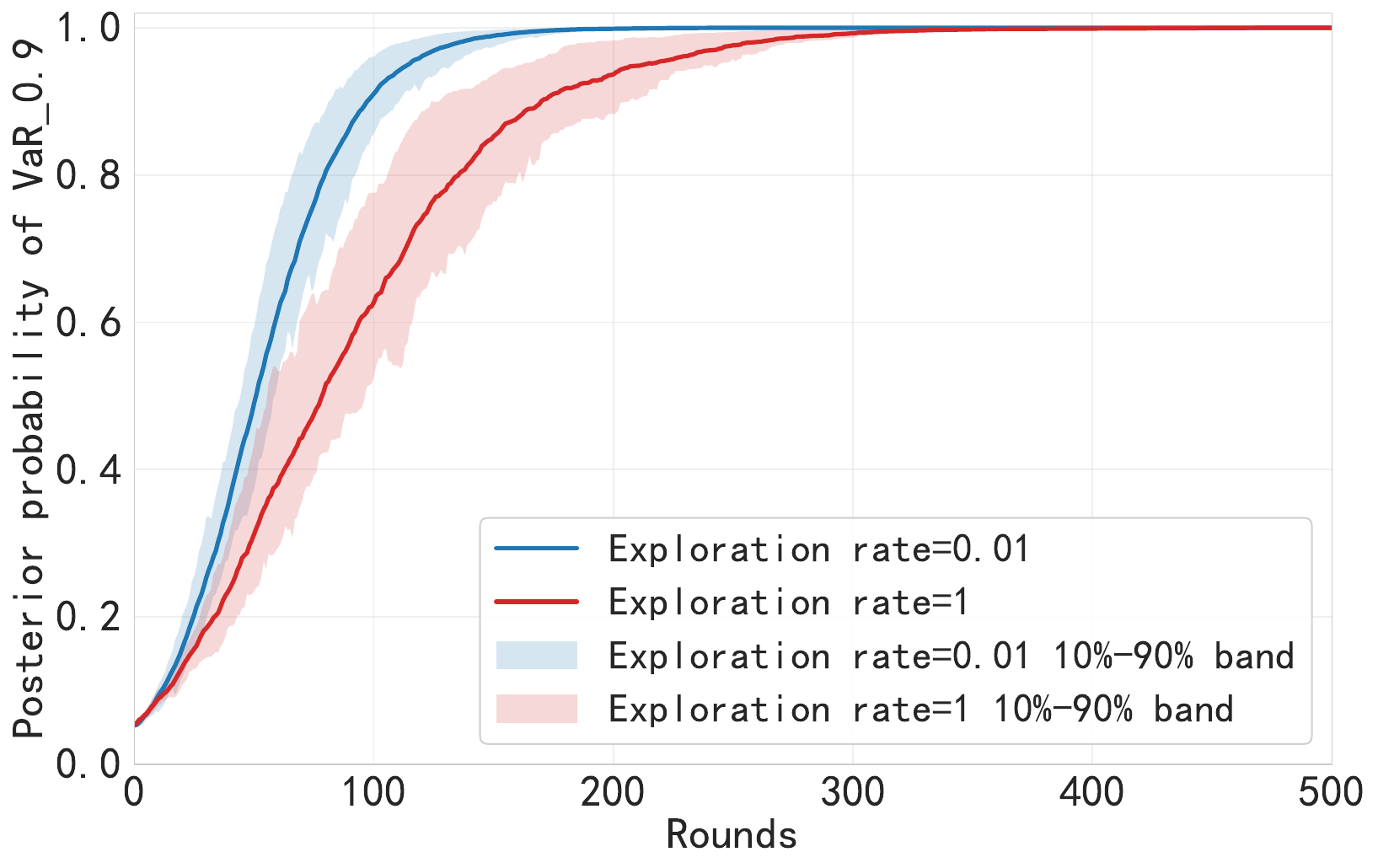} 
    \caption{Posterior probabilities of the agent’s true risk preference (VaR$_{0.9}$) at each round of the IRL algorithm when adopting an adaptive questioning method with exploration ($\epsilon=0.01$) or selecting the two distortion functions to be distinguished completely at random ($\epsilon=1$). The probability that the agent chooses the option consistent with his or her risk preference is fixed at 0.9.}
    \label{fig:random-or-adaptive}
\end{figure}

\subsubsection{Specific Decision-Making Model with Different $\beta$}
Figures \ref{fig:VaR_0.8}-\ref{fig:RVaR_0.7_0.8} show the posterior probabilities of distortion functions in $\hat{\mathcal{H}}$ at each round under different $\beta$. We can observe that the posterior probability of the true risk preference always converges to one upon algorithm termination. When $\beta=8$, the algorithm shows fast convergence and exhibits few fluctuations, because the agent behaves in a highly rational manner in the decision-making process. When $\beta=0.7$, the probability of the true risk preference converges relatively slowly, and its convergence curve displays noticeable oscillations. 
%Furthermore, the probabilities associated with other risk preferences do not quickly decay to zero. 
This slow and volatile convergence is due to the high stochasticity in the agent’s decision-making process. Even under this high level of randomness, the algorithm still converges to the true risk preference within a relatively small number of rounds.  Our algorithm not only identifies the agent’s risk preference effectively but also provides
visual evidence of the agent’s degree of decision-making rationality.

\begin{figure}[htbp]
	\small
	\centering
	\subfigure[$\beta=0.7$]{
		\begin{minipage}[t]{0.5\textwidth}
			\centering
			\includegraphics[width=\linewidth]{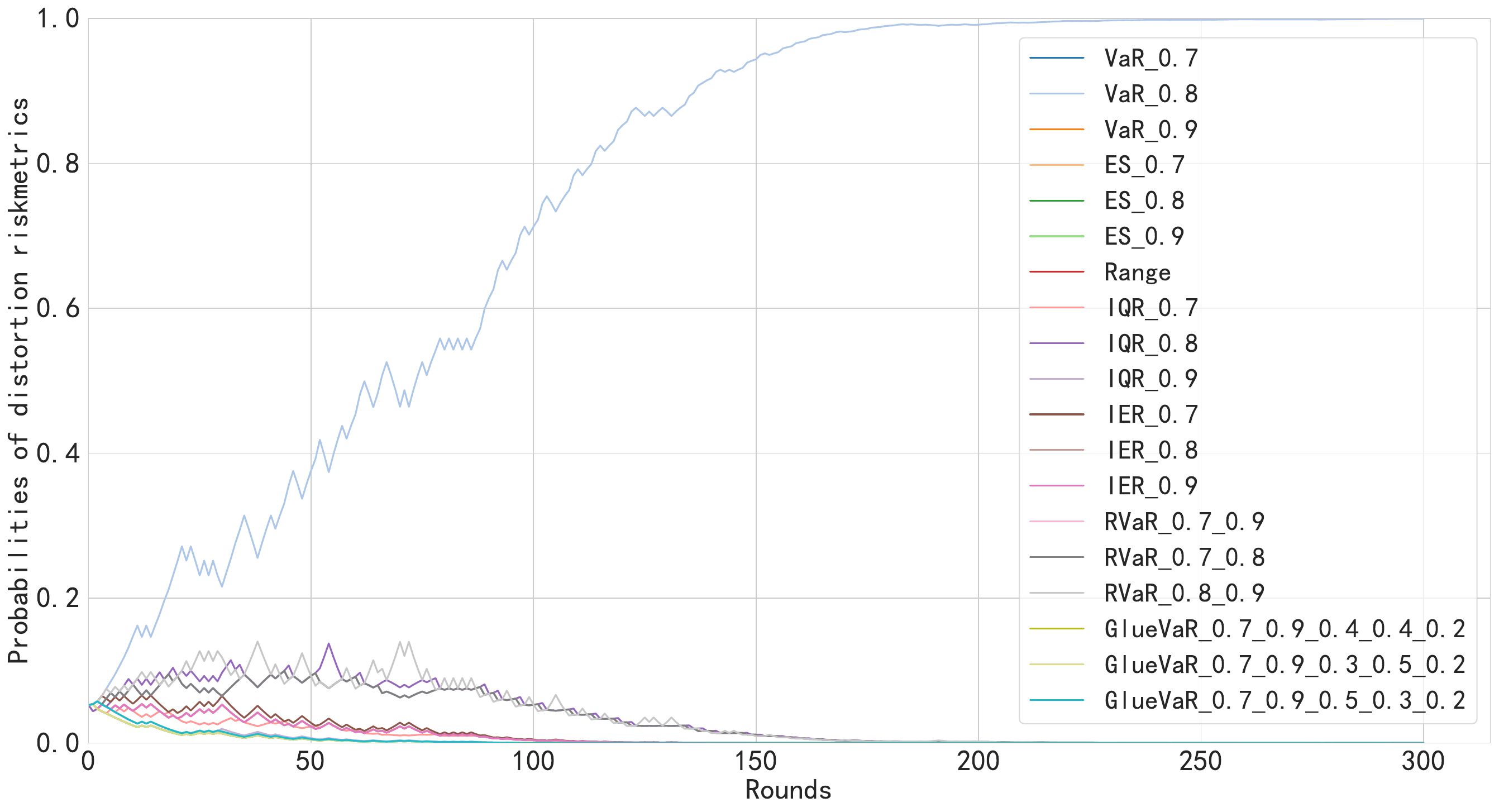}
			%\caption{$\beta=0.6$}
		\end{minipage}%
	}%
	\subfigure[$\beta=8$]{
		\begin{minipage}[t]{0.5\textwidth}
			\centering
			\includegraphics[width=\linewidth]{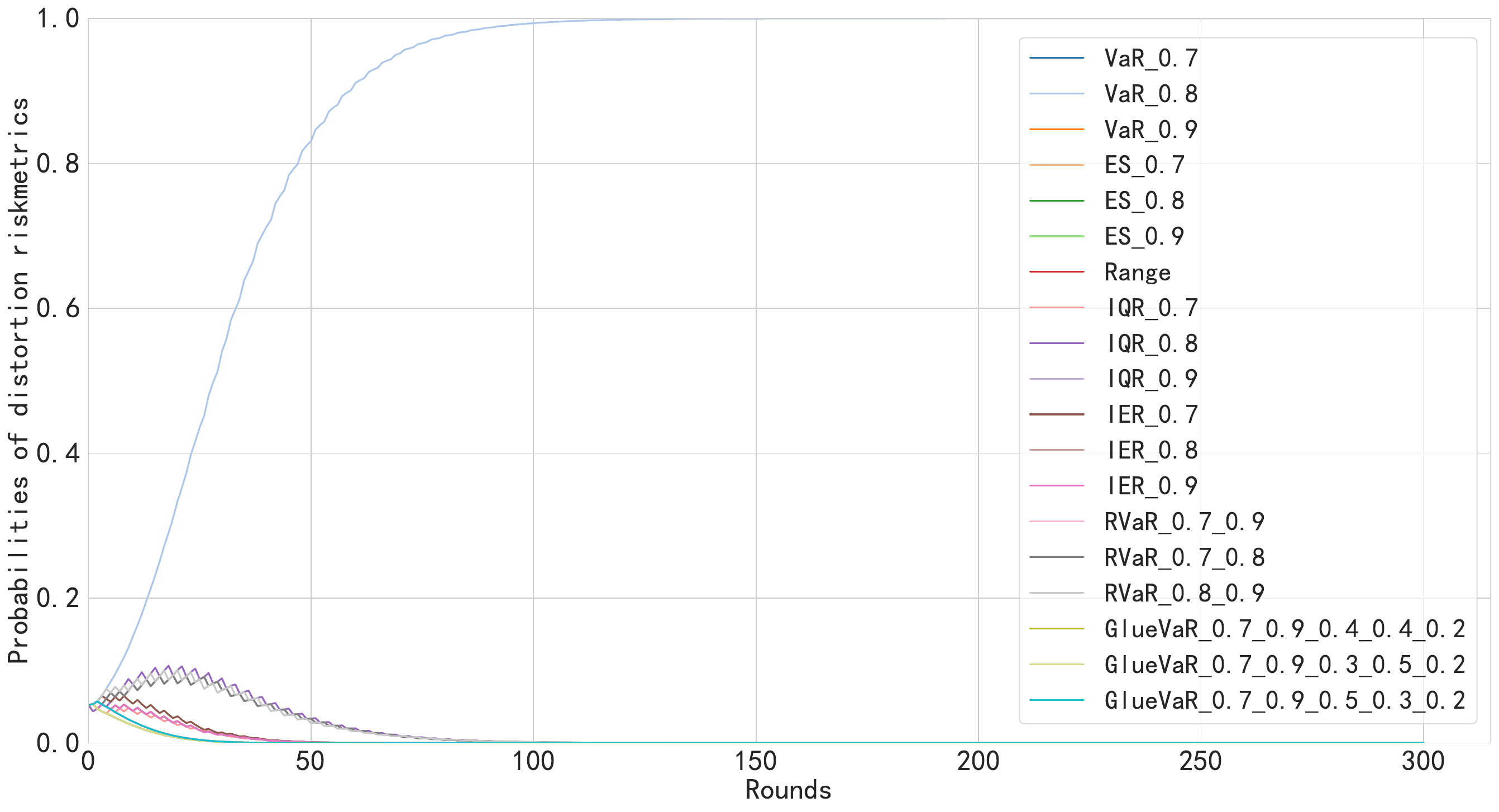}
			%\caption{$\beta=8$}
		\end{minipage}%
	}%   
	\centering
	\caption{Posterior probabilities of the distortion riskmetrics at each round. The agent's true risk preference is $\text{VaR}_{0.8}$ that belongs to $\hat{\H}$. The posterior probability of $\text{VaR}_{0.8}$ converges to 1.}
	\label{fig:VaR_0.8}
\end{figure}

\begin{figure}[htbp]
	\small
	\centering
	\subfigure[$\beta=0.7$]{
		\begin{minipage}[t]{0.5\textwidth}
			\centering
			\includegraphics[width=\linewidth]{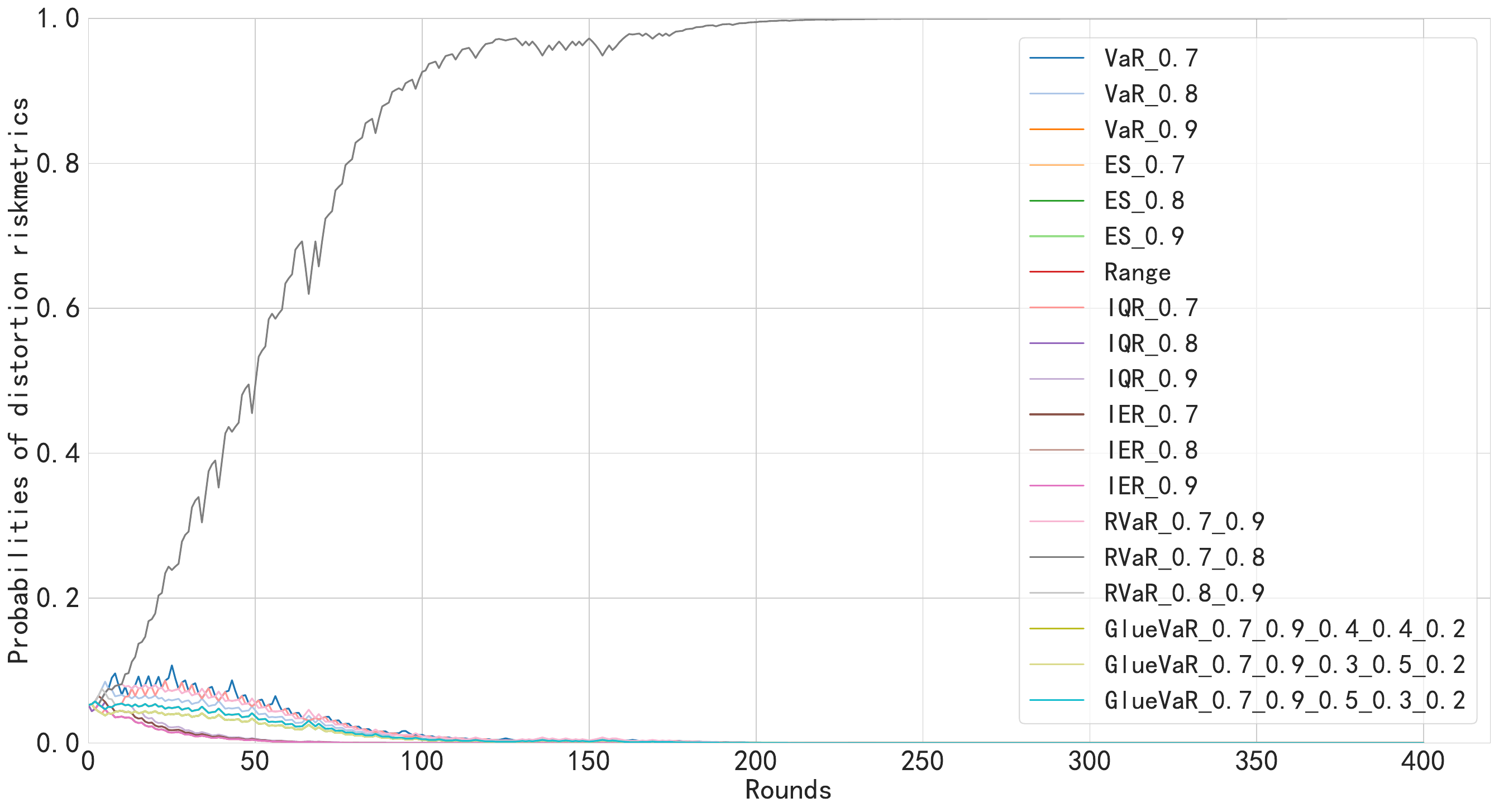}
			%\caption{$\beta=0.6$}
		\end{minipage}%
	}%
	\subfigure[$\beta=8$]{
		\begin{minipage}[t]{0.5\textwidth}
			\centering
			\includegraphics[width=\linewidth]{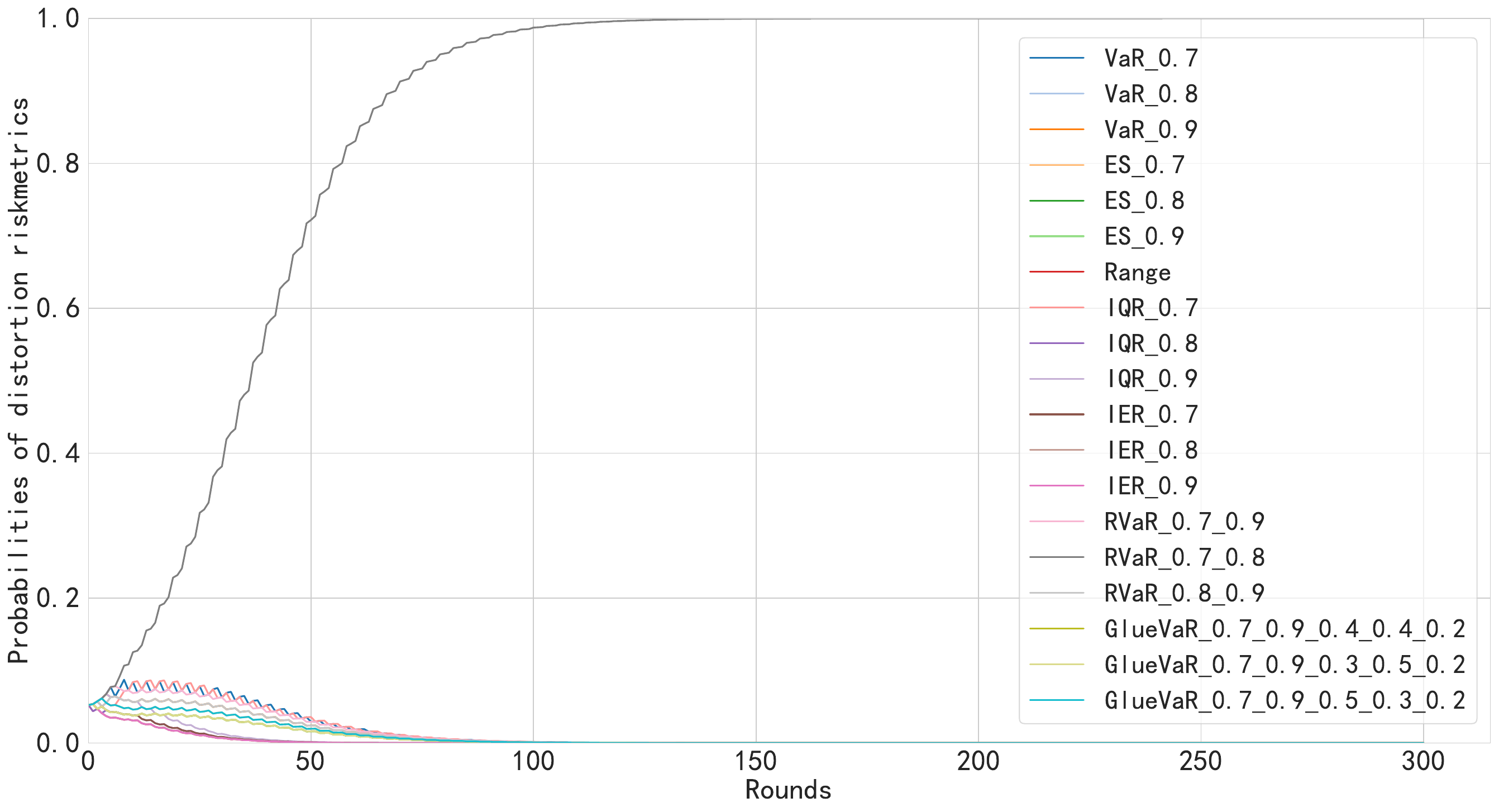}
			%\caption{$\beta=8$}
		\end{minipage}%
	}%   
	\centering
	\caption{Posterior probabilities of the distortion riskmetrics at each round. The agent's true risk preference is $\text{RVaR}_{0.7,0.8}$ that belongs to $\hat{\H}$. The posterior probability of $\text{RVaR}_{0.7,0.8}$ converges to 1.}
	\label{fig:RVaR_0.7_0.8}
\end{figure}

We also investigate scenarios in which the model is misspecified: the agent's true distortion function is not in the defined set $\hat{\mathcal{H}}$. To illustrate this, we consider the following distortion functions that are not in $\hat{\H}$: 
\begin{equation*}
    \left( \frac{\max(t-0.1, 0)}{0.3} \right) \wedge 1, \quad 0.3 \left( \frac{t}{0.2} \wedge 1 \right) + 0.5 \left( \frac{t}{0.1} \wedge 1 \right) + 0.2 \mathbf{1}_{\{t > 0.2\}},
\end{equation*}
which are $\text{RVaR}_{0.6,0.9}$ and $0.3\text{ES}_{0.8}+0.5\text{ES}_{0.9}+0.2\text{VaR}_{0.8}$, respectively.
Figures \ref{fig:RVaR_0.6_0.9}-\ref{fig:GlueVaR_0.8_0.9_0.3_0.5_0.2} present the evolution of the probabilities of the distortion functions in $\hat{\mathcal{H}}$ when the true distortion function is not in the set $\hat{\mathcal{H}}$. We can find that our algorithm will finally converge to a risk preference in $\hat{\mathcal{H}}$ that is close to the true risk preference of the agent for both small and large values of $\beta$  within a
relatively small number of rounds. In the case of the true risk preferences $\text{RVaR}_{0.6,0.9}$ and $0.3\text{ES}_{0.8}+0.5\text{ES}_{0.9}+0.2\text{VaR}_{0.8}$, the algorithm converges to $\text{RVaR}_{0.7,0.9}$ and $0.3\text{ES}_{0.7}+0.5\text{ES}_{0.9}+0.2\text{VaR}_{0.7}$ respectively, which are close to the true risk preferences. Our algorithm can effectively address the challenges arising from both the model misspecification and the uncertainty of the agent in making decisions. 

In all the above examples, we can preliminarily determine the agent’s preferred distortion function within 100 rounds, even when its structure is complex or the agent behaves randomly.

\begin{figure}[htbp]
	\small
	\centering
	\subfigure[$\beta=0.7$]{
		\begin{minipage}[t]{0.5\textwidth}
			\centering
			\includegraphics[width=\linewidth]{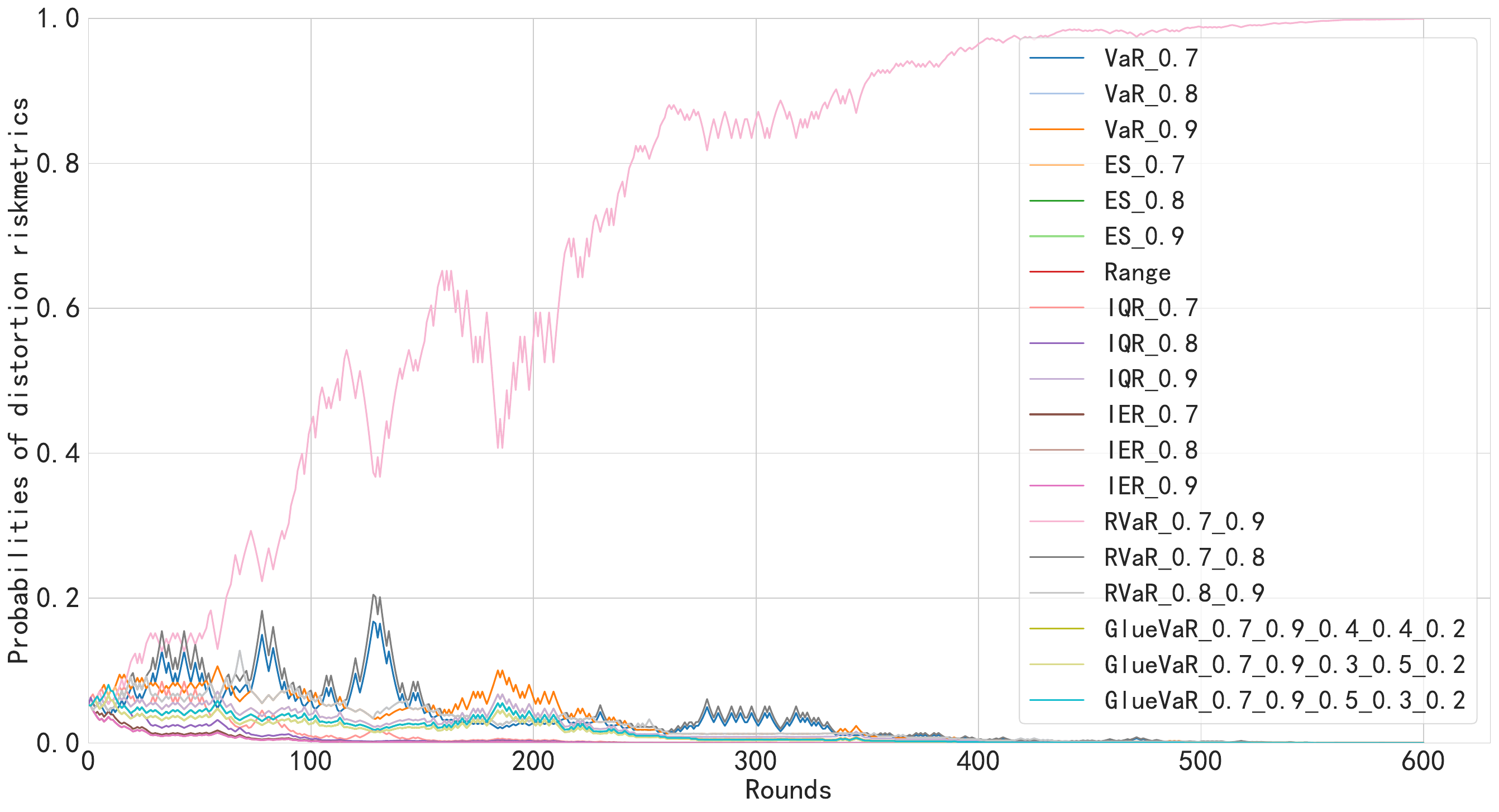}
			%\caption{$\beta=0.6$}
		\end{minipage}%
	}%
	\subfigure[$\beta=8$]{
		\begin{minipage}[t]{0.5\textwidth}
			\centering
			\includegraphics[width=\linewidth]{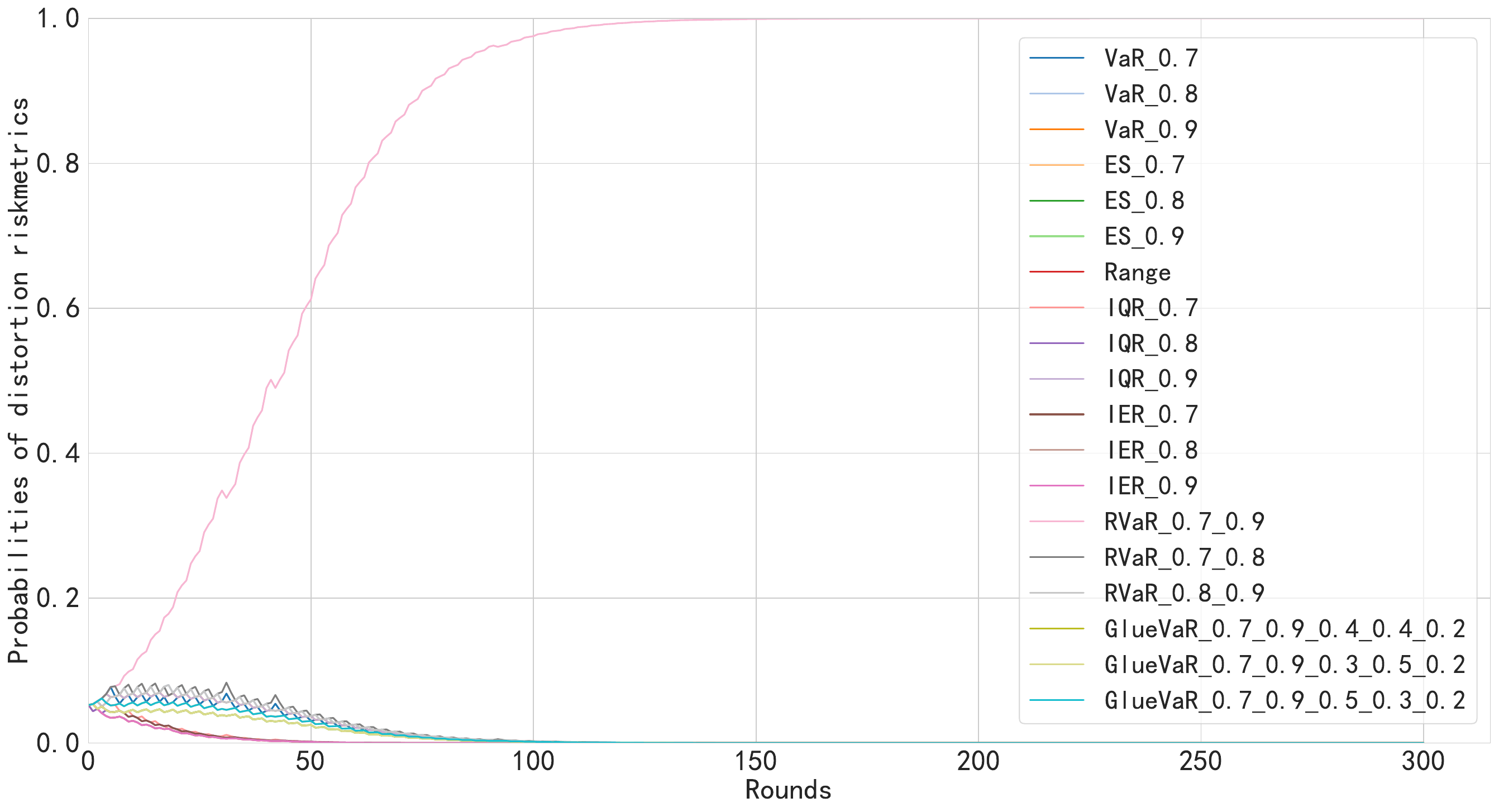}
			%\caption{$\beta=8$}
		\end{minipage}%
	}%   
	\centering
	\caption{Posterior probabilities of the distortion riskmetrics at each round. The agent's true risk preference is $\text{RVaR}_{0.6,0.9}$ that is not included in $\hat{\H}$. The posterior probability of $\text{RVaR}_{0.7,0.9}$ converges to 1.}
	\label{fig:RVaR_0.6_0.9}
\end{figure}

\begin{figure}[htbp]
	\small
	\centering
	\subfigure[$\beta=0.7$]{
		\begin{minipage}[t]{0.5\textwidth}
			\centering
			\includegraphics[width=\linewidth]{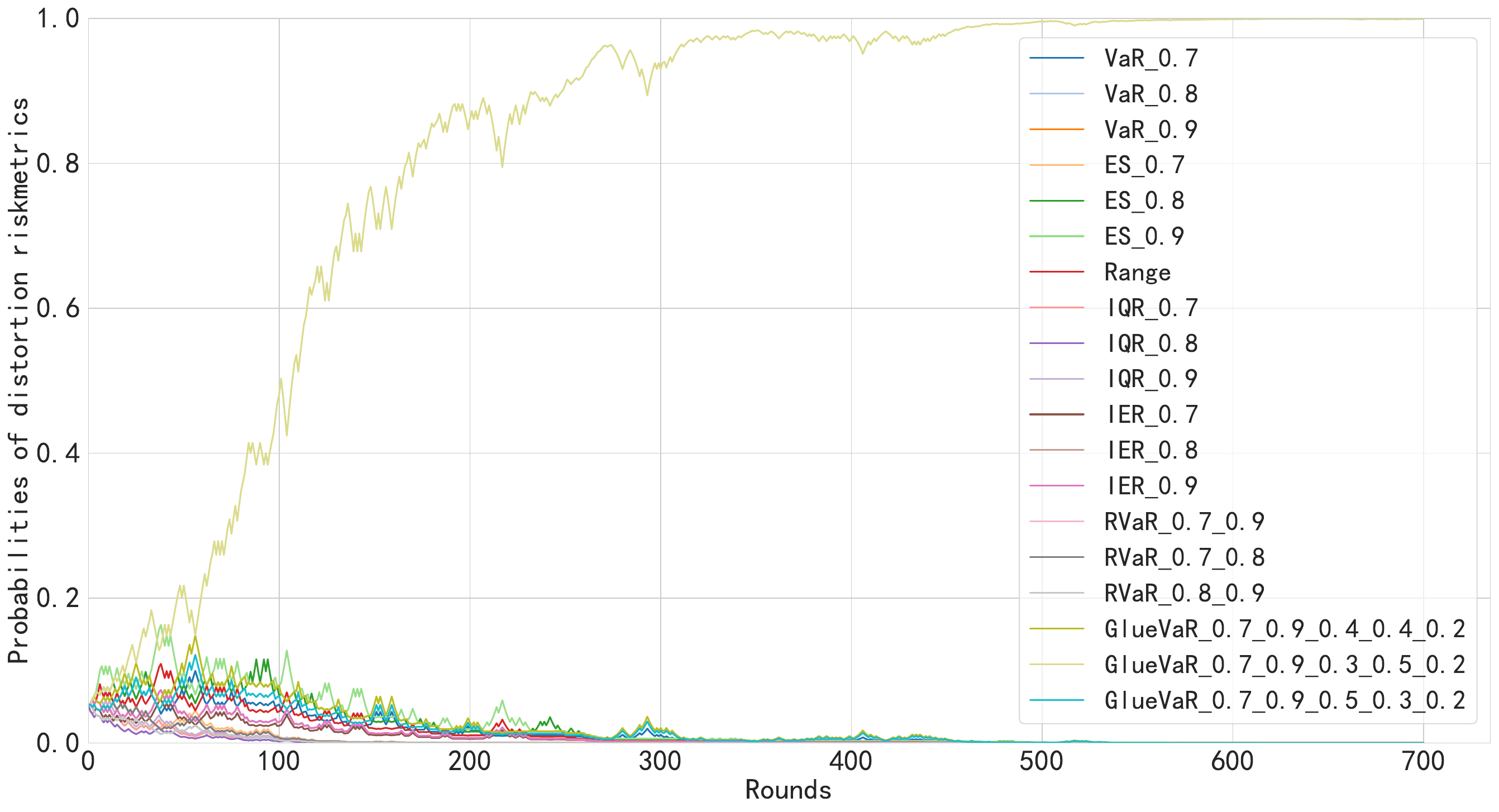}
			%\caption{$\beta=0.6$}
		\end{minipage}%
	}%
	\subfigure[$\beta=8$]{
		\begin{minipage}[t]{0.5\textwidth}
			\centering
			\includegraphics[width=\linewidth]{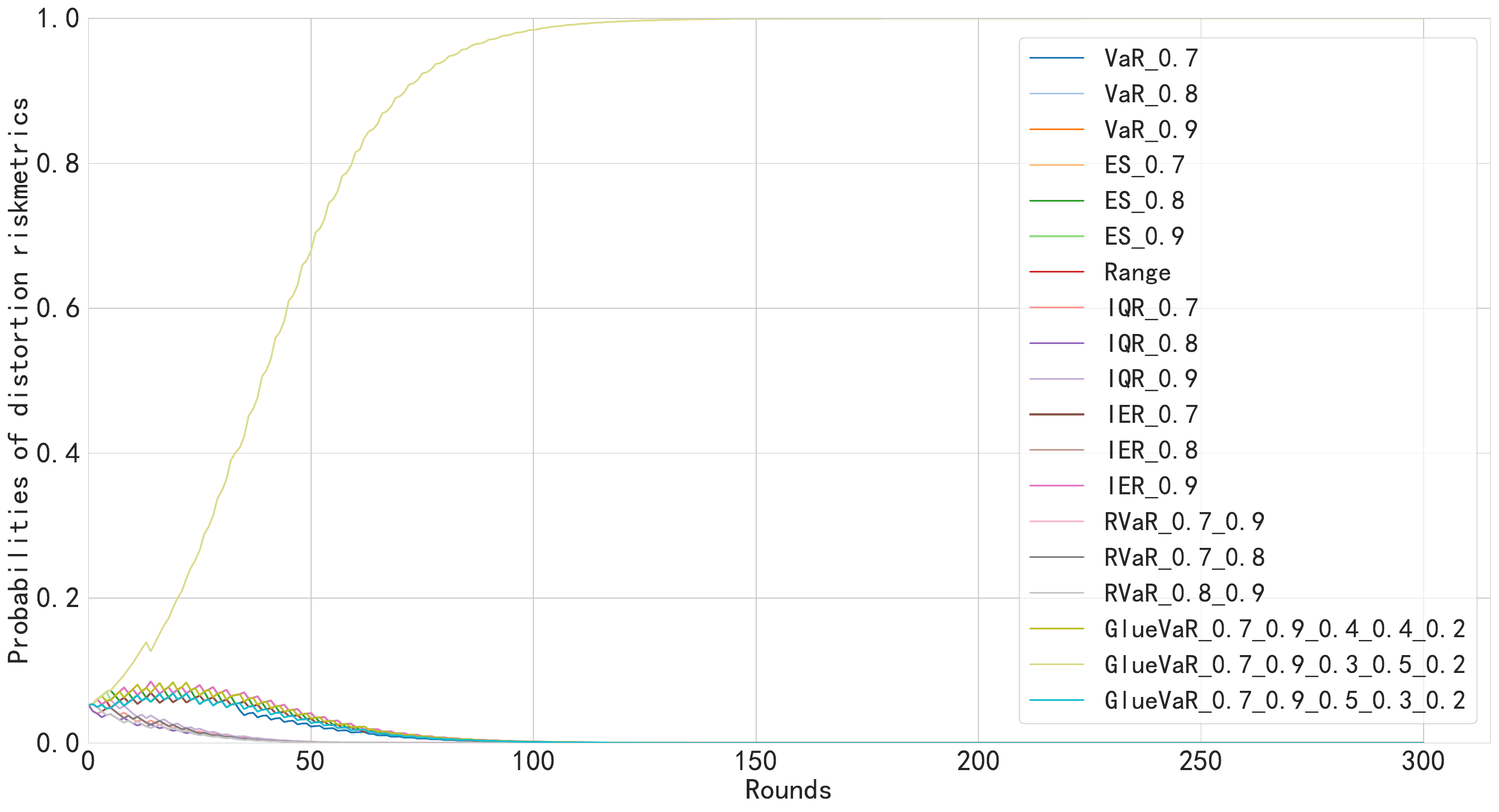}
			%\caption{$\beta=8$}
		\end{minipage}%
	}%   
	\centering
	\caption{Posterior probabilities of the distortion riskmetrics at each round. The agent's true risk preference is $0.3\text{ES}_{0.8}+0.5\text{ES}_{0.9}+0.2\text{VaR}_{0.8}$ that is not included in $\hat{\H}$. The posterior probability of $0.3\text{ES}_{0.7}+0.5\text{ES}_{0.9}+0.2\text{VaR}_{0.7}$ converges to 1.} 
	\label{fig:GlueVaR_0.8_0.9_0.3_0.5_0.2}
\end{figure}

\subsubsection{General Decision-Making Model with Different $\hat{p}$}
Under the general decision-making model \eqref{eq:fixed-probability-agent-a1} and \eqref{eq:fixed-probability-agent-a2}, we examine a setting where the agent, driven by his or her risk preference, makes an optimal choice with a fixed probability $\hat{p}>0.5$ and a suboptimal choice with probability $1-\hat{p}$ at each round.  Figures \ref{fig:ES_0.9-new-user-model}-\ref{fig:IQR_0.9-new-user-model} show the performances of our algorithm under this setting for $\hat{p}=0.6$ and $\hat{p}=0.9$ respectively.  We find that our algorithm still performs well in this setting. When $\hat{p}=0.6$, the agent's choices are already close to being completely random, but our algorithm can still preliminarily identify the agent’s true risk preference after approximately 100 rounds. When $\hat{p}=0.9$, the algorithm converges quickly with fewer oscillations. In the case of model misspecification, we consider two distortion functions that are not in the set $  \hat{\mathcal{H}}  $: $\frac{t}{0.05}\wedge 1$ ($\text{ES}_{0.95}$) and $\mathbf{1}_{\{0.6\ge t\ge 0.4\}}$ ($\text{IQR}_{0.6}$). Figures \ref{fig:ES_0.95-new-user-model}-\ref{fig:IER_0.6-new-user-model} show that the algorithm converges to $\text{ES}_{0.9}$ and $\text{IQR}_{0.7}$ respectively, which are close to the true risk preferences that are not in the set $  \hat{\mathcal{H}}  $. The results of this experiment demonstrate the robustness, flexibility, and stability of our algorithm under different settings. 

\begin{figure}[htbp]
	\small
	\centering
	\subfigure[$\hat{p}=0.6$]{
		\begin{minipage}[t]{0.5\textwidth}
			\centering
			\includegraphics[width=\linewidth]{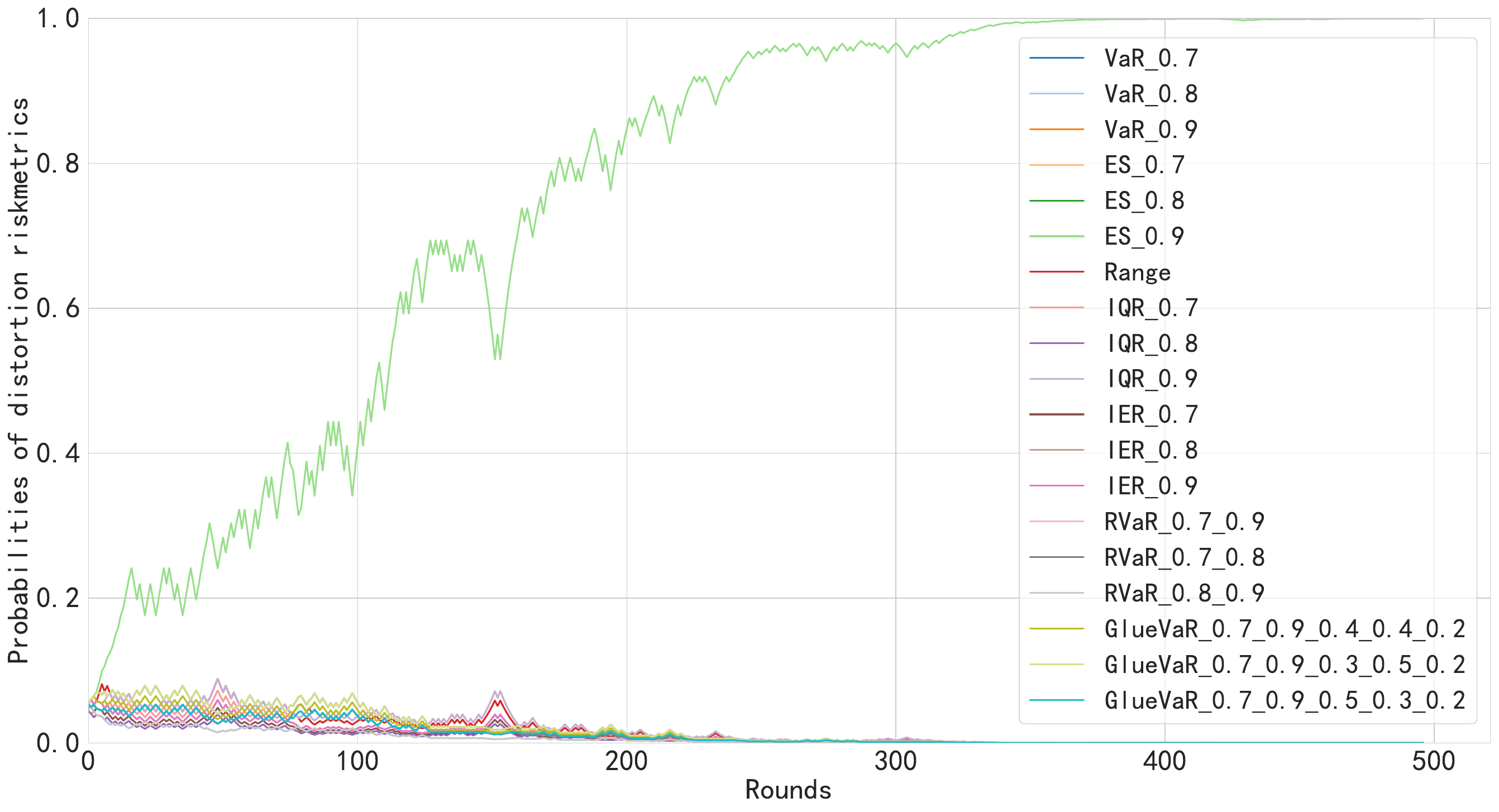}
			%\caption{$\hat{p}=0.75$}
		\end{minipage}%
	}%
	\subfigure[$\hat{p}=0.9$]{
		\begin{minipage}[t]{0.5\textwidth}
			\centering
			\includegraphics[width=\linewidth]{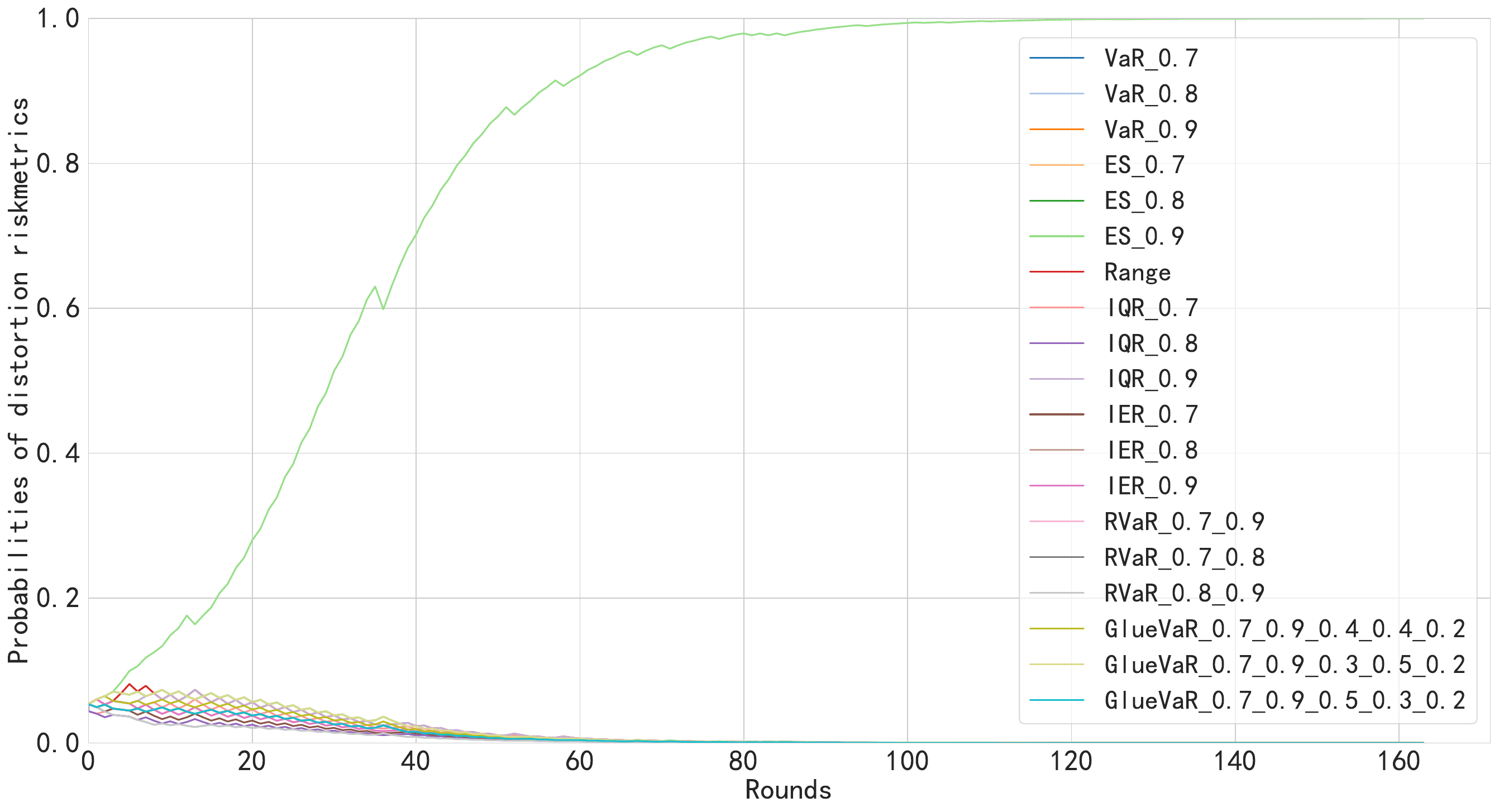}
			%\caption{$\hat{p}=0.95$}
		\end{minipage}%
	}%   
	\centering
	\caption{Posterior probabilities of the distortion riskmetrics at each round. The agent's true risk preference is $\text{ES}_{0.9}$ that belongs to $\hat{\H}$. The posterior probability of $\text{ES}_{0.9}$ converges to 1.}
	\label{fig:ES_0.9-new-user-model}
\end{figure}

\begin{figure}[htbp]
	\small
	\centering
	\subfigure[$\hat{p}=0.6$]{
		\begin{minipage}[t]{0.5\textwidth}
			\centering
			\includegraphics[width=\linewidth]{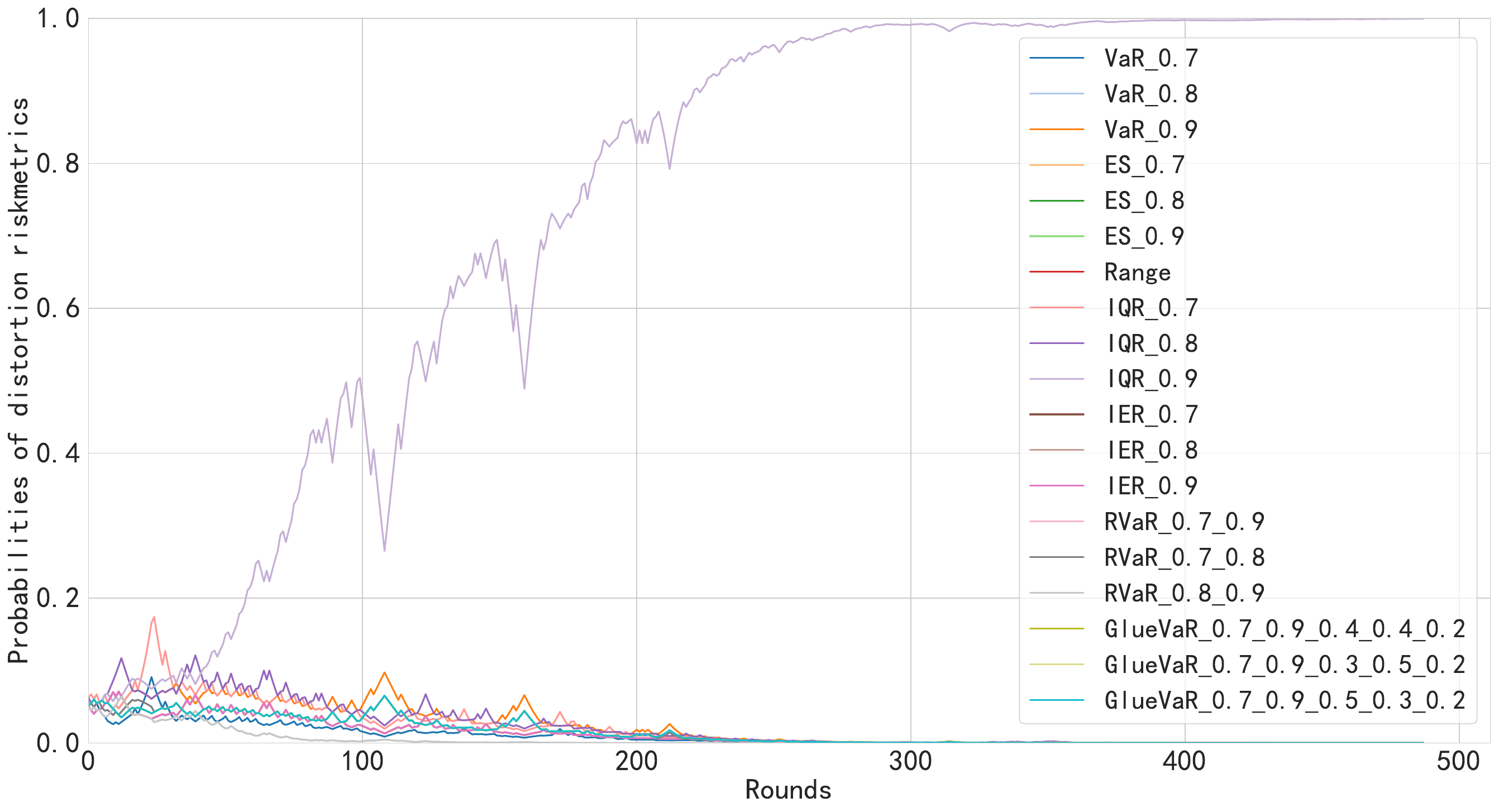}
			%\caption{$\hat{p}=0.75$}
		\end{minipage}%
	}%
	\subfigure[$\hat{p}=0.9$]{
		\begin{minipage}[t]{0.5\textwidth}
			\centering
			\includegraphics[width=\linewidth]{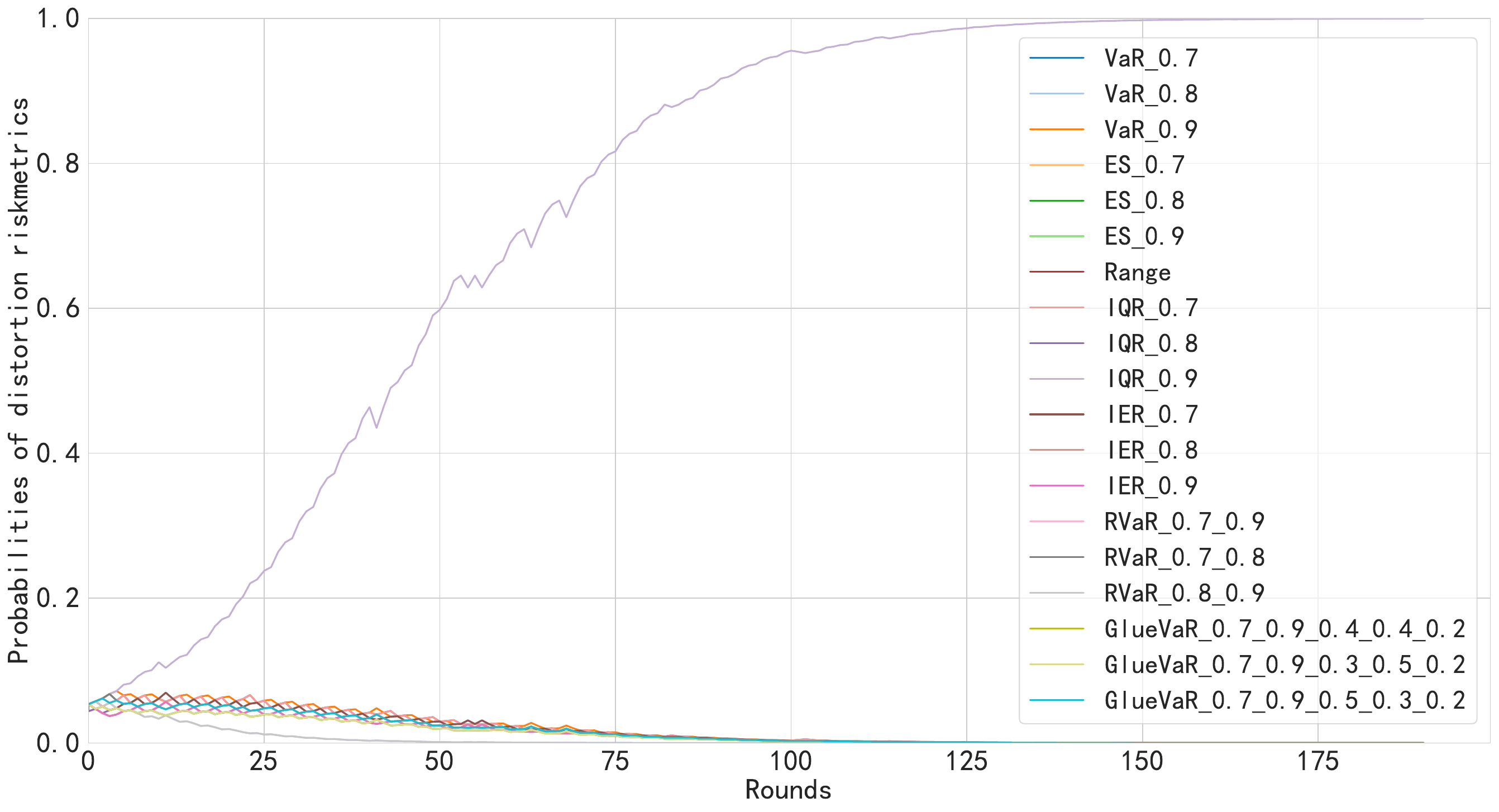}
			%\caption{$\hat{p}=0.95$}
		\end{minipage}%
	}%   
	\centering
	\caption{Posterior probabilities of the distortion riskmetrics at each round. The agent's true risk preference is $\text{IQR}_{0.9}$ that belongs to $\hat{\H}$. The posterior probability of $\text{IQR}_{0.9}$ converges to 1.}
	\label{fig:IQR_0.9-new-user-model}
\end{figure}

\begin{figure}[htbp]
	\small
	\centering
	\subfigure[$\hat{p}=0.6$]{
		\begin{minipage}[t]{0.5\textwidth}
			\centering
			\includegraphics[width=\linewidth]{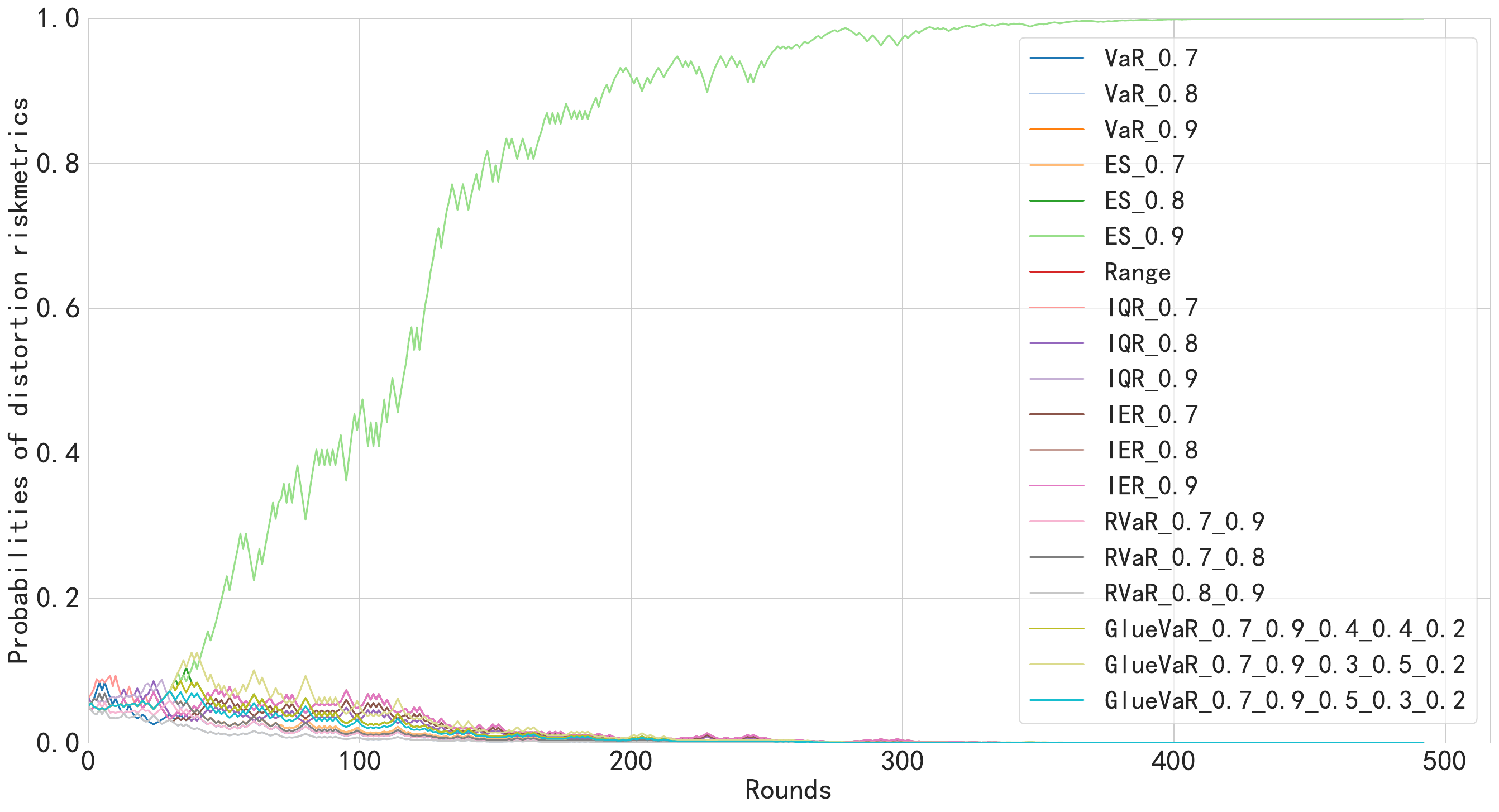}
			%\caption{$\hat{p}=0.75$}
		\end{minipage}%
	}%
	\subfigure[$\hat{p}=0.9$]{
		\begin{minipage}[t]{0.5\textwidth}
			\centering
			\includegraphics[width=\linewidth]{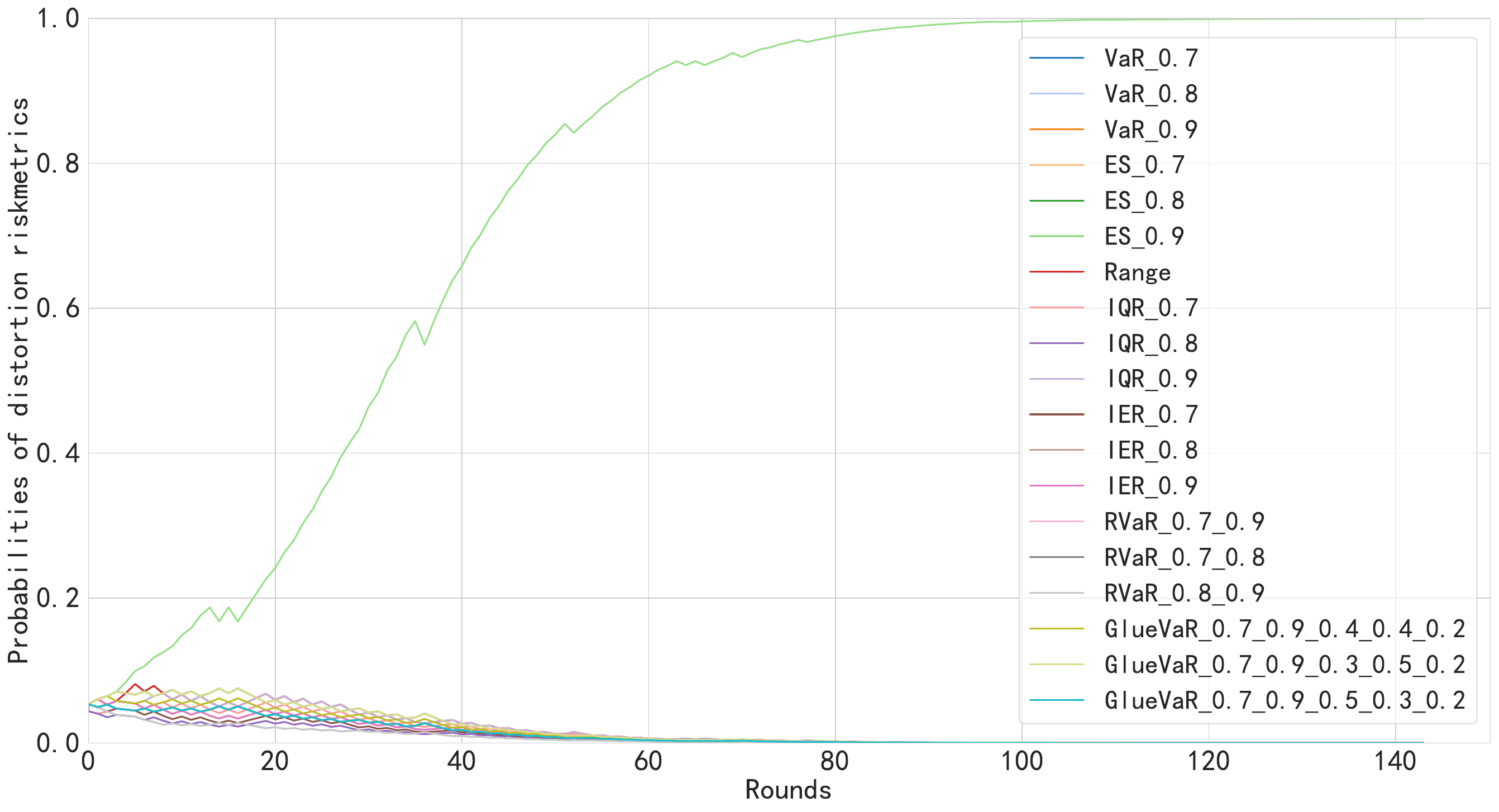}
			%\caption{$\hat{p}=0.95$}
		\end{minipage}%
	}%   
	\centering
	\caption{Posterior probabilities of the distortion riskmetrics at each round. The agent's true risk preference is $\text{ES}_{0.95}$ that is not included in $\hat{\H}$. The posterior  probability of $\text{ES}_{0.9}$ converges to 1.}
	\label{fig:ES_0.95-new-user-model}
\end{figure}

\begin{figure}[htbp]
	\small
	\centering
	\subfigure[$\hat{p}=0.6$]{
		\begin{minipage}[t]{0.5\textwidth}
			\centering
			\includegraphics[width=\linewidth]{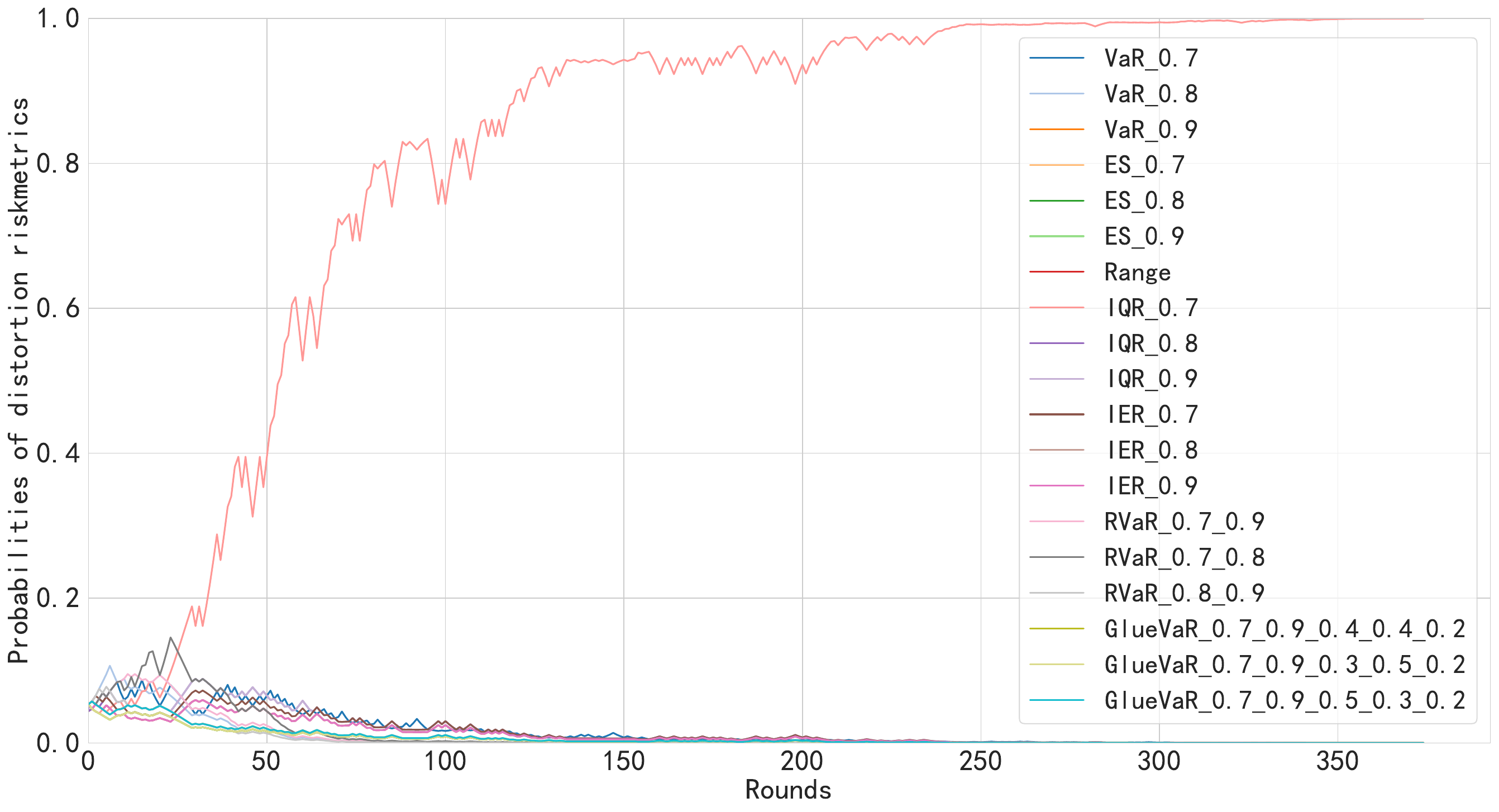}
			%\caption{$\hat{p}=0.75$}
		\end{minipage}%
	}%
	\subfigure[$\hat{p}=0.9$]{
		\begin{minipage}[t]{0.5\textwidth}
			\centering
			\includegraphics[width=\linewidth]{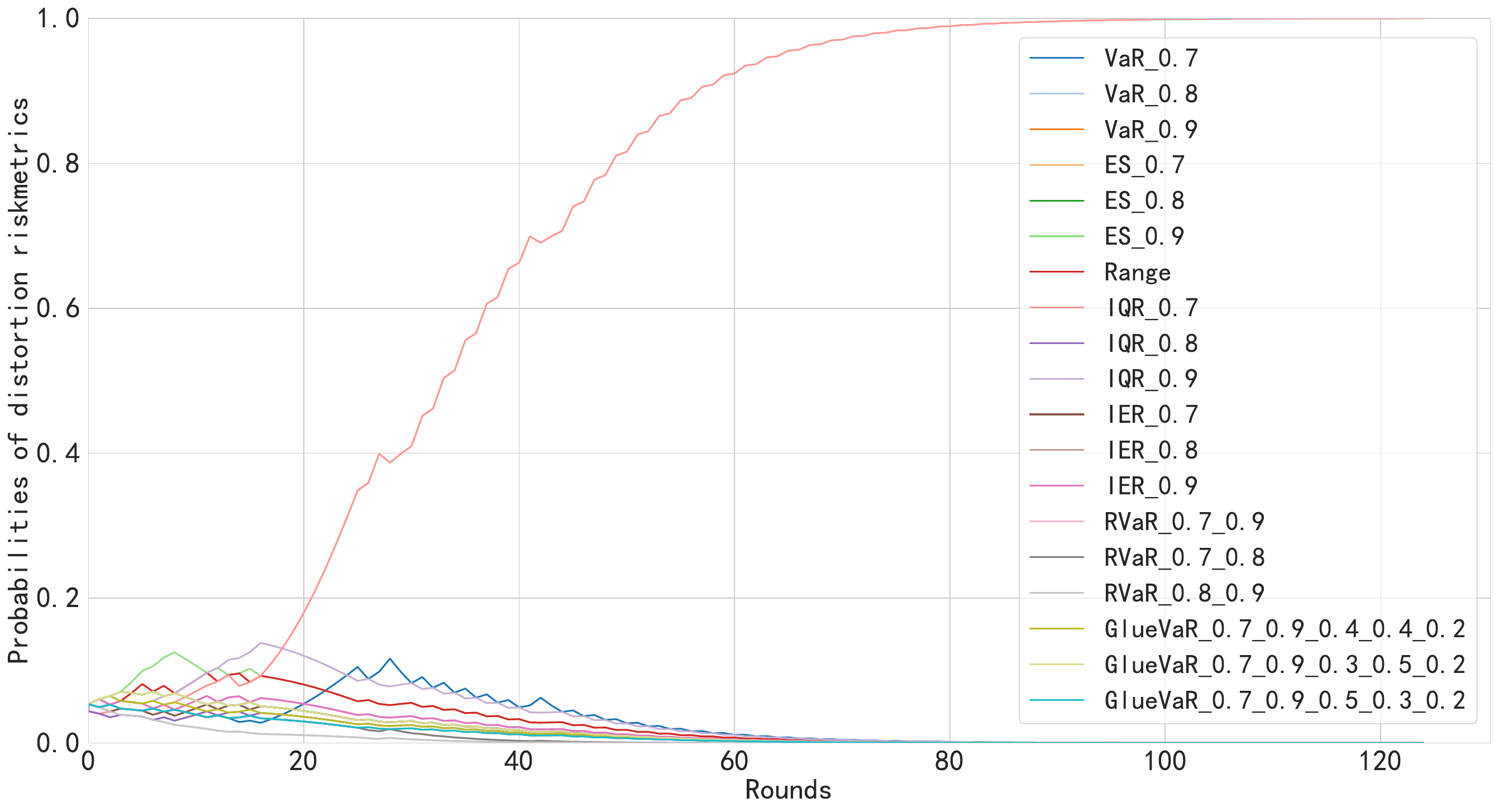}
			%\caption{$\hat{p}=0.95$}
		\end{minipage}%
	}%   
	\centering
	\caption{Posterior probabilities of the distortion riskmetrics at each round. The agent's true risk preference is $\text{IQR}_{0.6}$ that is not included in $\hat{\H}$. The posterior  probability of $\text{IQR}_{0.7}$ converges to 1.}
	\label{fig:IER_0.6-new-user-model}
\end{figure}

\subsection{Decision-Making under Distortion Riskmetrics Optimization Objectives}
After identifying the agent’s risk preference, we construct a portfolio that trades the S$\&$P 500 index and our objective is to solve problem \eqref{prob:SRP}. We assume that all trading actions occur at market close and therefore use the S$\&$P 500 daily closing prices for our tests. We collect 21-year daily data of the S$\&$P 500 index from 2005/01/03 to 2026/03/27  from Wharton Research Data Services (WRDS). The dataset consists of 5342 daily observations in total. We use the first 4333 observations as the training set and the subsequent 1009 observations as the test set. Rolling-window tests are conducted on the test set, resulting in a total of 1000 test paths. The input variables of the policy network at time $t$ include the index value $S_t$, the time-to-maturity $\tau_t$, the current wealth $W_t$ and the regime $Y_t$, while the output $\delta_t$ is the quantity of shares to buy or sell. The wealth process follows the dynamic 
\begin{equation*}
    W_{t+1}=W_{t}+\delta_{t}(S_{t+1}-S_t).
\end{equation*}
The initial wealth is $W_0=300000$ and the largest time to maturity $\tau_0$ is $\frac{10}{252}$. The loss is defined as $W_0-W_T$. We identify bull and bear market regimes from the historical data using a method similar to that of \cite{dai2010trend} and \cite{dai2016optimal}.  The market regime is classified as a bear market if the current index price falls by 19\% from its most recent peak. Conversely, it is classified as a bull market if the current index price rises by 24\% from its most recent trough. We choose the following distortion-riskmetric objectives: $  \mathrm{RVaR}_{0.6,0.9}  $, $  \mathrm{RVaR}_{0.5,0.9}  $, $  \mathrm{IQR}_{0.95}  $, $  \mathrm{IQR}_{0.9}  $, $  \mathrm{CVaR}_{0.9}  $, $  \mathrm{CVaR}_{0.95}  $, $  \mathrm{VaR}_{0.85}  $, $  \mathrm{VaR}_{0.95}  $, $  \mathrm{IER}_{0.9}  $, and GlueVaR ($0.6 \mathrm{CVaR}_{0.8}(X) + 0.25 \mathrm{CVaR}_{0.9}(X) + 0.15 \mathrm{VaR}_{0.8}(X)$). In the algorithm implementation, the coefficients in \eqref{eq:GAE} are $\gamma=1$ and $\lambda_{\text{GAE}}=0.95$;  the coefficients in \eqref{eq:total_loss} are $c_0=0$ and $c_1=0.04$; the coefficients in \eqref{eq:var_total_loss} are $c_2=0.08$ and $c_3=0.05$. 
%The parameter in the Richardson extrapolation recursion is $L_R=5$. 
As for the neural networks, there are 3 hidden layers, with each hidden layer comprising 64 neurons.  

Figures \ref{fig:Learned-policy-in-the-bull-market} and \ref{fig:Learned-policy-in-the-bear-market} compare the learned policies under different distortion-riskmetric objectives in bull and bear markets. Overall, for a fixed wealth level, the agent tends to sell at high prices and buy at low prices. As wealth increases, the agent can tolerate higher asset prices, so the buy--sell boundary generally slopes upward with diminishing marginal effects. The boundary often exhibits a hump-shaped pattern or an arch-shaped pattern: low-wealth agents buy at lower prices to recover capital, whereas high-wealth agents trade less to satisfy the risk-optimization objective. In the bear market, the buy region expands markedly relative to the bull market, and the buy--sell boundary shifts upward and rightward. This is plausible because depressed prices imply greater upside potential, while short positions are exposed to large losses from price rebounds. Since the policy patterns are broadly similar across regimes aside from the sizes of the buy and sell regions, we focus on differences among distortion-riskmetric objectives in the bull market.

\begin{itemize}
    \item Under $\mathrm{CVaR}_{0.90}$, the buy region is mainly concentrated in low-price and low-to-intermediate-wealth states. Relative to $\mathrm{CVaR}_{0.90}$, $\mathrm{CVaR}_{0.95}$ yields a smaller buy region and a larger short region, reflecting greater conservatism. A similar pattern holds for $\mathrm{VaR}_{0.85}$ and $\mathrm{VaR}_{0.95}$, with the buy-sell boundary under $\mathrm{VaR}_{0.95}$ shifting downward.
    
    \item The GlueVaR policy, combining $\mathrm{VaR}_{0.8}$, $\mathrm{CVaR}_{0.8}$, and $\mathrm{CVaR}_{0.9}$, is similar to the $\mathrm{CVaR}_{0.9}$ policy but has a wider buy region. This is because GlueVaR also incorporates the relatively less conservative risk measures $\mathrm{VaR}_{0.8}$ and $\mathrm{CVaR}_{0.8}$.
    
    \item $\mathrm{RVaR}_{0.5,0.9}$ produces the broadest buy region, because it averages quantiles only over an intermediate range and is therefore more tolerant of downside risk while emphasizing wealth recovery. Compared with $\mathrm{RVaR}_{0.5,0.9}$, $\mathrm{RVaR}_{0.6,0.9}$ shifts the buy--sell boundary uniformly downward.
    
    \item The transition from $\mathrm{IQR}_{0.9}$ to $\mathrm{IQR}_{0.95}$ is non-monotone. Under $\mathrm{IQR}_{0.95}$, the boundary is slightly higher at very low wealth levels but clearly lower at medium and high wealth levels. This reflects the fact that $\mathrm{IQR}_{\alpha}$ compresses the gap between favorable and unfavorable outcomes: loss compression dominates at very low wealth, making buying more attractive, whereas profit compression becomes more important as wealth increases, leading to more conservative policies.
    
    \item The $\mathrm{IER}_{0.90}$ policy has a geometry similar to that of $\mathrm{CVaR}_{0.90}$, but it penalizes both upper- and lower-tail averages of the loss distribution. Consequently, when the wealth is high and the price is low, the policy tends to hold zero positions to avoid extreme profits; when the wealth is low and the price is low, it tends to buy more to reduce extreme losses.
\end{itemize}

\begin{figure}[htbp]
	\small
	\centering
	\subfigure[$\text{CVaR}_{0.90}$]{
		\begin{minipage}[t]{0.24\textwidth}
			\centering
			\includegraphics[width=\linewidth]{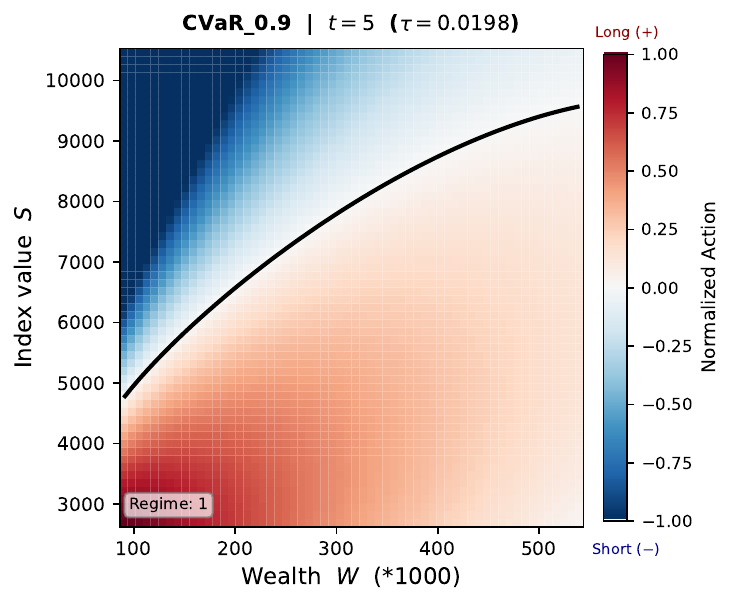}
			%\caption{$\hat{p}=0.75$}
		\end{minipage}%
	}%
	\subfigure[$\text{CVaR}_{0.95}$]{
		\begin{minipage}[t]{0.24\textwidth}
			\centering
			\includegraphics[width=\linewidth]{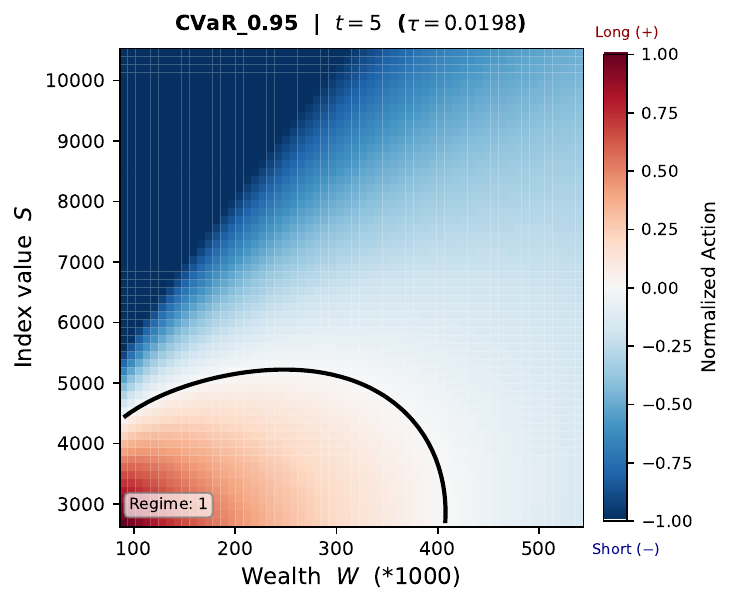}
			%\caption{$\hat{p}=0.95$}
		\end{minipage}%
	}%   
    \subfigure[GlueVaR]{
		\begin{minipage}[t]{0.24\textwidth}
			\centering
			\includegraphics[width=\linewidth]{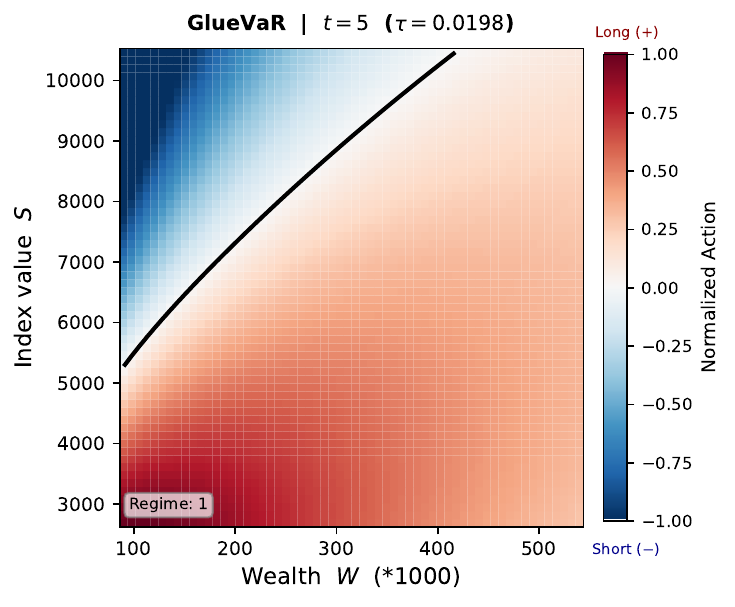}
			%\caption{$\hat{p}=0.75$}
		\end{minipage}%
	}%
	\subfigure[$\text{IER}_{0.90}$]{
		\begin{minipage}[t]{0.24\textwidth}
			\centering
			\includegraphics[width=\linewidth]{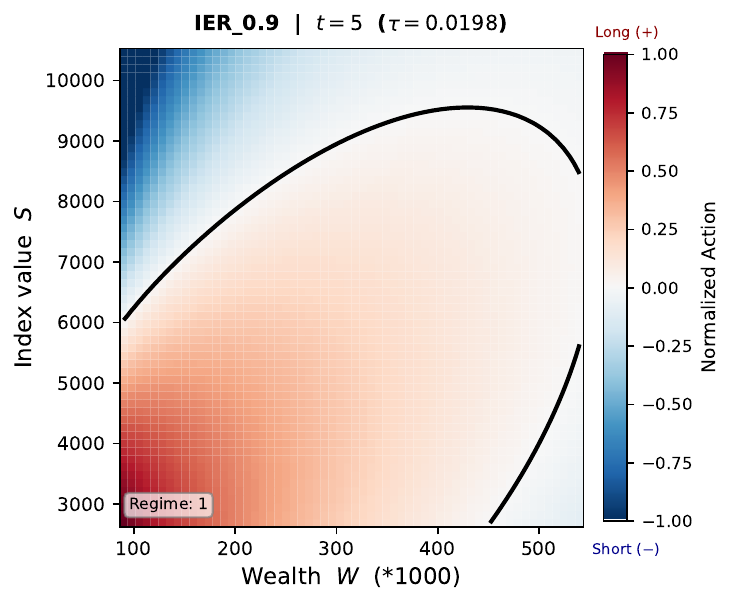}
			%\caption{$\hat{p}=0.95$}
		\end{minipage}%
	}%   

    \subfigure[$\text{IQR}_{0.90}$]{
		\begin{minipage}[t]{0.24\textwidth}
			\centering
			\includegraphics[width=\linewidth]{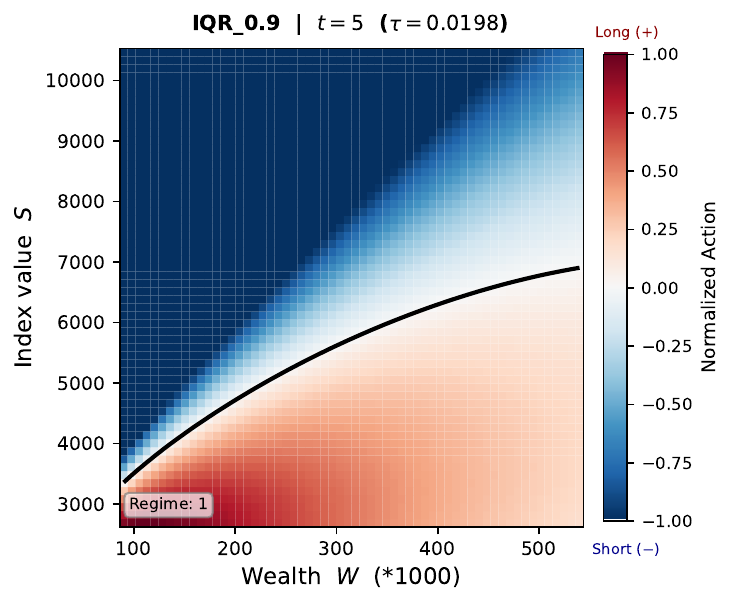}
			%\caption{$\hat{p}=0.95$}
		\end{minipage}%
	}%   
    \subfigure[$\text{IQR}_{0.95}$]{
		\begin{minipage}[t]{0.24\textwidth}
			\centering
			\includegraphics[width=\linewidth]{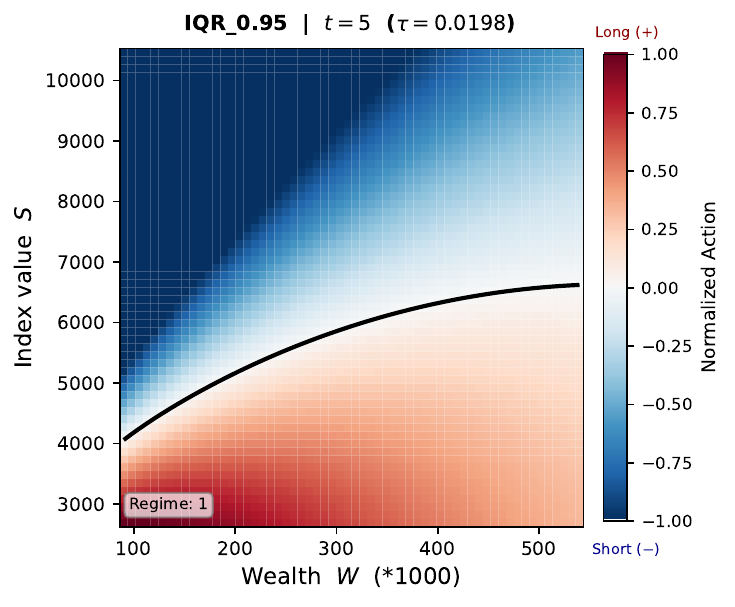}
			%\caption{$\hat{p}=0.95$}
		\end{minipage}%
	}%  
    \subfigure[$\text{RVaR}_{0.50,0.90}$]{
		\begin{minipage}[t]{0.24\textwidth}
			\centering
			\includegraphics[width=\linewidth]{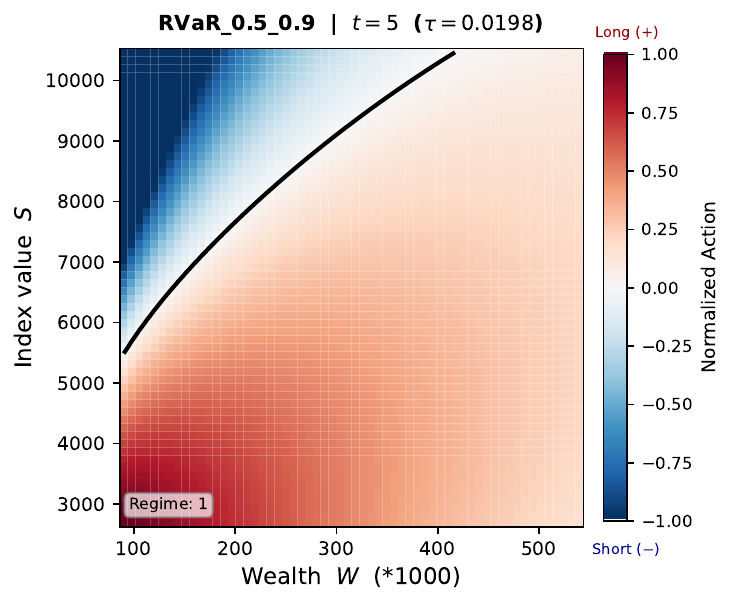}
			%\caption{$\hat{p}=0.95$}
		\end{minipage}%
	}%  
    \subfigure[$\text{RVaR}_{0.60,0.90}$]{
		\begin{minipage}[t]{0.24\textwidth}
			\centering
			\includegraphics[width=\linewidth]{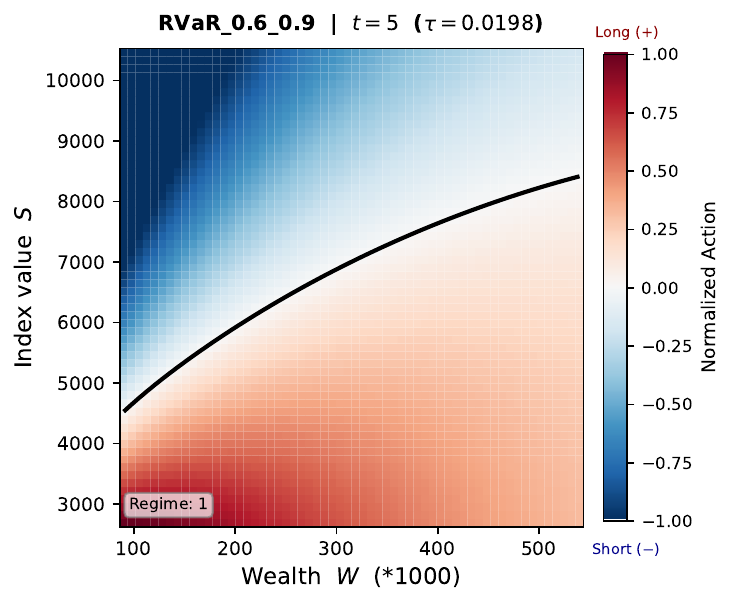}
			%\caption{$\hat{p}=0.95$}
		\end{minipage}%
	}%  

    \subfigure[$\text{VaR}_{0.85}$]{
		\begin{minipage}[t]{0.24\textwidth}
			\centering
			\includegraphics[width=\linewidth]{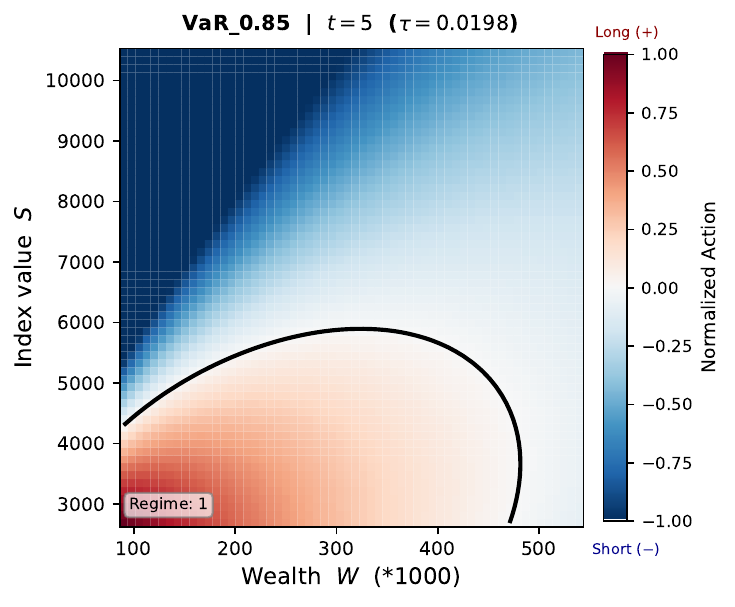}
			%\caption{$\hat{p}=0.95$}
		\end{minipage}%
	}% 
    \subfigure[$\text{VaR}_{0.95}$]{
		\begin{minipage}[t]{0.24\textwidth}
			\centering
			\includegraphics[width=\linewidth]{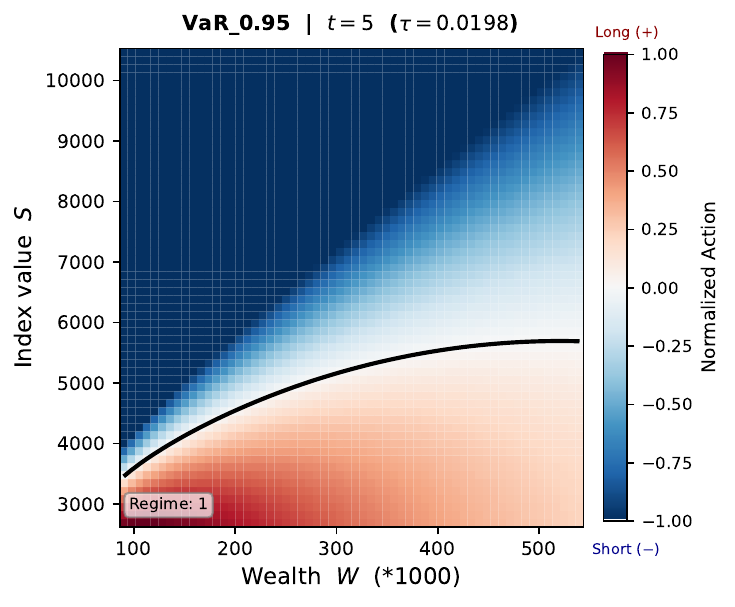}
			%\caption{$\hat{p}=0.95$}
		\end{minipage}%
	}% 
	\centering
	\caption{Learned policy in the bull market (regime $=1$). Colors indicate normalized actions in $[-1,1]$, with positive and negative values representing long and short positions, respectively.}
	\label{fig:Learned-policy-in-the-bull-market}
\end{figure}

\begin{figure}[htbp]
	\small
	\centering
	\subfigure[$\text{CVaR}_{0.90}$]{
		\begin{minipage}[t]{0.24\textwidth}
			\centering
			\includegraphics[width=\linewidth]{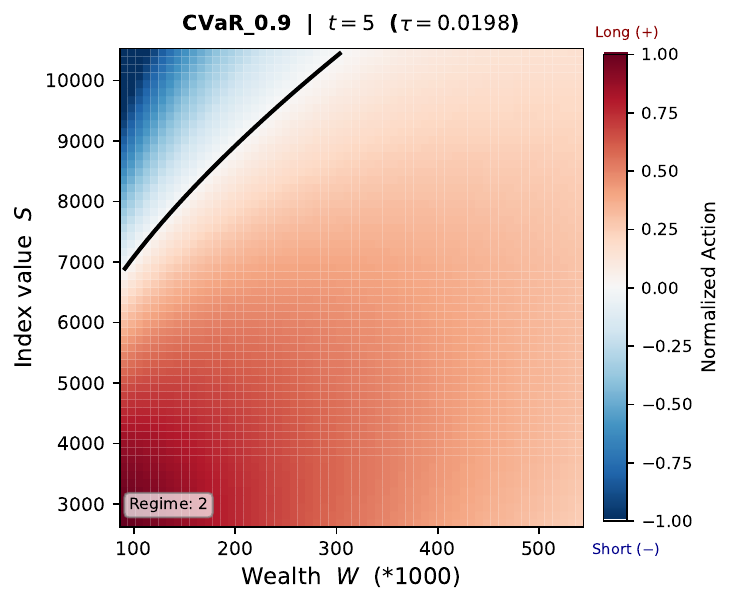}
			%\caption{$\hat{p}=0.75$}
		\end{minipage}%
	}%
	\subfigure[$\text{CVaR}_{0.95}$]{
		\begin{minipage}[t]{0.24\textwidth}
			\centering
			\includegraphics[width=\linewidth]{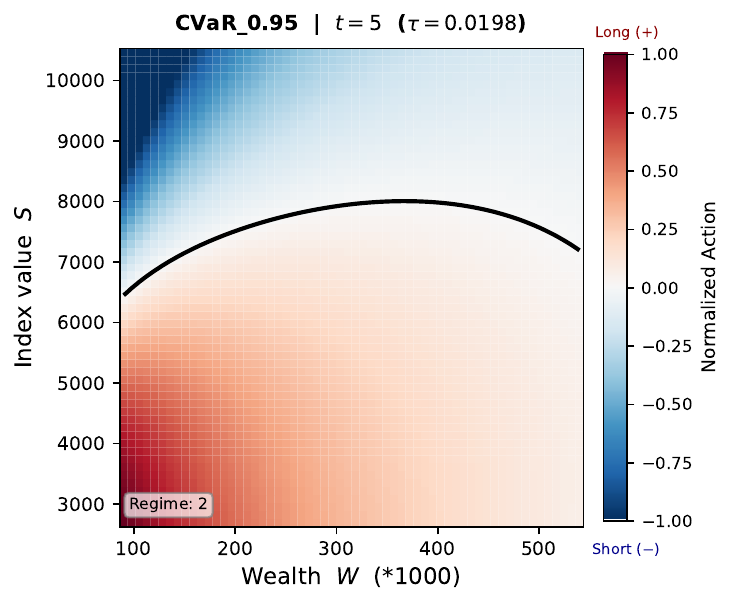}
			%\caption{$\hat{p}=0.95$}
		\end{minipage}%
	}%   
    \subfigure[GlueVaR]{
		\begin{minipage}[t]{0.24\textwidth}
			\centering
			\includegraphics[width=\linewidth]{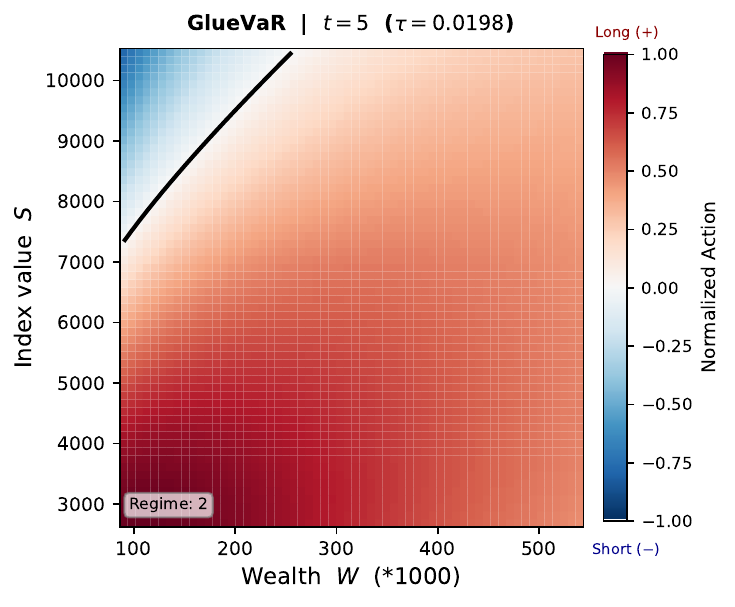}
			%\caption{$\hat{p}=0.75$}
		\end{minipage}%
	}%
	\subfigure[$\text{IER}_{0.90}$]{
		\begin{minipage}[t]{0.24\textwidth}
			\centering
			\includegraphics[width=\linewidth]{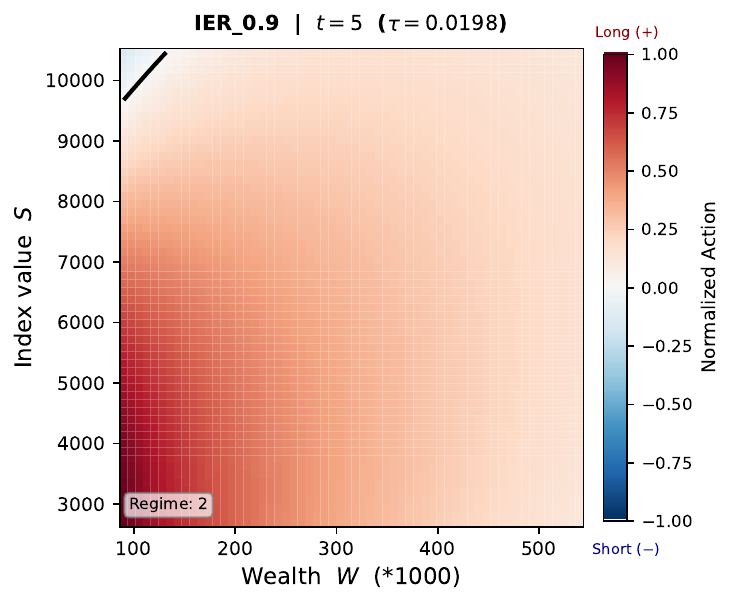}
			%\caption{$\hat{p}=0.95$}
		\end{minipage}%
	}%   

    \subfigure[$\text{IQR}_{0.90}$]{
		\begin{minipage}[t]{0.24\textwidth}
			\centering
			\includegraphics[width=\linewidth]{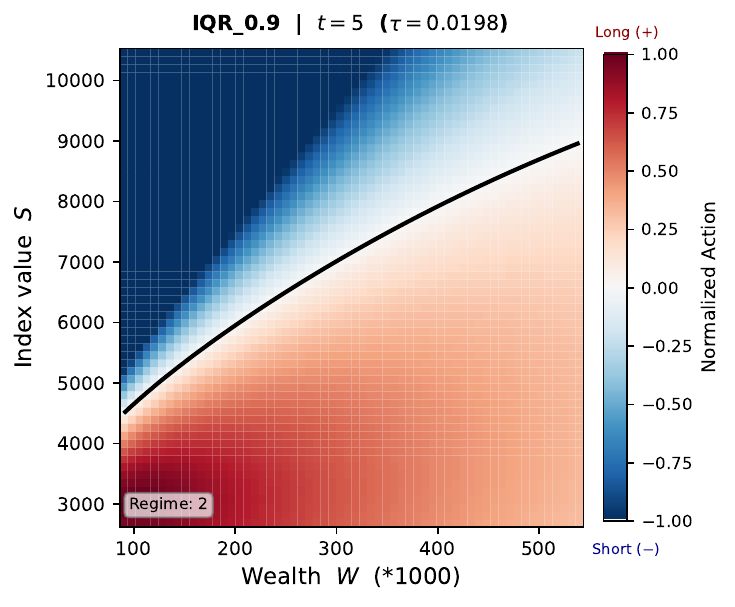}
			%\caption{$\hat{p}=0.95$}
		\end{minipage}%
	}%   
    \subfigure[$\text{IQR}_{0.95}$]{
		\begin{minipage}[t]{0.24\textwidth}
			\centering
			\includegraphics[width=\linewidth]{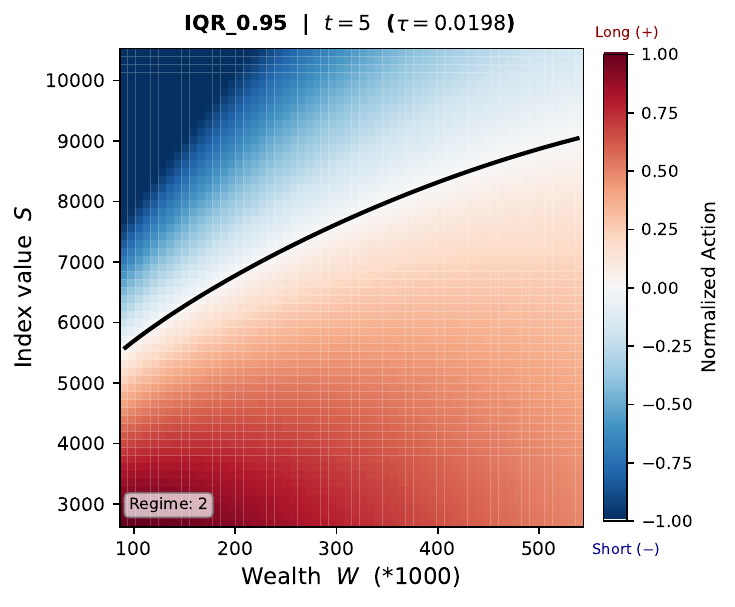}
			%\caption{$\hat{p}=0.95$}
		\end{minipage}%
	}%  
    \subfigure[$\text{RVaR}_{0.50,0.90}$]{
		\begin{minipage}[t]{0.24\textwidth}
			\centering
			\includegraphics[width=\linewidth]{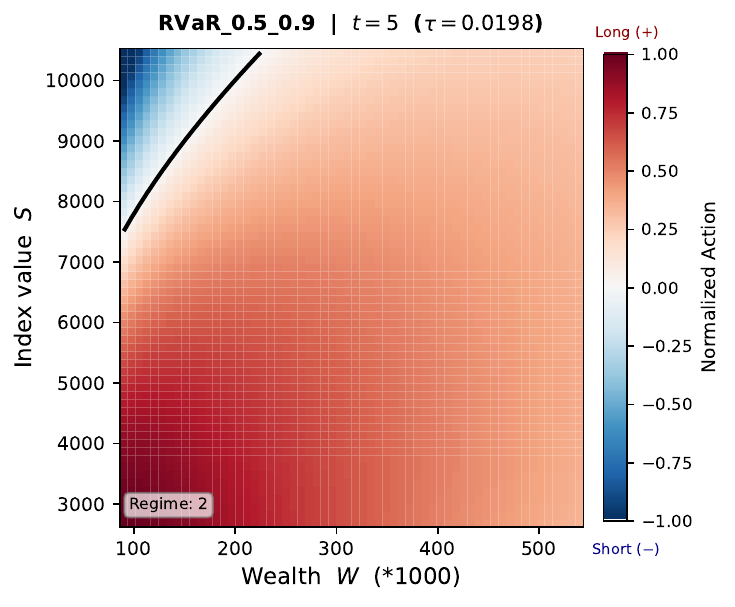}
			%\caption{$\hat{p}=0.95$}
		\end{minipage}%
	}%  
    \subfigure[$\text{RVaR}_{0.60,0.90}$]{
		\begin{minipage}[t]{0.24\textwidth}
			\centering
			\includegraphics[width=\linewidth]{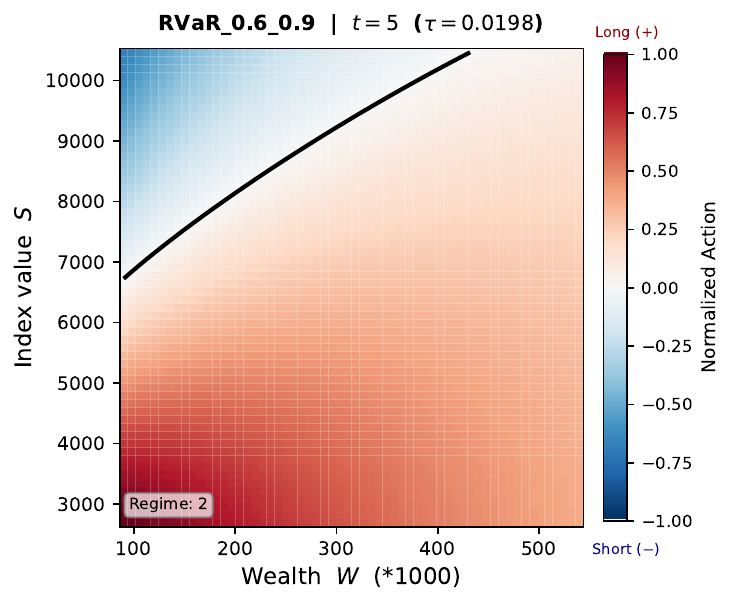}
			%\caption{$\hat{p}=0.95$}
		\end{minipage}%
	}%  

    \subfigure[$\text{VaR}_{0.85}$]{
		\begin{minipage}[t]{0.24\textwidth}
			\centering
			\includegraphics[width=\linewidth]{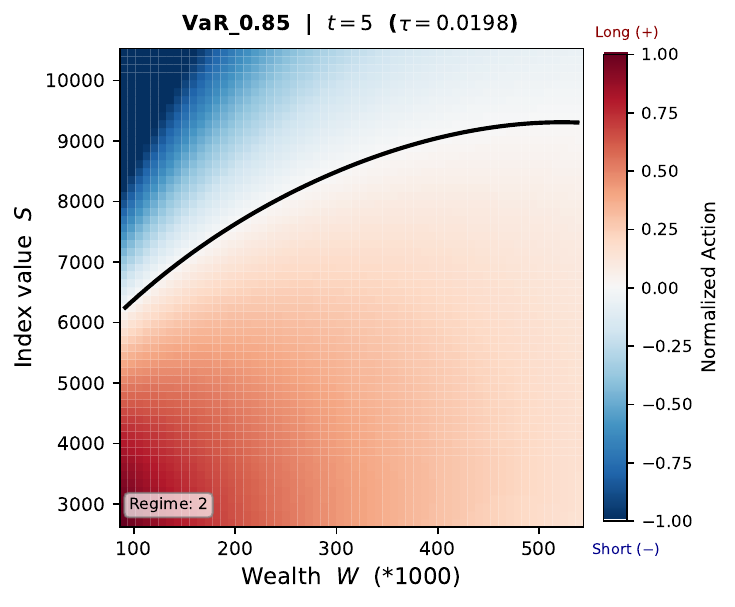}
			%\caption{$\hat{p}=0.95$}
		\end{minipage}%
	}% 
    \subfigure[$\text{VaR}_{0.95}$]{
		\begin{minipage}[t]{0.24\textwidth}
			\centering
			\includegraphics[width=\linewidth]{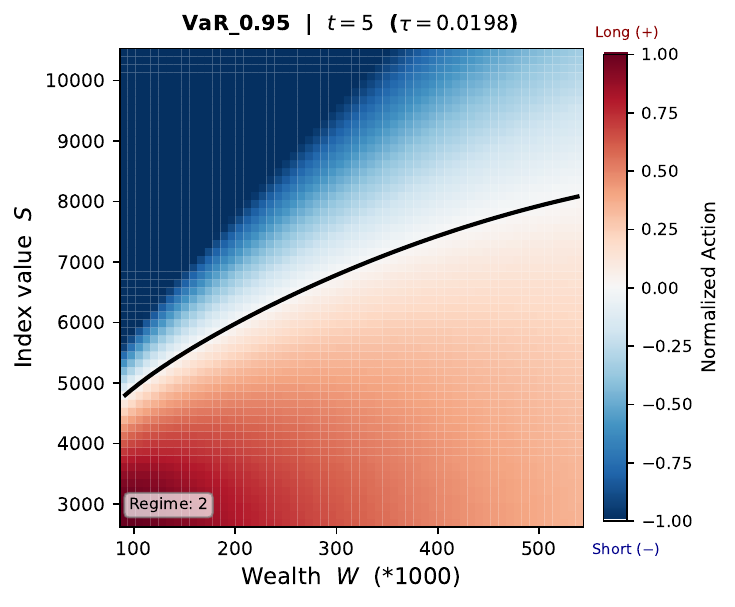}
			%\caption{$\hat{p}=0.95$}
		\end{minipage}%
	}% 
	\centering
	\caption{Learned policy in the bear market (regime $=2$). Colors indicate normalized actions in $[-1,1]$, with positive and negative values representing long and short positions, respectively.}
	\label{fig:Learned-policy-in-the-bear-market}
\end{figure}

\section{Conclusion}\label{sec:conslusion}
In this paper, we develop an IRL-RL framework for eliciting agents' risk preferences and optimizing decisions under the elicited preferences. Our framework applies to a broad class of distortion riskmetrics, including non-convex and non-coherent cases. We collect agents' choices through iterative questioning and employ a Bayesian IRL framework to identify their risk preferences, which remains effective despite agents' stochastic choices. We establish the existence of a finite set of distinguishing questions and prove the convergence rate of our IRL algorithm. In the decision-making stage, we extend the PPO algorithm to the conditional distortion riskmetric objective by representing the objective as an integral of the cost quantile function over $[0,1]$ with respect to the distortion function, and approximate it using a quantile neural network and the midpoint Riemann sum. Numerical experiments demonstrate the  effectiveness of our IRL-RL framework.

%\THEEndNotes
%\begingroup \parindent 0pt \parskip 0.0ex \def\enotesize{\normalsize} \theendnotes \endgroup

% Appendix here
% Options are (1) APPENDIX (with or without general title) or
%             (2) APPENDICES (if it has more than one unrelated sections)
% Outcomment the appropriate case if necessary
%
% \begin{APPENDIX}{<Title of the Appendix>}
% \end{APPENDIX}
%
%   or
%
% \begin{APPENDICES}
% \section{<Title of Section A>}
% \section{<Title of Section B>}
% etc
% \end{APPENDICES}

% Acknowledgments here
\ACKNOWLEDGMENT{The authors are grateful to members of the research group on financial mathematics and risk management at The Chinese University of Hong Kong (Shenzhen) for their useful feedback and conversations. 
Yang Liu acknowledges financial support from the National Natural Science Foundation of China (Grant No. 12401624), The Chinese University of Hong Kong (Shenzhen) University Development Fund (Grant No. UDF01003336), Guangdong Science and Technology Program (Grant No. 2024QN11X076) and Shenzhen Science and Technology Program (Grant No. RCBS20231211090814028, JCYJ20250604141203005, 2025TC0010) and is partly supported by the Guangdong Provincial Key Laboratory of Mathematical Foundations for Artificial Intelligence (Grant No. 2023B1212010001). Yunran Wei acknowledges financial support from the Natural Sciences and Engineering Research Council of Canada (RGPIN-2023-04674, DGECR-2023-00454), and the start-up fund at Carleton University and thanks %the School of Science and Engineering at 
The Chinese University of Hong Kong (Shenzhen) for the kind hospitality during her visit in 2025. }

% References here (outcomment the appropriate case)

% CASE 1: BiBTeX used to constantly update the references
%   (while the paper is being written).
%\bibliographystyle{informs2014} % outcomment this and next line in Case 1
%\bibliography{<your bib file(s)>} % if more than one, comma separated

\bibliographystyle{informs2014} % outcomment this and next line in Case 1
\bibliography{references} % if more than one, comma separated

% CASE 2: BiBTeX used to generate mypaper.bbl (to be further fine tuned)
%\input{mypaper.bbl} % outcomment this line in Case 2

%If you don't use BiBTex, you can manually itemize references as shown below.

%\bibliographystyle{nonumber}

%%%%%%%%%%%%%%%%%
\end{document}